\documentclass[accepted]{uai2026} 
                        
\usepackage[american]{babel}

\usepackage{natbib} 
\usepackage{mathtools} 
\usepackage{booktabs} 
\usepackage{tikz} 

\usepackage{setspace}

\usepackage{graphicx}

\usepackage{enumitem}

\usepackage{microtype}
\usepackage{subfigure}
\usepackage{comment}
\usepackage{wrapfig}
\usepackage{arydshln}
\usepackage{enumitem}
\usepackage{multirow}
\usepackage{url}
\usepackage{natbib}
\usepackage{authblk}

\usepackage{amsmath}
\usepackage{amssymb}
\usepackage{amsfonts}
\usepackage{bm}
\usepackage{bbm}

\usepackage{algorithm}
\usepackage{algorithmic}

\def\eqref#1{equation~\ref{#1}}

\def\1{\bm{1}}

\def\rmd{{\mathrm{d}}}

\DeclareMathAlphabet{\mathsfit}{\encodingdefault}{\sfdefault}{m}{sl}
\SetMathAlphabet{\mathsfit}{bold}{\encodingdefault}{\sfdefault}{bx}{n}

\def\calB{{\mathcal{B}}}

\def\calD{{\mathcal{D}}}

\def\calF{{\mathcal{F}}}

\def\calH{{\mathcal{H}}}

\def\calJ{{\mathcal{J}}}

\def\calW{{\mathcal{W}}}
\def\calX{{\mathcal{X}}}
\def\calY{{\mathcal{Y}}}
\def\calZ{{\mathcal{Z}}}

\def\bbE{{\mathbb{E}}}

\def\bbP{{\mathbb{P}}}

\def\bbR{{\mathbb{R}}}

\newcommand{\KL}{D_{\mathrm{KL}}}

\newcommand{\p}[1]{\left(#1\right)}
\newcommand{\sqb}[1]{\left[#1\right]}
\newcommand{\cb}[1]{\left\{#1\right\}}

\newcommand{\Bigp}[1]{\Big(#1\Big)}

\newcommand{\abs}[1]{\left|#1\right|}

\usepackage{amsthm}

\theoremstyle{plain}

\newtheorem{theorem}{Theorem}[section]

\newtheorem{corollary}[theorem]{Corollary}
\newtheorem{proposition}[theorem]{Proposition}

\newtheorem{assumption}{Assumption}[section]

\newtheorem*{remark}{Remark}

\usepackage[textsize=tiny]{todonotes}
\usepackage{multirow}
\usepackage{wrapfig}
\usepackage{subfigure}

\renewcommand{\eqref}[1]{(\ref{#1})}

\usepackage{xcolor}
\newcount\Comments  
\newcommand{\kibitz}[2]{\ifnum\Comments=1\textcolor{#1}{#2}\fi}

\usepackage{listings}

\allowdisplaybreaks

\title{General Bayesian Policy Learning}

\author{\href{mkato-csecon@g.ecc.u-tokyo.ac.jp}{Masahiro Kato}{}}
\affil{Mizuho-DL Financial Technology Co., Ltd.\\
    Tokyo, Japan
}
  
\begin{document}
\maketitle

\begin{abstract}
This study proposes a General Bayes framework for policy learning. We consider decision problems in which a decision-maker chooses an action from a given set to maximize expected welfare. Typical examples include treatment choice and portfolio optimization. In such problems, the statistical target is a decision rule, and predicting each potential outcome is not necessarily of primary interest. We formulate this policy-learning problem through loss-based Bayesian updating. Our main technical device is a squared-loss surrogate for welfare maximization. We show that maximizing empirical welfare over a policy class with a quadratic penalty controlled by a tuning parameter $\zeta>0$ is equivalent to minimizing a scaled squared error in the outcome difference. The resulting General Bayes posterior over decision rules admits two equivalent characterizations: a Gaussian pseudo-likelihood representation and a decision-theoretic loss-based characterization. As one implementation, we introduce GBPLNet, a neural network implementation with a tanh-squashed output. Finally, we establish PAC-Bayes-type guarantees for the surrogate risk and the corresponding penalized welfare.
\end{abstract}

\section{Introduction}
Policy learning aims to train a policy function $\delta(x)$ that maps context features $x\in\calX$ to an action $a$ in a constrained action set, so as to maximize an expected outcome or, equivalently, to minimize an expected task loss. When the action set is discrete, many objectives can be written as a cost-sensitive $0$-$1$ loss, where the loss penalizes choosing a suboptimal action with a weight given by an outcome gap. This perspective suggests using a learning rule that targets a task loss directly rather than fitting a potentially misspecified generative model.

General Bayes is a framework for updating beliefs using a loss function rather than a likelihood function \citep{Bissiri2016generalframework}.
Given a prior distribution and an empirical loss, General Bayes yields a generalized posterior that is coherent from a decision-theoretic viewpoint. This framework is attractive when a probabilistic model is misspecified, inconvenient, or unnecessary for the downstream task.
This study develops a General Bayes framework for policy learning. The generalized posterior provides uncertainty over the decision rule itself. It can be used to summarize uncertainty in policy welfare and to assess whether the preferred action is stable across posterior draws. The goal is not to obtain uniform welfare improvements over empirical welfare maximization, weighted classification, or plug-in regression, but to retain a welfare-based objective while providing posterior information about policies and welfare.

However, a practical obstacle is that welfare objectives are typically linear in the policy, so they do not directly correspond to a convenient likelihood. We address this by rewriting welfare maximization as a squared-loss minimization problem.
In the binary action case, the key surrogate loss has the form $\p{\frac{1}{\sqrt{\zeta}}\p{y\p{1}-y\p{0}}-\sqrt{\zeta}f\p{x}}^2$, 
where $f\p{x}\in\sqb{-1,1}$ encodes a policy and $\zeta>0$ is a tuning parameter.
This surrogate admits a Gaussian pseudo-likelihood interpretation, so it enables a computationally convenient General Bayes posterior and standard approximation methods.

More generally, General Bayes allows any loss, so one could update directly using the negative welfare contribution as a loss.
This direct update yields objectives that are linear in the policy. It does not induce the quadratic regularization in Theorem~\ref{thm:binary-equivalence}, and it does not provide a Gaussian pseudo-likelihood that can be exploited by standard Bayesian computation. 
The squared-loss surrogate yields a regression-style objective with stable gradients, an explicit regularization path indexed by $\zeta$, and a unified extension to $K$ actions and missing outcome settings.

For $K$ actions, a baseline-gap formulation in terms of outcome differences is natural, but it introduces a dependence on the baseline action through the regularization term.
We therefore develop a baseline-free symmetric surrogate that operates on the full feedback vector.
Furthermore, in missing outcome settings such as observational studies and logged bandit feedback, outcome differences are not observed.
We show that inverse propensity weighting (IPW) and doubly robust (DR) pseudo-outcomes yield explicit empirical losses that can be used for General Bayes updating, and we provide population-level target characterizations.

\paragraph{Contributions.}
We list our contributions as follows:
\begin{itemize}[noitemsep, topsep=0pt, leftmargin=0.40cm]
  \item We propose a General Bayes framework for policy learning that updates a prior over decision rules.
\item For binary actions, we show that empirical welfare maximization is equivalent to minimizing a scaled squared-loss surrogate, up to a quadratic regularization controlled by $\zeta$ (Section~\ref{sec:binary}).
\item We define the associated General Bayes posterior, give its Gaussian pseudo-likelihood and loss-based characterizations, and discuss the role of the temperature parameter $\eta$ (Section~\ref{sec:binary}).
  \item For $K$ actions, we present a baseline-gap surrogate and a baseline-free symmetric full-vector surrogate, and we discuss baseline choice (Section~\ref{sec:K}).
  \item For missing outcomes, we define empirical losses based on IPW and DR pseudo-outcomes, and we provide population-level target results (Section~\ref{sec:missing}).
  \item We give PAC-Bayes bounds for the surrogate risk and translate them into bounds for penalized welfare (Section~\ref{sec:theory}).
  \item We describe GBPLNet, neural networks with tanh-squashed outputs, as one implementation example for bounded scores (Section~\ref{sec:computation} and Appendix~\ref{appdx:detail_gbpl}).
\end{itemize}
We review related work in Appendix~\ref{app:related}. 
We refer to our framework as \emph{General Bayesian Policy Learning (GBPL)} and our proposed implementation with neural networks as \emph{GBPLNet}.
Our proposed method can be applied to various tasks, including treatment choice \citep{Kitagawa2018whoshould,Athey2021policylearning,Zhou2023offlinemultiaction} and portfolio optimization \citep{Tallman2023bayesianpredictive,Tallman2024predictivedecision,Kato2024bayesianportfolio,Kato2024generalbayesian}.
Note that we use empirical welfare maximization in a general sense, not restricted to counterfactual risk minimization \citep{Swaminathan2015batchlearning,Kitagawa2018whoshould}.

\section{Setup}
\label{sec:setup}
Assume that there are $K$ actions, indexed by $1,2,\dots,K$.
Let $X\in\calX\subseteq \bbR^d$ be a regressor and $Y\p{a}\in\calY\subseteq \bbR$ be an outcome for each action $a\in\cb{1,\dots,K}$.
Let $Z=\p{X,Y\p{1},\dots,Y\p{K}}$ be a random vector following a distribution $P$ on $\calZ=\calX\times\calY^K$.
Given i.i.d. observations $\calD=\cb{z_i}_{i=1}^n$ with $z_i=\p{x_i,y_i\p{1},\dots,y_i\p{K}}$, a decision-maker trains a decision rule $\delta\colon \calX\to\cb{1,\dots,K}$ to maximize the expected welfare
\begin{align}
  V\p{\delta}
  =
  \bbE\sqb{\sum_{a\in\cb{1,\dots,K}}\mathbbm{1}\sqb{\delta\p{X}=a}Y\p{a}}.
  \label{eq:welfare}
\end{align}
The Bayes optimal rule satisfies
\begin{align*}
  \delta^\star\p{x}\in\arg\max_{a\in\cb{1,\dots,K}}\bbE\sqb{Y\p{a}\mid X=x}
\end{align*}
for almost every $x$.

In many applications, outcomes are not fully observed, for example, under the potential outcome framework in causal inference and in bandit feedback in online learning.
The missing outcome setting is introduced in Section~\ref{sec:missing}.
Sections \ref{sec:binary} to \ref{sec:K} focus on the full feedback setting to isolate the surrogate construction.

\section{Background: Generalized Bayesian Updating}
Let $\Theta$ be a parameter space and let $\theta\in\Theta$ index a model for a decision rule or a score function.
Let $\Pi$ be a prior distribution on $\Theta$.
Given a loss function $\ell\p{\theta;z}$ and data $\calD=\cb{z_i}_{i=1}^n$, generalized Bayesian updating defines a generalized posterior
\begin{align}
  \rmd \Pi_{\eta}\p{\theta\mid \calD}
  \propto
  \rmd \Pi\p{\theta}\exp\p{-\eta\sum^n_{i=1}\ell\p{\theta;z_i}},
  \label{eq:gb-posterior}
\end{align}
where $\eta>0$ is a temperature parameter.

\paragraph{Decision-theoretic derivation.}
A key result of \citet{Bissiri2016generalframework} is that \eqref{eq:gb-posterior} is the optimizer of a variational problem.
Let $Q$ be any probability distribution on $\Theta$ such that $Q$ is absolutely continuous with respect to $\Pi$.
Define
\begin{align}
  \calJ\p{Q}
  =
  \eta \bbE_{\theta\sim Q}\sqb{\sum^n_{i=1}\ell\p{\theta;z_i}}
  +
  \KL\p{Q\|\Pi}.
  \label{eq:variational-objective}
\end{align}

\begin{proposition}
\label{prop:variational}
Assume that the normalizing constant $Z_{\eta}\p{\calD}=\int \exp\p{-\eta\sum^n_{i=1}\ell\p{\theta;z_i}}\rmd \Pi\p{\theta}$
is finite.
Then the unique minimizer of $\calJ\p{Q}$ over all such $Q$ is $Q=\Pi_{\eta}\p{\cdot\mid \calD}$.
\end{proposition}

The proof is provided in Appendix~\ref{app:proof-variational}.

\paragraph{Relation to ordinary Bayes.}
If the loss is chosen as the negative log-likelihood, $\ell\p{\theta;z}=-\log p\p{z\mid \theta}$, and $\eta=1$, then \eqref{eq:gb-posterior} reduces to ordinary Bayes' rule.
For other losses, \eqref{eq:gb-posterior} yields a coherent update even when no global likelihood model is posited \citep{Bissiri2016generalframework}.

\section{GBPL with Binary Actions}
\label{sec:binary}
For simplicity, we first consider the binary action case, $K=2$.
We reindex actions $1$ and $2$ as $1$ and $0$.
We consider a possibly randomized policy $\delta\colon\calX\to\sqb{0,1}$.
The expected welfare can be written as
\begin{align*}
  V\p{\delta}
  =
  \bbE\sqb{\delta\p{X}Y\p{1}+\p{1-\delta\p{X}}Y\p{0}}.
\end{align*}

\subsection{Empirical Welfare Maximization}
Let $\calH_{\mathrm{pol}}$ be a class of measurable maps $\delta\colon\calX\to\sqb{0,1}$.
The empirical welfare is
\begin{align}
  \widehat{V}\p{\delta}
  =
  \frac{1}{n}\sum^n_{i=1}\p{\delta\p{X_i}Y_i\p{1}+\p{1-\delta\p{X_i}}Y_i\p{0}}.
  \label{eq:empwelfare}
\end{align}
A natural estimator is
\begin{align}
  \widehat{\delta}
  \in
  \arg\max_{\delta\in\calH_{\mathrm{pol}}}\widehat{V}\p{\delta}.
  \label{eq:empwelfare-opt}
\end{align}
From a Bayesian perspective, \eqref{eq:empwelfare-opt} does not specify a likelihood, so it is unclear how to define a posterior distribution over $\delta$ without further modeling choices.

\subsection{Squared-loss surrogate}
Define the score function $f\p{x}=2\delta\p{x}-1$ and the induced class $\calF_{\calH_{\mathrm{pol}}}
  =
  \cb{f\colon f\p{x}=2\delta\p{x}-1,\delta\in\calH_{\mathrm{pol}}}$. 
Let $\zeta>0$ be a constant.
We consider the squared-loss surrogate
\begin{align}
\label{eq:surrogate-empirical}
  &\widehat{f}\p{\zeta}
  \in\\
  &\arg\min_{f\in\calF_{\calH_{\mathrm{pol}}}}
    \frac{1}{n}\sum^n_{i=1}\p{\frac{1}{\sqrt{\zeta}}\p{Y_i\p{1}-Y_i\p{0}}-\sqrt{\zeta}f\p{X_i}}^2.\notag
\end{align}

To state the equivalence, define the penalized welfare maximizer
\begin{align}
  \widetilde{\delta}\p{\lambda}
  \in
  \arg\max_{\delta\in\calH_{\mathrm{pol}}}\cb{
    \widehat{V}\p{\delta}
    -
    \lambda\frac{1}{n}\sum^n_{i=1}\p{2\delta\p{X_i}-1}^2
  }.
  \label{eq:empwelfare-penalized}
\end{align}

\begin{theorem}
\label{thm:binary-equivalence}
For any $\zeta>0$, a score $\widehat{f}\p{\zeta}$ minimizes \eqref{eq:surrogate-empirical} if and only if there exists a maximizer $\widetilde{\delta}\p{\zeta/4}$ of \eqref{eq:empwelfare-penalized} such that $\widehat{f}\p{\zeta}=2\widetilde{\delta}\p{\zeta/4}-1$.
\end{theorem}

The proof is provided in Appendix~\ref{app:proof-binary}.
As $\zeta\to 0$, the penalty in \eqref{eq:empwelfare-penalized} vanishes.
If the maximizer in \eqref{eq:empwelfare-opt} is unique, then $\widetilde{\delta}\p{\lambda}\to \widehat{\delta}$ as $\lambda\to 0$.
Without uniqueness, any limit point of $\widetilde{\delta}\p{\lambda}$ belongs to the set of maximizers in \eqref{eq:empwelfare-opt}.
The penalty term in \eqref{eq:empwelfare-penalized} is minimized at $\delta\p{x}=1/2$, so for finite $\zeta$ the surrogate induces shrinkage toward randomized policies.

\subsection{Loss and Generalized Posterior}
Let $f_{\theta}$ be a parametric score model indexed by $\theta\in\Theta$.
In the binary action case, we define the loss as
\begin{align}
  \ell\p{\theta;\p{x,y\p{1},y\p{0}}}
  =
  \frac{1}{2}\p{\frac{1}{\sqrt{\zeta}}\p{y\p{1}-y\p{0}}-\sqrt{\zeta}f_{\theta}\p{x}}^2.
  \label{eq:loss-binary}
\end{align}
The generalized posterior is then
\begin{align}
  \rmd \Pi_{\eta}\p{\theta\mid \calD}
  \propto
  \rmd \Pi\p{\theta}\exp\p{-\eta\sum^n_{i=1}\ell\p{\theta;z_i}},
  \label{eq:gb-binary}
\end{align}
where $z_i=\p{x_i,y_i\p{1},y_i\p{0}}$.

\subsection{Gaussian Pseudo-Likelihood Representation}
Let $u=y\p{1}-y\p{0}$ denote the outcome difference.
The exponential weight in \eqref{eq:gb-binary} can be rewritten as
\[
\exp\p{-\eta\ell\p{\theta;\p{x,y\p{1},y\p{0}}}}
=
\exp\p{-\frac{\eta}{2\zeta}\p{u-\zeta f_{\theta}\p{x}}^2}.
\]
Therefore, the exponential weight in \eqref{eq:gb-binary} is identical to that induced by the Gaussian pseudo-likelihood
\begin{align}
  U\mid X=x,\quad \theta \sim \mathcal{N}\p{\zeta f_{\theta}\p{x},\zeta/\eta}.
  \label{eq:gaussian-pseudolikelihood}
\end{align}
Equation~\eqref{eq:gaussian-pseudolikelihood} is an algebraic representation of the General Bayes update and is not an assumption that the true conditional distribution of $U$ given $X$ is Gaussian.

\subsection{General Bayes Interpretation}
The posterior \eqref{eq:gb-binary} has two equivalent characterizations.
First, under the Gaussian pseudo-likelihood in \eqref{eq:gaussian-pseudolikelihood}, it has the form of an ordinary Bayesian posterior.
Second, without committing to any likelihood, it is the unique optimizer of the variational problem in Proposition~\ref{prop:variational}, so it is a coherent loss-based belief update \citep{Bissiri2016generalframework}.

The parameters $\zeta$ and $\eta$ play different roles.
The parameter $\zeta$ changes the learning objective itself, because it determines the strength of the quadratic regularization in Theorem~\ref{thm:binary-equivalence}.
The parameter $\eta$ changes posterior concentration, and it can be viewed as a calibration parameter.
Moreover, the parameter $\eta$ affects point estimators derived from the generalized posterior, for example, the maximum a posteriori (MAP) estimator, because it rescales the empirical loss relative to the prior through \eqref{eq:variational-objective}.
Under the Gaussian pseudo-likelihood in \eqref{eq:gaussian-pseudolikelihood}, $\zeta$ sets the scale of the conditional mean and $\zeta/\eta$ is the conditional variance, so $\zeta$ and $\eta$ jointly determine the pseudo-likelihood scale.
Selecting $\eta$ is therefore a calibration problem, and practical strategies include Fisher-information matching through the loss-likelihood bootstrap \citep{Lyddon2019generalbayesian} and adaptive learning-rate procedures \citep{Grunwald2012thesafe}.

\section{GBPL with Multiple Actions}
\label{sec:K}
We now extend the squared-loss surrogate and the posterior construction to $K$ actions.
We present two surrogate constructions.
The first is based on outcome gaps relative to a baseline action.
The second is baseline-free and symmetric, and it uses the full feedback vector.

\subsection{Baseline-Gap Surrogate}
Fix a baseline action $K$.
Define outcome gaps $\overline{U}\p{a}=Y\p{a}-Y\p{K}$ for $a\in\cb{1,\dots,K-1}$. 
For a possibly randomized policy $\delta\colon\calX\to\Delta_K$, where $\Delta_K=\cb{p\in\bbR^K\colon p_a\ge 0,\sum_{a=1}^K p_a=1}$, welfare can be written as
\begin{align*}
V\p{\delta}
=
\bbE\sqb{Y\p{K}+\sum_{a=1}^{K-1}\delta_a\p{X}\overline{U}\p{a}}.
\end{align*}
The baseline term $\bbE\sqb{Y\p{K}}$ does not depend on $\delta$.

Let $\calH_{\mathrm{pol}}^{\p{K}}$ be a class of policies $\delta\colon\calX\to\Delta_K$.
For $a\in\cb{1,\dots,K-1}$, define the componentwise score transform $f_a\p{x}=2\delta_a\p{x}-1$. 
Let $\calF_{\calH_{\mathrm{pol}}^{\p{K}}}$ denote the induced class of score vectors $f\p{x}=\p{f_1\p{x},\dots,f_{K-1}\p{x}}$.
Because this is an induced class, its components retain the simplex constraint inherited from $\delta$ and cannot be chosen independently in $\sqb{-1,1}^{K-1}$.
For $\zeta>0$, define the empirical surrogate
\begin{align}
\label{eq:surrogate-K-gap}
&\widehat{f}^{\mathrm{Gap}}\p{\zeta}
\in \arg\min_{f\in\calF_{\calH_{\mathrm{pol}}^{\p{K}}}}\\
&
\frac{1}{n}\sum^n_{i=1}\sum_{a=1}^{K-1}\p{\frac{1}{\sqrt{\zeta}}\p{Y_i\p{a}-Y_i\p{K}}-\sqrt{\zeta}f_a\p{X_i}}^2.\notag
\end{align}

Define the empirical welfare $\widehat{V}\p{\delta}
=
\frac{1}{n}\sum^n_{i=1}\sum_{a=1}^{K}\delta_a\p{X_i}Y_i\p{a}$
and the penalized empirical welfare maximizer
\begin{align}
\label{eq:penalized-K-gap}
&\widetilde{\delta}^{\mathrm{Gap}}\p{\lambda}
\in\\
&\arg\max_{\delta\in\calH_{\mathrm{pol}}^{\p{K}}}\cb{
\widehat{V}\p{\delta}
-
\lambda\frac{1}{n}\sum^n_{i=1}\sum_{a=1}^{K-1}\p{2\delta_a\p{X_i}-1}^2
}.\notag
\end{align}

\begin{theorem}
\label{thm:K-equivalence-gap}
For any $\zeta>0$, a score $\widehat{f}^{\mathrm{Gap}}\p{\zeta}$ minimizes \eqref{eq:surrogate-K-gap} if and only if there exists a maximizer $\widetilde{\delta}^{\mathrm{Gap}}\p{\zeta/4}$ of \eqref{eq:penalized-K-gap} such that $\widehat{f}^{\mathrm{Gap}}_a\p{\zeta}=2\widetilde{\delta}^{\mathrm{Gap}}_a\p{\zeta/4}-1$ for every $a\in\cb{1,\dots,K-1}$.
\end{theorem}

The proof is provided in Appendix~\ref{app:proof-K-gap}.

\subsection{Baseline-Free Symmetric Full-Vector Surrogate}
The baseline-gap surrogate is convenient. However, its regularization depends on the chosen baseline through the term $\sum_{a=1}^{K-1}\p{2\delta_a\p{X}-1}^2$.
To remove baseline dependence, we use a full-vector surrogate that treats all actions symmetrically.

Let $\calH_{\mathrm{pol}}^{\p{K}}$ be a class of policies $\delta\colon\calX\to\Delta_K$.
For $\zeta>0$, define the empirical full-vector surrogate
\begin{align}
\label{eq:surrogate-K-full}
&\widehat{\delta}^{\mathrm{Full}}\p{\zeta}
\in\\
&\arg\min_{\delta\in\calH_{\mathrm{pol}}^{\p{K}}}\cb{
\frac{1}{n}\sum^n_{i=1}\sum_{a=1}^{K}\p{\frac{1}{\sqrt{\zeta}}Y_i\p{a}-\sqrt{\zeta}\delta_a\p{X_i}}^2
}.\notag
\end{align}
Define the penalized empirical welfare maximizer
\begin{align}
\widetilde{\delta}^{\mathrm{Full}}\p{\lambda}
\in
\arg\max_{\delta\in\calH_{\mathrm{pol}}^{\p{K}}}\cb{
\widehat{V}\p{\delta}
-
\lambda\frac{1}{n}\sum^n_{i=1}\sum_{a=1}^{K}\delta_a\p{X_i}^2
}.
\label{eq:penalized-K-full}
\end{align}

The penalty term in \eqref{eq:penalized-K-full} is minimized at $\delta_a\p{x}=1/K$ for all $a$, so for finite $\zeta$ the full-vector surrogate induces shrinkage toward uniform randomization.

\begin{theorem}
\label{thm:K-equivalence-full}
For any $\zeta>0$, a policy $\widehat{\delta}^{\mathrm{Full}}\p{\zeta}$ minimizes \eqref{eq:surrogate-K-full} if and only if it maximizes \eqref{eq:penalized-K-full} with $\lambda=\zeta/2$.
\end{theorem}

The proof is provided in Appendix~\ref{app:proof-K-full}.

\subsection{Baseline Dependence and Symmetric Alternatives}
Theorems \ref{thm:K-equivalence-gap} and \ref{thm:K-equivalence-full} show that both surrogate constructions correspond to penalized welfare maximization, but with different regularizers.
In particular, the baseline-gap penalty depends on which action is chosen as baseline, while the full-vector penalty is baseline-free and symmetric across actions.

If the interest is the unregularized welfare maximizer, one may take $\zeta$ small so that the regularization is negligible.
In that regime, baseline choice does not materially affect the maximizer, but it can affect numerical stability, because it changes the scaling and the variance of estimated gaps in missing outcome settings.
If one prefers an invariant formulation for finite $\zeta$, the full-vector surrogate provides a direct baseline-free alternative.

Another symmetric alternative is to average baseline-gap surrogates over baselines.
This produces a pairwise gap objective that depends on all differences $Y\p{a}-Y\p{b}$.
We do not pursue this direction further, because it introduces $K\p{K-1}/2$ pairwise terms.

\subsection{Shift Invariance and Centering}
Welfare is invariant to adding an action-common shift that depends on $X$, because $\sum_{a=1}^K \delta_a\p{X}=1$.
That is, for any measurable $c\p{X}$, we have $\bbE\sqb{\sum_{a=1}^K \delta_a\p{X}\p{Y\p{a}+c\p{X}}}
=
V\p{\delta}+\bbE\sqb{c\p{X}}$. 
The argmax over $\delta$ is therefore unchanged.
Furthermore, the full-vector surrogate inherits the same invariance, because replacing $Y_i\p{a}$ by $Y_i\p{a}+c_i$ for any $c_i$ that does not depend on $a$ changes \eqref{eq:surrogate-K-full} only through additive constants, using again that $\sum_{a=1}^K \delta_a\p{X_i}=1$.
This observation provides a simple baseline-free variance reduction idea. One can center outcomes by an action-common baseline without changing the minimizer. 

In the full feedback setting, a symmetric choice is $c_i=-\frac{1}{K}\sum_{a=1}^K Y_i\p{a}$, 
which centers the outcomes across actions for each unit.
In missing outcome settings, one may use a regression estimate $c\p{x}$ of an action-common component of the outcome and replace pseudo-outcomes by centered pseudo-outcomes, which can reduce the variance of IPW and DR constructions without changing the optimal policy.

\subsection{General Bayes Posterior}
Let $\delta_{\theta}\p{x}\in\Delta_K$ be a parametric policy model indexed by $\theta\in\Theta$.
For the full-vector surrogate, define the loss
\begin{align}
\label{eq:loss-K-full}
&\ell^{\mathrm{Full}}_K\p{\theta;\p{x,y\p{1},\dots,y\p{K}}}\\
&=
\frac{1}{2}\sum_{a=1}^{K}\p{\frac{1}{\sqrt{\zeta}}y\p{a}-\sqrt{\zeta}\delta_{\theta,a}\p{x}}^2.\notag
\end{align}
The generalized posterior is
\[
\rmd \Pi_{\eta}\p{\theta\mid \calD}
\propto
\rmd \Pi\p{\theta}\exp\p{-\eta\sum^n_{i=1}\ell^{\mathrm{Full}}_K\p{\theta;z_i}}.
\]
As in the binary case, \eqref{eq:loss-K-full} corresponds to a Gaussian pseudo-likelihood with conditionally independent components, $Y\p{a}\mid X=x,\theta\sim\mathcal{N}\p{\zeta\delta_{\theta,a}\p{x},\zeta/\eta}$ for $a\in\cb{1,\dots,K}$.

\section{GBPLNet: A Neural Network Implementation}
\label{sec:computation}
The generalized posterior is typically intractable for flexible models.
Standard approximation approaches include MAP estimation, Gaussian approximations, and stochastic gradient Langevin dynamics (SGLD) \citep{Welling2011bayesianlearning}.
As one implementation example, we parameterize the bounded score $f_{\theta}\p{x}\in\sqb{-1,1}$ by a neural network with a tanh-squashed output,
$f_{w}\p{x}=\tanh\p{g_{w}\p{x}}$.
We refer to this implementation as GBPLNet.

Given a fitted score, we form a policy by $\delta_{w}\p{x}=\p{f_{w}\p{x}+1}/2$.
For evaluation in the full feedback setting, it is also convenient to use the deterministic decision rule $a_{w}\p{x}=\mathbbm{1}\sqb{f_{w}\p{x}\ge 0}$.
This deterministic rule selects the more likely action under the randomized policy $\delta_{w}\p{x}$.
Furthermore, for $K$ actions, a deterministic decision rule can be formed by $a_{\theta}\p{x}\in\arg\max_{a\in\cb{1,\dots,K}} \delta_{\theta,a}\p{x}$.
MAP computation corresponds to minimizing the negative log generalized posterior, which is the empirical surrogate loss plus the negative log prior.
Gaussian approximations and SGLD provide complementary ways to represent posterior uncertainty.
However, neural networks and tanh are not essential to the framework, and any model class that produces bounded scores can be used. Appendix~\ref{appdx:detail_gbpl} provides computational details, and Appendix~\ref{app:posterior-visualization} illustrates posterior score and welfare summaries.

\paragraph{Implementation summary}
The GBPL update is determined by a choice of empirical loss, a prior, and tuning parameters.
\begin{itemize}[noitemsep, topsep=0pt, leftmargin=0.40cm]
  \item Inputs are data, a parametric score or policy model, a prior $\Pi$, and tuning parameters $\zeta$ and $\eta$.
  \item In the full feedback setting, the loss is \eqref{eq:loss-binary} for binary actions and \eqref{eq:loss-K-full} for $K$ actions.
  \item In missing outcome settings, the loss is constructed as $\widehat{L}^{\mathrm{IPW}}_n\p{\theta}$ or $\widehat{L}^{\mathrm{DR}}_n\p{\theta}$.
  \item Outputs are the generalized posterior in \eqref{eq:gb-posterior} and derived decision rules, including a point estimate such as a MAP or approximate samples such as from SGLD.
\end{itemize}

\section{GBPL with Missing Outcomes}
\label{sec:missing}
We now consider a missing outcome setting where only the outcome under the chosen action is observed.
We focus on a standard observational study setting with bandit feedback.

\subsection{Observation Model and Assumptions}
We observe i.i.d. data $\cb{\p{X_i,A_i,Y_i}}_{i=1}^n$, where $A_i\in\cb{1,\dots,K}$ is the chosen action and $Y_i=Y_i\p{A_i}$ is the observed outcome.
Let $e_a\p{x}=\bbP\p{A=a\mid X=x}$ be the propensity score.
We impose standard conditions \citep{Rosenbaum1983centralrole}.

\begin{assumption}
\label{ass:unconfounded}
\textbf{Unconfoundedness.}
The potential outcomes are conditionally independent of the action given the context, $\p{Y\p{1},\dots,Y\p{K}}\perp A\mid X$.
\end{assumption}

\begin{assumption}
\label{ass:overlap}
\textbf{Overlap.}
There exists $\epsilon>0$ such that $e_a\p{X}\ge \epsilon$ almost surely for all $a\in\cb{1,\dots,K}$.
\end{assumption}

Under these assumptions, the welfare of a policy $\delta\colon\calX\to\Delta_K$ satisfies $V\p{\delta}=\bbE\sqb{\sum_{a=1}^K \delta_a\p{X}\bbE\sqb{Y\p{a}\mid X}}$. 

\subsection{IPW Pseudo-Outcomes}
If the propensity score $e_a\p{x}$ is known, or replaced by an estimator $\widehat e_a\p{x}$, IPW constructs pseudo-outcomes \citep{Horvitz1952generalization}
\begin{align}
\widetilde{Y}^{\mathrm{IPW}}_i\p{a}
=
\frac{\mathbbm{1}\sqb{A_i=a}Y_i}{\widehat e_a\p{X_i}}.
\label{eq:ipw-pseudo}
\end{align}
Using the full-vector surrogate, we define the IPW empirical loss for a parametric policy $\delta_{\theta}$ as $\widehat{L}^{\mathrm{IPW}}_n\p{\theta}
=
\sum^n_{i=1}\frac{1}{2}\sum_{a=1}^{K}\p{\frac{1}{\sqrt{\zeta}}\widetilde{Y}^{\mathrm{IPW}}_i\p{a}-\sqrt{\zeta}\delta_{\theta,a}\p{X_i}}^2$. 
The corresponding generalized posterior is $\rmd \Pi^{\mathrm{IPW}}_{\eta}\p{\theta\mid \calD}
\propto
\rmd \Pi\p{\theta}\exp\p{-\eta \widehat{L}^{\mathrm{IPW}}_n\p{\theta}}$. 

\begin{theorem}
\label{thm:ipw-equivalence}
Assume that $\widehat e_a\p{X_i}>0$ for every observation $i$ and action $a$. 
Then minimizing $\widehat{L}^{\mathrm{IPW}}_n\p{\theta}$ over $\theta$ is equivalent to maximizing the IPW welfare estimator with a quadratic penalty,
\[
\frac{1}{n}\sum^n_{i=1}\sum_{a=1}^K \delta_{\theta,a}\p{X_i}\widetilde{Y}^{\mathrm{IPW}}_i\p{a}
-
\frac{\zeta}{2}\frac{1}{n}\sum^n_{i=1}\sum_{a=1}^K \delta_{\theta,a}\p{X_i}^2.
\]
\end{theorem}

The relation is exact because $\widehat{L}^{\mathrm{IPW}}_n\p{\theta}$ equals a term independent of $\theta$ minus $n$ times the displayed criterion.
The proof is provided in Appendix~\ref{app:proof-ipw-equivalence}.

\subsection{DR Pseudo-Outcomes}
Let $\gamma_a\p{x}=\bbE\sqb{Y\p{a}\mid X=x}$.
Let $\widehat\gamma_a\p{x}$ be an estimator of $\gamma_a\p{x}$.
The DR pseudo-outcome is given as follows \citep{Bang2005doublyrobust}:
\begin{align}
\widetilde{Y}^{\mathrm{DR}}_i\p{a}
=
\widehat\gamma_a\p{X_i}
+
\frac{\mathbbm{1}\sqb{A_i=a}\p{Y_i-\widehat\gamma_a\p{X_i}}}{\widehat e_a\p{X_i}}.
\label{eq:dr-pseudo}
\end{align}
Using the full-vector surrogate, define the DR empirical loss $\widehat{L}^{\mathrm{DR}}_n\p{\theta}
=
\sum^n_{i=1}\frac{1}{2}\sum_{a=1}^{K}\p{\frac{1}{\sqrt{\zeta}}\widetilde{Y}^{\mathrm{DR}}_i\p{a}-\sqrt{\zeta}\delta_{\theta,a}\p{X_i}}^2$. 
The generalized posterior is $\rmd \Pi^{\mathrm{DR}}_{\eta}\p{\theta\mid \calD}
\propto
\rmd \Pi\p{\theta}\exp\p{-\eta \widehat{L}^{\mathrm{DR}}_n\p{\theta}}$. 

However, in practice, $\widehat\gamma_a$ and $\widehat e_a$ may be estimated from the same data used to fit $\delta_{\theta}$.
To reduce bias from the empirical process, we usually assume Donsker conditions for the nuisance estimators, or apply sample splitting, also called cross-fitting \citep{Klaassen1987consistentestimation,VanderVaart2002semiparametricstatistics}. See \citet{Chernozhukov2018doubledebiased}, \citet{Zhou2023offlinemultiaction}, and \citet{Schuler2024introductionmodern} for details.

\subsection{Binary Action Case}
When $K=2$, the surrogate in Section~\ref{sec:binary} is expressed in terms of the outcome difference $U=Y\p{1}-Y\p{0}$.
Under missing outcomes, however, $U$ is not observed, but it can be replaced by an IPW or DR pseudo-difference.

For IPW, define $\widetilde{U}^{\mathrm{IPW}}_i
=
\frac{\mathbbm{1}\sqb{A_i=1}Y_i}{\widehat e_1\p{X_i}}
-
\frac{\mathbbm{1}\sqb{A_i=0}Y_i}{\widehat e_0\p{X_i}}$. 
For DR, define $\widetilde{U}^{\mathrm{DR}}_i
=
\Bigp{\widehat\gamma_1\p{X_i}-\widehat\gamma_0\p{X_i}} +
\frac{\mathbbm{1}\sqb{A_i=1}\p{Y_i-\widehat\gamma_1\p{X_i}}}{\widehat e_1\p{X_i}}
-
\frac{\mathbbm{1}\sqb{A_i=0}\p{Y_i-\widehat\gamma_0\p{X_i}}}{\widehat e_0\p{X_i}}$. 
Let $\widetilde{U}_i$ denote a generic pseudo-difference, with $\widetilde{U}_i=\widetilde{U}^{\mathrm{IPW}}_i$ for IPW and $\widetilde{U}_i=\widetilde{U}^{\mathrm{DR}}_i$ for DR.
Let $f_{\theta}\p{x}$ be a score model with $f_{\theta}\p{x}\in\sqb{-1,1}$ and define the binary loss by replacing $y\p{1}-y\p{0}$ in \eqref{eq:loss-binary} with $\widetilde{U}_i$.
For example, the DR empirical loss is $\widehat{L}^{\mathrm{DR}}_{n,\mathrm{bin}}\p{\theta}
=
\sum^n_{i=1}\frac{1}{2}\p{\frac{1}{\sqrt{\zeta}}\widetilde{U}^{\mathrm{DR}}_i-\sqrt{\zeta}f_{\theta}\p{X_i}}^2$, 
and the corresponding generalized posterior is defined as in \eqref{eq:gb-binary} with $\widehat{L}^{\mathrm{DR}}_{n,\mathrm{bin}}$.
The objective identity in Theorem \ref{thm:binary-equivalence} continues to hold after replacing $U_i$ by $\widetilde{U}_i$, see Appendix~\ref{app:proof-ipw-dr-binary}.

\subsection{Population Targets}
The following proposition summarizes the population interpretation of IPW and DR pseudo-outcomes.
Its proof is provided in Appendix~\ref{app:proof-dr-target}.

\begin{proposition}
\label{prop:dr-target}
Assume Assumptions \ref{ass:unconfounded} and \ref{ass:overlap}.
Fix an action $a$.
\begin{itemize}[noitemsep, topsep=0pt, leftmargin=0.40cm]
  \item If $\widehat e_a=e_a$ almost surely, then $\bbE\sqb{\widetilde{Y}^{\mathrm{IPW}}\p{a}\mid X}=\gamma_a\p{X}$.
  \item If either $\widehat e_a=e_a$ almost surely or $\widehat\gamma_a=\gamma_a$ almost surely, then $\bbE\sqb{\widetilde{Y}^{\mathrm{DR}}\p{a}\mid X}=\gamma_a\p{X}$.
\end{itemize}
Consequently, in either case the population minimizer of the full-vector squared-loss surrogate based on $\widetilde{Y}\p{a}$ matches the full feedback target, in the sense that the pointwise minimizer is the Euclidean projection of $\p{\gamma_1\p{x},\dots,\gamma_K\p{x}}/\zeta$ onto the simplex $\Delta_K$.
\end{proposition}

\subsection{Stabilization under Weak Overlap}
Small estimated propensities can produce large IPW and DR pseudo-outcomes. In the additional experiments, we clip the binary propensity estimate as
$\widehat e^{\mathrm{clip}}_1\p{x}
=
\max\p{\epsilon_{\mathrm{clip}},
\min\p{1-\epsilon_{\mathrm{clip}},\widehat e_1\p{x}}}$
and set
$\widehat e^{\mathrm{clip}}_0\p{x}
=
1-\widehat e^{\mathrm{clip}}_1\p{x}$.
The algebraic equivalences in Theorems~\ref{thm:binary-equivalence} and \ref{thm:ipw-equivalence} continue to hold after replacing the pseudo-outcomes by their clipped versions, because these equivalences follow from expansion of the squared loss. However, the population target in Proposition~\ref{prop:dr-target} need not be preserved exactly after clipping. 
Appendix~\ref{app:cf-exp} reports synthetic and UCI experiments with logged feedback using the IPW and DR losses.

\section{Theoretical Analysis}
\label{sec:theory}
This section summarizes theoretical guarantees for the surrogate loss and their implications for welfare.
Proofs are provided in the Appendix.

\subsection{Population Surrogate Minimizer}
\paragraph{Binary case. }
In the binary action case, define $U=Y\p{1}-Y\p{0}$.
For a measurable $f\colon\calX\to\sqb{-1,1}$, define the population surrogate risk $R_{\zeta}\p{f}
=
\bbE\sqb{
\frac{1}{2}\p{\frac{1}{\sqrt{\zeta}}U-\sqrt{\zeta}f\p{X}}^2
}$. 
Let $m\p{x}=\bbE\sqb{U\mid X=x}$.
Then the unconstrained minimizer satisfies $f^\star\p{x}=m\p{x}/\zeta$.
Under the constraint $f\p{x}\in\sqb{-1,1}$, the minimizer is the clipped conditional mean $f^\star\p{x}=\max\p{-1,\min\p{1,m\p{x}/\zeta}}$. Clipping preserves the sign of $m\p{x}$.

\paragraph{$K \geq 3$ case.}
For the full-vector surrogate, define $\gamma\p{x}=\p{\gamma_1\p{x},\dots,\gamma_K\p{x}}$ where $\gamma_a\p{x}=\bbE\sqb{Y\p{a}\mid X=x}$.
Consider the population risk associated with \eqref{eq:loss-K-full}.
Conditionally on $X=x$, minimizing the conditional risk over $\delta\in\Delta_K$ yields the Euclidean projection of $\gamma\p{x}/\zeta$ onto $\Delta_K$. Euclidean projection onto the simplex preserves coordinate order. Consequently, except at ties, the deterministic action induced by either population minimizer agrees with the plug-in Bayes rule. For deterministic action selection, the distinction is therefore in the loss-based posterior and the regularized finite-sample construction rather than in a different population action ranking. A proof is given in Appendix~\ref{app:proof-population}.

\subsection{PAC-Bayes Bound under a Moment Condition}
We state a PAC-Bayes bound for the surrogate loss under a sub-exponential moment condition, following unbounded-loss PAC-Bayes analyses \citep{Grunwald2020fastrates,Haddouche2021pacbayes,Alquier2024userfriendly}.
Let $\widehat{R}_{\zeta}\p{\theta}=\frac{1}{n}\sum^n_{i=1}\ell\p{\theta;z_i}$ and $R_{\zeta}\p{\theta}=\bbE\sqb{\ell\p{\theta;Z}}$.

\begin{assumption}
\label{ass:subexp}
There exist constants $v>0$ and $b>0$ such that, for every $\theta\in\Theta$, the centered loss $\ell\p{\theta;Z}-R_{\zeta}\p{\theta}$ satisfies
$\log\bbE\sqb{\exp\p{t\p{\ell\p{\theta;Z}-R_{\zeta}\p{\theta}}}}\leq t^2v/2$
for every $t$ such that $\lvert t\rvert<1/b$.
\end{assumption}

Assumption \ref{ass:subexp} is a moment condition on the loss.
In the full feedback binary setting with bounded scores, it holds under standard tail conditions on the outcome difference $U=Y\p{1}-Y\p{0}$.
In missing outcome settings, the IPW and DR pseudo-outcomes in \eqref{eq:ipw-pseudo} and \eqref{eq:dr-pseudo} involve inverse propensity weights, so the constants $v$ and $b$ can deteriorate as the overlap constant $\epsilon$ in Assumption \ref{ass:overlap} decreases.
Moreover, when nuisance estimators are learned from the same sample, the empirical losses are no longer i.i.d., and PAC-Bayes analyses typically rely on sample splitting or cross-fitting to control this dependence.
In heavy-tailed regimes, one can replace the squared surrogate by a robust loss and apply generalized Bayesian updating with the modified loss.

\begin{theorem}
\label{thm:pacbayes}
Assume Assumption~\ref{ass:subexp}.
Fix $\rho\in\p{0,1}$ and $t\in\p{0,1/b}$.
With probability at least $1-\rho$ over $\calD$, every probability distribution $Q$ on $\Theta$ with $\KL\p{Q\|\Pi}<\infty$ satisfies
$\bbE_{\theta\sim Q}\sqb{R_{\zeta}\p{\theta}}
\leq
\bbE_{\theta\sim Q}\sqb{\widehat{R}_{\zeta}\p{\theta}}
+
\frac{\KL\p{Q\|\Pi}+\log\p{1/\rho}}{tn}
+
\frac{tv}{2}$.
\end{theorem}
A proof is provided in Appendix~\ref{app:proof-pacbayes}.

\begin{corollary}
\label{cor:pacbayes-twosided}
Assume Assumption~\ref{ass:subexp}.
Fix $\rho\in\p{0,1}$ and $t\in\p{0,1/b}$.
With probability at least $1-\rho$ over $\calD$, every probability distribution $Q$ on $\Theta$ with $\KL\p{Q\|\Pi}<\infty$ satisfies
$\abs{
\bbE_{\theta\sim Q}\sqb{R_{\zeta}\p{\theta}}
-
\bbE_{\theta\sim Q}\sqb{\widehat{R}_{\zeta}\p{\theta}}
}
\leq
\frac{\KL\p{Q\|\Pi}+\log\p{2/\rho}}{tn}
+
\frac{tv}{2}$.
\end{corollary}

A proof is provided in Appendix~\ref{app:proof-pacbayes-twosided}.

\subsection{Welfare Corollaries}
The squared-loss surrogate and penalized welfare are linked by an identity. The proofs in this subsection are provided in Appendices \ref{app:proof-welfare-cor-binary}--\ref{app:proof-welfare-missing}.

\paragraph{Binary case.}
Define the population penalized welfare $W_{\lambda}\p{\delta}=V\p{\delta}-\lambda\bbE\sqb{\p{2\delta\p{X}-1}^2}$. 
Let $\lambda=\zeta/4$ and let $f=2\delta-1$.
Then $W_{\zeta/4}\p{\delta}$ differs from $R_{\zeta}\p{f}$ only through an additive constant that does not depend on $\delta$.
The next corollary makes the link explicit.

\begin{corollary}
\label{cor:welfare-risk-binary}
Let $\lambda=\zeta/4$.
For any two policies $\delta_1,\delta_2$ in the binary action setting, let $f_j=2\delta_j-1$.
Then, we have $W_{\lambda}\p{\delta_1}-W_{\lambda}\p{\delta_2}
=
\frac{1}{2}\p{R_{\zeta}\p{f_2}-R_{\zeta}\p{f_1}}$. 
In particular, if $\delta^\star_{\lambda}$ maximizes $W_{\lambda}$ over a class and $f^\star_{\zeta}$ minimizes $R_{\zeta}$ over the induced class, then we have $W_{\lambda}\p{\delta^\star_{\lambda}}-W_{\lambda}\p{\delta}
=
\frac{1}{2}\p{R_{\zeta}\p{f}-R_{\zeta}\p{f^\star_{\zeta}}}$. 
\end{corollary}

\begin{corollary}
\label{cor:welfare-unregularized}
Let $\lambda=\zeta/4$.
For any policy $\delta$ in the binary action setting, we have $W_{\lambda}\p{\delta}\le V\p{\delta}\le W_{\lambda}\p{\delta}+\lambda$. 
Consequently, if $\delta^\star$ maximizes $V$ over a class and $\delta^\star_{\lambda}$ maximizes $W_{\lambda}$ over the same class, then we have $V\p{\delta^\star}-V\p{\delta}
\le
W_{\lambda}\p{\delta^\star_{\lambda}}-W_{\lambda}\p{\delta}+\lambda$. 
\end{corollary}

\paragraph{$K \geq 3$ case.}
For the full-vector surrogate, define the penalized welfare $W^{\mathrm{Full}}_{\lambda}\p{\delta}
=
V\p{\delta}
-
\lambda\bbE\sqb{\sum_{a=1}^K \delta_a\p{X}^2}$. 
The next corollary parallels Corollary \ref{cor:welfare-risk-binary}.

\begin{corollary}
\label{cor:welfare-risk-K}
Let $\lambda=\zeta/2$ and consider the full-vector surrogate risk associated with \eqref{eq:loss-K-full}.
For any two policies $\delta_1,\delta_2$, we have $W^{\mathrm{Full}}_{\lambda}\p{\delta_1}-W^{\mathrm{Full}}_{\lambda}\p{\delta_2}
=
\p{R^{\mathrm{Full}}_{\zeta}\p{\delta_2}-R^{\mathrm{Full}}_{\zeta}\p{\delta_1}}$, 
where $R^{\mathrm{Full}}_{\zeta}\p{\delta}$ is the population risk induced by \eqref{eq:loss-K-full} after substituting $\delta$ for $\delta_{\theta}$.
\end{corollary}

\begin{corollary}
\label{cor:welfare-missing}
Assume Assumptions \ref{ass:unconfounded} and \ref{ass:overlap}.
Consider the full-vector surrogate loss $\widehat{L}^{\mathrm{IPW}}_n\p{\theta}$ or $\widehat{L}^{\mathrm{DR}}_n$, and assume the conditional mean correctness property in Proposition \ref{prop:dr-target}, so that $\bbE\sqb{\widetilde{Y}\p{a}\mid X}=\gamma_a\p{X}$ for all $a$.
Then the welfare identity in Corollary \ref{cor:welfare-risk-K} continues to hold after replacing $R^{\mathrm{Full}}_{\zeta}$ by the population risk defined using $\widetilde{Y}\p{a}$.
In particular, for $\lambda=\zeta/2$ and any two policies $\delta_1,\delta_2$,
 we have $W^{\mathrm{Full}}_{\lambda}\p{\delta_1}-W^{\mathrm{Full}}_{\lambda}\p{\delta_2}
=
\p{\widetilde{R}^{\mathrm{Full}}_{\zeta}\p{\delta_2}-\widetilde{R}^{\mathrm{Full}}_{\zeta}\p{\delta_1}}$, 
where $\widetilde{R}^{\mathrm{Full}}_{\zeta}$ denotes the surrogate risk computed with $\widetilde{Y}\p{a}$.
\end{corollary}

\subsection{PAC-Bayes to Welfare}
Combining Theorem \ref{thm:pacbayes} with Corollaries \ref{cor:welfare-risk-binary} and \ref{cor:welfare-risk-K} yields welfare bounds.
For example, in the binary case, we have $W_{\lambda}\p{\delta^\star_{\lambda}}-W_{\lambda}\p{\delta}
=
\frac{1}{2}\p{R_{\zeta}\p{f}-R_{\zeta}\p{f^\star_{\zeta}}}$. 
Thus, an upper bound on the excess surrogate risk $R_{\zeta}\p{f}-R_{\zeta}\p{f^\star_{\zeta}}$ yields an upper bound on penalized-welfare regret relative to the best-in-class penalized policy.
Furthermore, the same translation applies in missing outcome settings when the conditional mean correctness property in Proposition~\ref{prop:dr-target} holds after replacing $Y\p{a}$ by IPW or DR pseudo-outcomes. Theorem~\ref{thm:ipw-equivalence} gives the corresponding empirical optimization identity.

\begin{table}[ht]
\centering
\caption{Synthetic binary results, $100$ trials. Columns report welfare mean, welfare variance, welfare standard error, regret mean, and regret standard error across trials.}
\label{tab:synthetic-binary}
\resizebox{\linewidth}{!}{
\begin{tabular}{llcrrrrr}
\toprule
DGP & method & $\zeta$ & welfare\_mean & welfare\_var & welfare\_se & regret\_mean & regret\_se \\
\midrule
\multirow{7}{*}{DGP1} & DiffReg & -- & 1.0272 & 0.0081 & 0.0090 & 0.2291 & 0.0033 \\
\cmidrule(lr){2-8}
 & \multirow{4}{*}{GBPLNet} & 0.01 & 1.0240 & 0.0084 & 0.0091 & 0.2323 & 0.0034 \\
 &  & 0.1 & 1.0253 & 0.0081 & 0.0090 & 0.2310 & 0.0033 \\
 &  & 1.0 & 1.0298 & 0.0083 & 0.0091 & 0.2265 & 0.0034 \\
 &  & CV & 1.0295 & 0.0083 & 0.0091 & 0.2268 & 0.0033 \\
\cmidrule(lr){2-8}
 & PluginReg & -- & 1.0265 & 0.0085 & 0.0092 & 0.2298 & 0.0034 \\
\cmidrule(lr){2-8}
 & WeightedLogistic & -- & 1.0331 & 0.0083 & 0.0091 & 0.2232 & 0.0034 \\
\midrule
\multirow{7}{*}{DGP2} & DiffReg & -- & 0.6810 & 0.0038 & 0.0061 & 0.3368 & 0.0029 \\
\cmidrule(lr){2-8}
 & \multirow{4}{*}{GBPLNet} & 0.01 & 0.7165 & 0.0040 & 0.0063 & 0.3013 & 0.0030 \\
 &  & 0.1 & 0.7192 & 0.0037 & 0.0061 & 0.2986 & 0.0030 \\
 &  & 1.0 & 0.7108 & 0.0038 & 0.0061 & 0.3070 & 0.0032 \\
 &  & CV & 0.7126 & 0.0037 & 0.0061 & 0.3052 & 0.0029 \\
\cmidrule(lr){2-8}
 & PluginReg & -- & 0.6908 & 0.0039 & 0.0063 & 0.3270 & 0.0032 \\
\cmidrule(lr){2-8}
 & WeightedLogistic & -- & 0.6746 & 0.0041 & 0.0064 & 0.3432 & 0.0034 \\
\midrule
\multirow{7}{*}{DGP3} & DiffReg & -- & 0.4847 & 0.0046 & 0.0068 & 0.3019 & 0.0032 \\
\cmidrule(lr){2-8}
 & \multirow{4}{*}{GBPLNet} & 0.01 & 0.4847 & 0.0048 & 0.0070 & 0.3019 & 0.0031 \\
 &  & 0.1 & 0.4839 & 0.0049 & 0.0070 & 0.3026 & 0.0034 \\
 &  & 1.0 & 0.4875 & 0.0048 & 0.0069 & 0.2991 & 0.0032 \\
 &  & CV & 0.4878 & 0.0046 & 0.0068 & 0.2988 & 0.0032 \\
\cmidrule(lr){2-8}
 & PluginReg & -- & 0.4851 & 0.0047 & 0.0069 & 0.3015 & 0.0031 \\
\cmidrule(lr){2-8}
 & WeightedLogistic & -- & 0.4863 & 0.0045 & 0.0067 & 0.3003 & 0.0031 \\
\bottomrule
\end{tabular}
}
\end{table}

\begin{figure}[ht]
\centering
\includegraphics[width=0.95\linewidth]{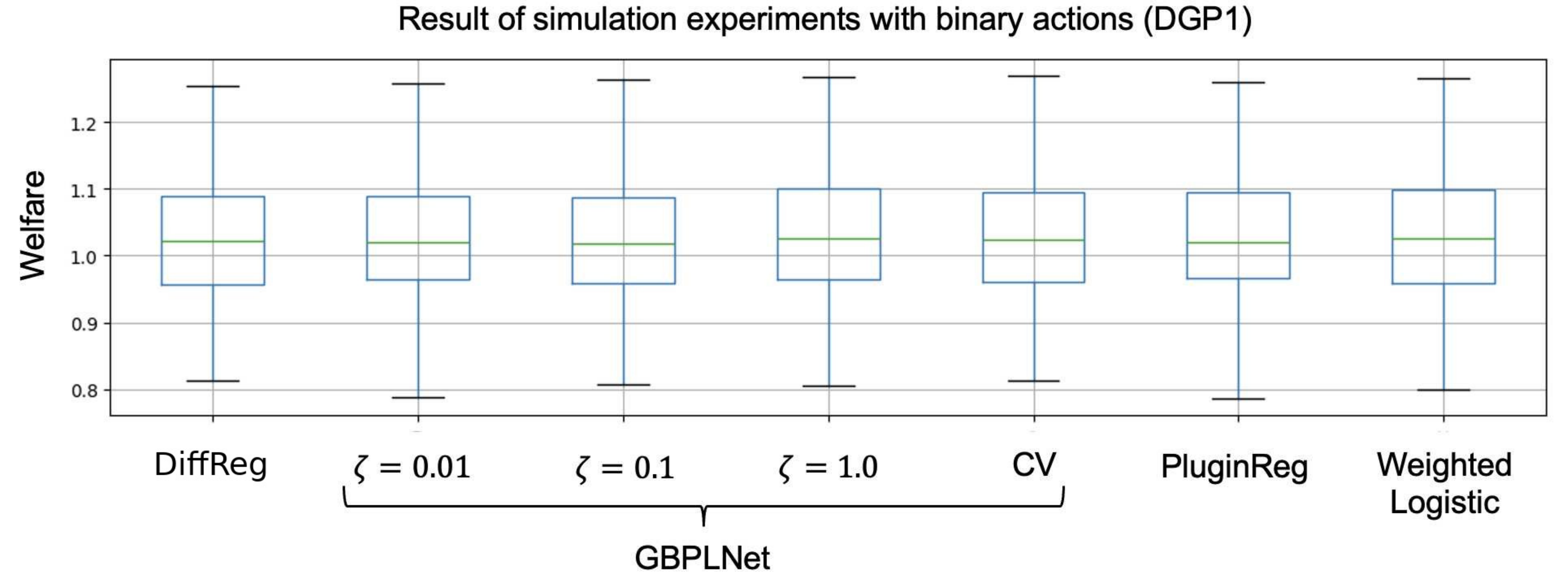}
\includegraphics[width=0.95\linewidth]{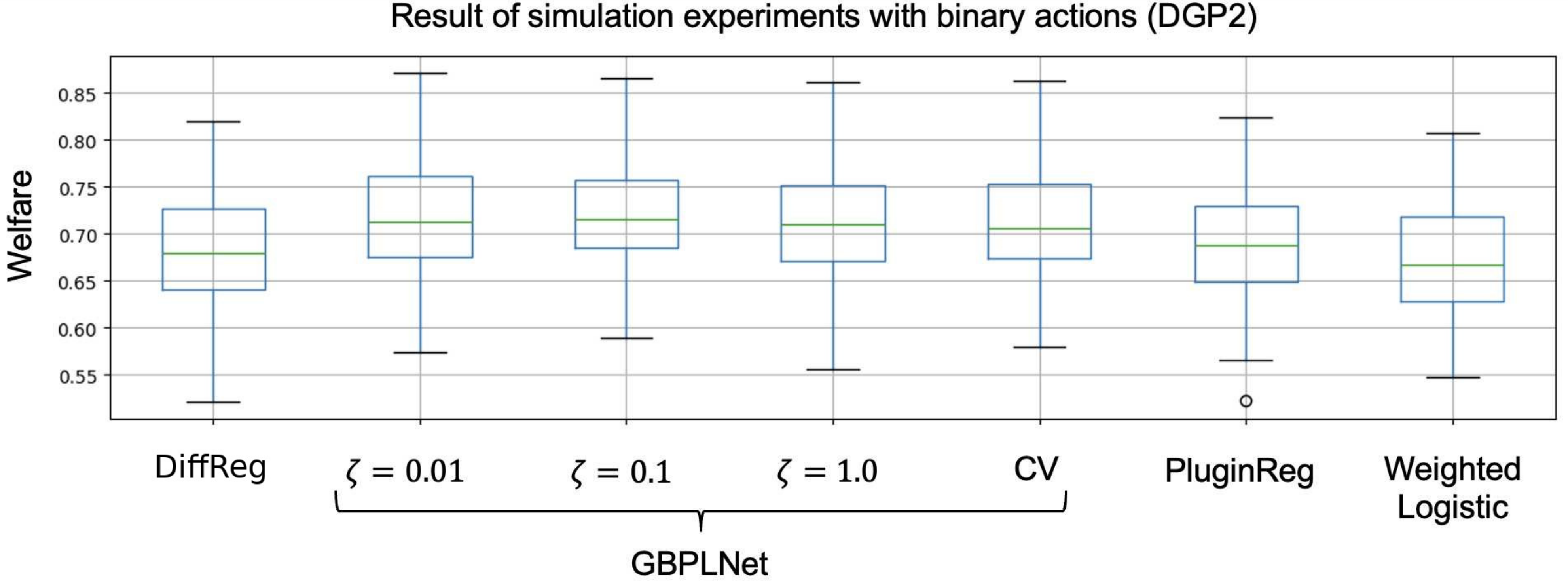}
\includegraphics[width=0.95\linewidth]{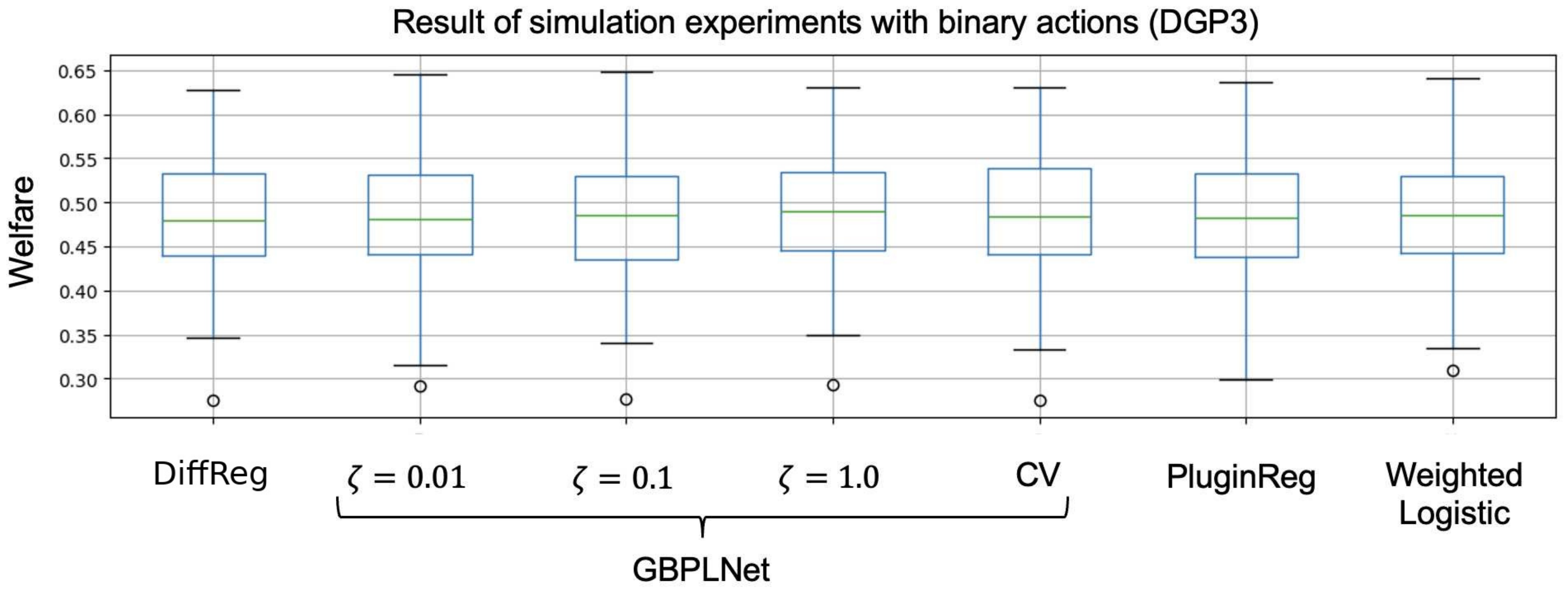}
\caption{Synthetic binary welfare boxplots across $100$ trials for DGP1, DGP2, and DGP3.}
\label{fig:synthetic-binary-boxplots}
\end{figure}

\section{Experiments}
\label{sec:experiments}
This section evaluates GBPLNet as an implementation example of the proposed General Bayes policy learning framework.
Throughout, we focus on the full feedback setting, in which each observation contains the entire outcome vector $\p{Y\p{1},\dots,Y\p{K}}$.
We report welfare and regret computed on held-out test data, and we aggregate results over $100$ independent trials. In this section, we only show the results with binary actions using simulation datasets. The experimental details and additional results are shown in Appendices~\ref{app:additional-posterior-experiments}--\ref{app:cf-exp}. 

\paragraph{Evaluation protocol}
For each trial, we generate or load a dataset and split it into training, validation, and test sets.
We train each method on the training set and select tuning parameters using the validation set.
We then evaluate welfare on the test set using the empirical welfare formulas in \eqref{eq:empwelfare} and their analogs for $K$ actions.
For binary actions, we report the realized test welfare of the deterministic policy induced by the learned score function $f$, $\widehat{\delta}\p{x}=\mathbbm{1}\sqb{f\p{x}\ge 0}$.
Moreover, for $K$ actions, we evaluate the learned stochastic policy by the test-sample average $\frac{1}{n_{\mathrm{test}}}\sum_i\sum_{a=1}^{K}\widehat{\delta}_a\p{X_i}Y_i\p{a}$.

We define test regret relative to the realized-outcome oracle policy, which maximizes the realized outcome for each test observation.
In the binary case, the oracle action is $\delta^\mathrm{orc}\p{X_i}\in\arg\max_{a\in\cb{0,1}} Y_i\p{a}$.
We report the empirical regret on the test set, $\widehat{V}_\mathrm{test}\p{\delta^\mathrm{orc}}-\widehat{V}_\mathrm{test}\p{\widehat{\delta}}$, and refer to it as realized-outcome regret. Because the binary experiments use deterministic rules while the multi-action experiments evaluate stochastic policies, their regret values are not directly comparable.
For each scenario and method, the tables report the mean, variance, and standard error of welfare across $100$ trials, and the mean and standard error of regret across $100$ trials.

We compare four values of the surrogate tuning parameter $\zeta$ for the proposed method.
Three are fixed at $\zeta\in\cb{1,0.1,0.01}$.
One is selected by validation from the grid $\cb{1,0.1,0.01,0.001}$, and is reported as GBPLNet (CV).

For GBPLNet (CV), we select $\zeta$ by maximizing validation welfare rather than minimizing the validation surrogate loss. 
In the binary case, expanding the squared loss in \eqref{eq:loss-binary} yields the term $\frac{1}{2\zeta}\frac{1}{n}\sum^n_{i=1}U_i^2$, where $U_i=Y_i\p{1}-Y_i\p{0}$, which does not depend on the fitted score, see Proposition \ref{prop:gbplnet-decomposition} in Appendix~\ref{appdx:detail_gbpl}.
If the surrogate loss value itself is used for validation, this term can bias selection toward larger $\zeta$ values.

\paragraph{Synthetic Experiments}
We consider three synthetic data generating processes, denoted by DGP1, DGP2, and DGP3.
Each trial uses $n=5000$ observations and $d=10$ features.
Table~\ref{tab:synthetic-binary} reports welfare and regret, and Figure~\ref{fig:synthetic-binary-boxplots} visualizes the distribution of welfare. We consider cases with $K = 2$. For the results with $K = 5$, see Appendix~\ref{appdx:additionalsim}.

The comparison methods are as follows.
DiffReg fits a regression model for the outcome difference $Y\p{1}-Y\p{0}$ and chooses actions by its sign.
PluginReg fits separate regression models for $Y\p{1}$ and $Y\p{0}$ and chooses the action with the larger predicted outcome.
WeightedLogistic fits a weighted logistic classifier for the best action, with weights given by the observed outcome gap magnitude.
GBPLNet trains the tanh-squashed neural score model in Section~\ref{sec:computation} using the squared-loss surrogate.

Table~\ref{tab:synthetic-binary} shows that relative performance is scenario dependent.
In DGP1 and DGP3, the methods have similar welfare. However, in DGP2, GBPLNet improves on DiffReg and PluginReg. WeightedLogistic attains the highest mean welfare in DGP1.
The sensitivity to $\zeta$ is also scenario dependent.

\section{Conclusion}
This study proposes a General Bayes framework for policy learning (GBPL). The key technical device is a squared-loss surrogate that rewrites welfare maximization as a regression-style objective, with a Gaussian pseudo-likelihood interpretation. We demonstrate its implementation in both full-feedback and bandit-feedback settings. In our theoretical analysis, we give PAC-Bayes bounds for the surrogate risk and explicit corollaries for penalized welfare.

\bibliography{UAI.bbl}

\onecolumn

\appendix

\section{Related Work}
\label{app:related}

This appendix reviews related work and clarifies how the proposed GBPL framework is connected to existing lines of work.
Our focus is on loss-based Bayesian updating and decision-theoretic learning rules for welfare objectives.
Policy learning in causal inference and offline learning provides an important application area, but our framework is not tied to counterfactual semantics.

\subsection{Loss-Based Bayes, Gibbs Posteriors, and Learning Rate Selection}
Generalized Bayesian updating, also called loss-based Bayes, replaces a likelihood by a task loss and yields a Gibbs-type posterior distribution \citep{Bissiri2016generalframework}.
This update is closely connected to the PAC-Bayes literature \citep{McAllester1999pacbayes,Seeger2002pacbayes,Catoni2008pacbayesian} and to exponential weighting and aggregation methods for prediction under squared loss \citep{Leung2006information,Dalalyan2008aggregation,Audibert2009fast}.
From this viewpoint, the generalized posterior is an exponentially weighted distribution over decision rules, where the prior acts as a regularizer through a KL penalty \citep{Bissiri2016generalframework}.

A central practical issue is the choice of the learning rate, since the loss scale is not identified by probability axioms \citep{Bissiri2016generalframework}.
Several calibration principles have been developed, including the loss-likelihood bootstrap \citep{Lyddon2019generalbayesian}, generalized posterior calibration \citep{Syring2019calibrating}, and comparative studies of learning rate selection criteria in misspecified settings \citep{Wu2023comparison}.
Moreover, Safe Bayes is a data-driven approach that targets robustness to model misspecification \citep{Grunwald2012thesafe}.

Furthermore, when the loss is a negative log-likelihood, the generalized posterior becomes a tempered likelihood posterior, which connects General Bayes to power posteriors and fractional posteriors \citep{HolmesWalker2017powerlikelihood,Bhattacharya2019fractionalposterior}.
For the squared-loss surrogate, the Gaussian pseudo-likelihood produces the same exponential weight as the General Bayes update, and Proposition~\ref{prop:variational} gives its loss-based variational characterization \citep{Bissiri2016generalframework}. 

Finally, when exact sampling is infeasible, variational approximations provide tractable alternatives with theoretical support in the PAC-Bayes literature \citep{Alquier2016onthe}.
Optimization-centric viewpoints also emphasize that generalized Bayesian procedures can be studied as regularized empirical risk minimization problems over distributions \citep{Knoblauch2022optimizationcentric}.
For unbounded losses, including squared loss with unbounded outcomes, recent work develops fast-rate guarantees and learning-theoretic analyses for generalized Bayes under suitable moment conditions \citep{Grunwald2020fastrates}.

\subsection{Bayesian Decision Making and Predictive Decision Analysis}
Bayesian inference is naturally connected to decision making through the Bayes act, which minimizes posterior expected loss, and through statistical decision theory more broadly \citep{DeGroot1970optimal,Berger1985statistical,Robert2007bayesianchoice}.
This connection motivates Bayesian predictive decision rules in which decisions are evaluated under a predictive distribution and the loss encodes the target task.
Recent work studies predictive decision analysis for welfare objectives in portfolio choice and other economic problems, including generalized Bayesian predictive decision rules \citep{Tallman2023bayesianpredictive,Tallman2024predictivedecision,Kato2024generalbayesian}. General Bayesian predictive synthesis combines predictive distributions supplied by experts through a loss-minimization framework, whereas GBPL updates a prior directly over score or policy parameters from outcome data.

Moreover, a related line in machine learning and operations research is decision-focused learning, which trains predictive components to minimize a downstream decision objective.
Representative examples include predict then optimize methods \citep{Elmachtoub2022smartpredict,Donti2017taskbased,Wilder2019endtoend}.
These works typically deliver point estimates of predictors or policies.
Our focus is complementary because we construct a Gibbs posterior over decision rules via loss-based Bayesian updating.

\subsection{Portfolio Choice}
Classical portfolio choice originates in mean-variance analysis \citep{Markowitz1952portfolioselection,Markowitz2000meanvarianceanalysis}.
Moreover, Bayesian portfolio rules based on predictive distributions have been studied, including early Bayesian treatments \citep{Winkler1975abayesian} and influential empirical Bayes and shrinkage approaches \citep{Jorion1986bayesstein,BlackLitterman1992global}.
Recent work revisits portfolio choice from a predictive decision perspective and connects posterior predictive distributions to portfolio welfare, including generalized Bayesian formulations \citep{Tallman2024predictivedecision,Kato2024bayesianportfolio,Kato2024generalbayesian}.

\subsection{Policy Learning with Counterfactuals and Offline Feedback}
In causal inference, treatment choice is a central application of welfare maximization \citep{Manski2002treatmentchoice,Manski2004statisticaltreatment,Stoye2009minimaxregret}.
Policy learning methods train flexible decision rules using counterfactual identification strategies, including DR and IPW objectives \citep{Dudik2011doublyrobust,Swaminathan2015batchlearning,Kitagawa2018whoshould,Athey2021policylearning,Zhou2023offlinemultiaction}.
Furthermore, in offline learning from logged bandit feedback, London and Sandler derive PAC-Bayesian regularization for counterfactual risk minimization, while Sakhi et al. optimize PAC-Bayesian guarantees by representing policies as mixtures of decision rules \citep{London2019bayesiancounterfactual,Sakhi2023pacbayesianoffline}.
These literatures motivate practical empirical welfare objectives when outcomes are not fully observed.
In contrast, GBPL constructs a Gibbs posterior from squared welfare surrogates and uses the same loss-based update in full-feedback and pseudo-outcome settings.

\subsection{Classification-Based Surrogates and Multi-Action Extensions}
A large literature reduces welfare maximization to classification or regression by designing surrogate losses that target welfare directly.
Outcome-weighted learning casts individualized treatment rules as weighted classification and motivates surrogate losses for tractable estimation \citep{Zhao2012outcomeweighted,Zhang2012optimalclassification}.
Multi-action extensions must account for all actions simultaneously. A baseline-gap formulation provides a convenient reduction in our construction.
Our baseline-free full-vector surrogate in Section~\ref{sec:K} is motivated by the desire to retain symmetry across actions while keeping a squared-loss structure that supports Gaussian pseudo-likelihood computation.

\subsection{PAC-Bayes Bounds and Generalization Guarantees}
PAC-Bayes bounds control the generalization error of randomized estimators and apply naturally to Gibbs posteriors \citep{McAllester1999pacbayes,Seeger2002pacbayes,Catoni2008pacbayesian}.
Recent work develops user-friendly PAC-Bayes bounds and extensions to unbounded losses under moment and tail conditions \citep{Haddouche2021pacbayes,Alquier2024userfriendly,Grunwald2020fastrates}.
Our theoretical results in Section~\ref{sec:theory} follow this line and provide corollaries that translate surrogate risk bounds into welfare guarantees.

\section{Proof of Proposition \ref{prop:variational}}
\label{app:proof-variational}
\begin{proof}
Let $L\p{\theta}=\sum^n_{i=1}\ell\p{\theta;z_i}$ and let $Z_{\eta}\p{\calD}=\int \exp\p{-\eta L\p{\theta}}\rmd \Pi\p{\theta}$.
Define $\Pi_{\eta}\p{\rmd\theta\mid\calD}=\exp\p{-\eta L\p{\theta}}Z_{\eta}\p{\calD}^{-1}\Pi\p{\rmd\theta}$.
For any $Q$ such that $Q\ll \Pi$,
\[
\KL\p{Q\|\Pi_{\eta}\p{\cdot\mid\calD}}
=
\bbE_{\theta\sim Q}\sqb{\log\p{\frac{\rmd Q}{\rmd \Pi}}}
+
\eta\bbE_{\theta\sim Q}\sqb{L\p{\theta}}
+
\log Z_{\eta}\p{\calD}.
\]
Rearranging gives
\[
\calJ\p{Q}
=
\KL\p{Q\|\Pi_{\eta}\p{\cdot\mid\calD}}
-
\log Z_{\eta}\p{\calD}.
\]
Since the last term does not depend on $Q$, $\calJ\p{Q}$ is minimized uniquely at $Q=\Pi_{\eta}\p{\cdot\mid\calD}$.
\end{proof}

\section{Proof of Theorem \ref{thm:binary-equivalence}}
\label{app:proof-binary}
\begin{proof}
Write $U_i=Y_i\p{1}-Y_i\p{0}$.
Expanding the objective in \eqref{eq:surrogate-empirical} gives
\[
\frac{1}{n}\sum^n_{i=1}\p{\frac{1}{\sqrt{\zeta}}U_i-\sqrt{\zeta}f\p{X_i}}^2
=
\frac{1}{n}\sum^n_{i=1}\p{\frac{1}{\zeta}U_i^2-2U_if\p{X_i}+\zeta f\p{X_i}^2}.
\]
The first term does not depend on $f$.
Thus minimizing the surrogate is equivalent to maximizing
\[
\frac{1}{n}\sum^n_{i=1}\p{U_if\p{X_i}-\frac{\zeta}{2}f\p{X_i}^2}.
\]
Substituting $f\p{x}=2\delta\p{x}-1$ and using \eqref{eq:empwelfare} yields the penalized welfare criterion in \eqref{eq:empwelfare-penalized} with $\lambda=\zeta/4$, up to constants that do not depend on $\delta$.

\paragraph{Limit as \texorpdfstring{$\zeta\to 0$}{zeta}.}
Let $\lambda_m\to 0$.
For any $\delta\in\calH_{\mathrm{pol}}$,
\[
\widehat{V}\p{\widetilde{\delta}\p{\lambda_m}}
-
\lambda_m\frac{1}{n}\sum^n_{i=1}\p{2\widetilde{\delta}\p{\lambda_m}\p{X_i}-1}^2
\ge
\widehat{V}\p{\delta}
-
\lambda_m\frac{1}{n}\sum^n_{i=1}\p{2\delta\p{X_i}-1}^2.
\]
Since $\p{2\delta\p{X_i}-1}^2\le 1$, taking $m\to\infty$ yields
\[
\liminf_{m\to\infty}\widehat{V}\p{\widetilde{\delta}\p{\lambda_m}}
\ge
\widehat{V}\p{\delta}.
\]
Therefore, any limit point of $\widetilde{\delta}\p{\lambda_m}$ is a maximizer of $\widehat{V}$.
If the maximizer is unique, convergence follows.
\end{proof}

\section{Proofs in Section~\ref{sec:K} (Policy Learning with Multiple Actions)}
\subsection{Proof of Theorem \ref{thm:K-equivalence-gap}}
\label{app:proof-K-gap}
\begin{proof}
For each $i$ and each $a\in\cb{1,\dots,K-1}$, let $\overline{U}_i\p{a}=Y_i\p{a}-Y_i\p{K}$.
Expanding the objective in \eqref{eq:surrogate-K-gap} yields
\[
\frac{1}{n}\sum^n_{i=1}\sum_{a=1}^{K-1}\p{\frac{1}{\sqrt{\zeta}}\overline{U}_i\p{a}-\sqrt{\zeta}f_a\p{X_i}}^2
=
\frac{1}{n}\sum^n_{i=1}\sum_{a=1}^{K-1}\p{\frac{1}{\zeta}\overline{U}_i\p{a}^2-2\overline{U}_i\p{a}f_a\p{X_i}+\zeta f_a\p{X_i}^2}.
\]
Dropping constants and rescaling, minimizing is equivalent to maximizing
\[
\frac{1}{n}\sum^n_{i=1}\sum_{a=1}^{K-1}\p{\overline{U}_i\p{a}f_a\p{X_i}-\frac{\zeta}{2}f_a\p{X_i}^2}.
\]
Substituting $f_a\p{x}=2\delta_a\p{x}-1$ yields the penalized welfare criterion in \eqref{eq:penalized-K-gap} with $\lambda=\zeta/4$, up to constants that do not depend on $\delta$.
\end{proof}

\subsection{Proof of Theorem \ref{thm:K-equivalence-full}}
\label{app:proof-K-full}
\begin{proof}
Expanding the objective in \eqref{eq:surrogate-K-full} yields
\[
\frac{1}{n}\sum^n_{i=1}\sum_{a=1}^{K}\p{\frac{1}{\sqrt{\zeta}}Y_i\p{a}-\sqrt{\zeta}\delta_a\p{X_i}}^2
=
\frac{1}{n}\sum^n_{i=1}\sum_{a=1}^{K}\p{\frac{1}{\zeta}Y_i\p{a}^2-2Y_i\p{a}\delta_a\p{X_i}+\zeta \delta_a\p{X_i}^2}.
\]
Dropping the constant term and rescaling, minimizing is equivalent to maximizing
\[
\frac{1}{n}\sum^n_{i=1}\sum_{a=1}^{K}\p{Y_i\p{a}\delta_a\p{X_i}-\frac{\zeta}{2}\delta_a\p{X_i}^2},
\]
which is \eqref{eq:penalized-K-full} with $\lambda=\zeta/2$.
\end{proof}

\section{Proofs in Section~\ref{sec:missing} (Policy Learning with Missing Outcomes)}
\subsection{Proof of Theorem \ref{thm:ipw-equivalence}}
\label{app:proof-ipw-equivalence}
\begin{proof}
Expanding $\widehat{L}^{\mathrm{IPW}}_n\p{\theta}$ gives
\[
\sum^n_{i=1}\frac{1}{2}\sum_{a=1}^{K}\p{\frac{1}{\sqrt{\zeta}}\widetilde{Y}^{\mathrm{IPW}}_i\p{a}-\sqrt{\zeta}\delta_{\theta,a}\p{X_i}}^2
=
\sum^n_{i=1}\frac{1}{2}\sum_{a=1}^{K}\p{\frac{1}{\zeta}\widetilde{Y}^{\mathrm{IPW}}_i\p{a}^2-2\widetilde{Y}^{\mathrm{IPW}}_i\p{a}\delta_{\theta,a}\p{X_i}+\zeta\delta_{\theta,a}\p{X_i}^2}.
\]
Thus, $\widehat{L}^{\mathrm{IPW}}_n\p{\theta}$ equals a term independent of $\theta$ minus $n$ times the criterion in Theorem~\ref{thm:ipw-equivalence}.
\end{proof}

\subsection{Proof of Proposition \ref{prop:dr-target}}
\label{app:proof-dr-target}
\begin{proof}
Fix $a$.
For IPW with $\widehat e_a=e_a$,
\[
\bbE\sqb{\widetilde{Y}^{\mathrm{IPW}}\p{a}\mid X}
=
\bbE\sqb{\frac{\mathbbm{1}\sqb{A=a}Y\p{A}}{e_a\p{X}}\mid X}
=
\bbE\sqb{\bbE\sqb{\frac{\mathbbm{1}\sqb{A=a}Y\p{a}}{e_a\p{X}}\mid X,A}\mid X}.
\]
Under unconfoundedness, $\bbE\sqb{Y\p{a}\mid X,A}=\bbE\sqb{Y\p{a}\mid X}=\gamma_a\p{X}$.
Thus the inner conditional expectation equals $\mathbbm{1}\sqb{A=a}\gamma_a\p{X}/e_a\p{X}$ and taking expectation over $A\mid X$ yields $\gamma_a\p{X}$.

For DR, write
\[
\widetilde{Y}^{\mathrm{DR}}\p{a}
=
\widehat\gamma_a\p{X}
+
\frac{\mathbbm{1}\sqb{A=a}\p{Y\p{a}-\widehat\gamma_a\p{X}}}{\widehat e_a\p{X}}.
\]
Taking conditional expectation given $X$ and using unconfoundedness yields
\[
\bbE\sqb{\widetilde{Y}^{\mathrm{DR}}\p{a}\mid X}
=
\widehat\gamma_a\p{X}
+
\bbE\sqb{\frac{\mathbbm{1}\sqb{A=a}}{\widehat e_a\p{X}}\p{\gamma_a\p{X}-\widehat\gamma_a\p{X}}\mid X}.
\]
If $\widehat e_a=e_a$, then $\bbE\sqb{\mathbbm{1}\sqb{A=a}/\widehat e_a\p{X}\mid X}=1$ and the right-hand side equals $\gamma_a\p{X}$.
If $\widehat\gamma_a=\gamma_a$, then the second term is zero, and the right-hand side equals $\gamma_a\p{X}$.
The final statement about the projection target follows from the conditional risk minimization argument in Appendix~\ref{app:proof-population}.
\end{proof}

\section{Proofs in Section~\ref{sec:theory} (Theoretical Analysis)}

\subsection{Population Minimizers}
\label{app:proof-population}
\paragraph{Binary case.}
\begin{proof}
For any fixed $x$, minimizing
\[
\bbE\sqb{\frac{1}{2}\p{\frac{1}{\sqrt{\zeta}}U-\sqrt{\zeta}q}^2\mid X=x}
=
\frac{1}{2\zeta}\bbE\sqb{\p{U-\zeta q}^2\mid X=x}
\]
over $q\in\bbR$ yields $q=m\p{x}/\zeta$ where $m\p{x}=\bbE\sqb{U\mid X=x}$.
If $q$ is constrained to $\sqb{-1,1}$, the minimizer is the projection of $m\p{x}/\zeta$ onto $\sqb{-1,1}$, which is $\max\p{-1,\min\p{1,m\p{x}/\zeta}}$. Because $\zeta>0$, this projection has the same sign as $m\p{x}$, except when $m\p{x}=0$.
\end{proof}

\paragraph{$K\geq 3$ case.}
\begin{proof}
Fix $x$ and consider the conditional risk for $\delta\in\Delta_K$,
\[
\bbE\sqb{\frac{1}{2}\sum_{a=1}^K\p{\frac{1}{\sqrt{\zeta}}Y\p{a}-\sqrt{\zeta}\delta_a}^2\mid X=x}
=
\frac{1}{2\zeta}\sum_{a=1}^K \bbE\sqb{\p{Y\p{a}-\zeta\delta_a}^2\mid X=x}.
\]
Expanding the square yields a term that does not depend on $\delta$ and a quadratic term
\[
\frac{1}{2\zeta}\sum_{a=1}^K \p{\gamma_a\p{x}-\zeta\delta_a}^2
=
\frac{\zeta}{2}\sum_{a=1}^K\p{\delta_a-\gamma_a\p{x}/\zeta}^2.
\]
Thus the minimizer is the Euclidean projection of $\gamma\p{x}/\zeta$ onto $\Delta_K$. The simplex projection has coordinates $\max\p{\gamma_a\p{x}/\zeta-c,0}$ for a common threshold $c$, so it preserves coordinate order. Hence its argmax agrees with the argmax of $\gamma\p{x}$, except at ties.
\end{proof}

\subsection{PAC-Bayes bound}
\label{app:proof-pacbayes}
\begin{proof}
Fix $t\in\p{0,1/b}$.
For each $\theta\in\Theta$, Assumption~\ref{ass:subexp} applied with $-t$ gives
\[
\bbE_{\calD}\sqb{
\exp\p{
tn\p{R_{\zeta}\p{\theta}-\widehat{R}_{\zeta}\p{\theta}}
}
}
\leq
\exp\p{\frac{nt^2v}{2}}.
\]
Integrating this inequality with respect to $\Pi$ and applying Fubini's theorem gives
\[
\bbE_{\calD}\sqb{
\int
\exp\p{
tn\p{R_{\zeta}\p{\theta}-\widehat{R}_{\zeta}\p{\theta}}
}
\rmd\Pi\p{\theta}
}
\leq
\exp\p{\frac{nt^2v}{2}}.
\]
Markov's inequality implies that, with probability at least $1-\rho$,
\[
\int
\exp\p{
tn\p{R_{\zeta}\p{\theta}-\widehat{R}_{\zeta}\p{\theta}}
}
\rmd\Pi\p{\theta}
\leq
\frac{1}{\rho}\exp\p{\frac{nt^2v}{2}}.
\]
On this event, the variational inequality for KL divergence gives, for every $Q$ with $\KL\p{Q\|\Pi}<\infty$,
\[
tn\bbE_{\theta\sim Q}\sqb{
R_{\zeta}\p{\theta}-\widehat{R}_{\zeta}\p{\theta}
}
\leq
\KL\p{Q\|\Pi}
+
\log
\int
\exp\p{
tn\p{R_{\zeta}\p{\theta}-\widehat{R}_{\zeta}\p{\theta}}
}
\rmd\Pi\p{\theta}.
\]
Substituting the preceding bound and dividing by $tn$ yields the result.
\end{proof}

\subsection{Proof of Corollary \ref{cor:pacbayes-twosided}}
\label{app:proof-pacbayes-twosided}
\begin{proof}
Apply Theorem~\ref{thm:pacbayes} to $\ell$ with failure probability $\rho/2$ and the fixed parameter $t$.
This bounds $\bbE_{\theta\sim Q}\sqb{R_{\zeta}\p{\theta}}-\bbE_{\theta\sim Q}\sqb{\widehat{R}_{\zeta}\p{\theta}}$.
Apply the same theorem to $-\ell$ with failure probability $\rho/2$.
Assumption~\ref{ass:subexp} also holds for $-\ell$ because it is stated for every $t$ with $\lvert t\rvert<1/b$.
This yields an upper bound on $\bbE_{\theta\sim Q}\sqb{\widehat{R}_{\zeta}\p{\theta}}-\bbE_{\theta\sim Q}\sqb{R_{\zeta}\p{\theta}}$.
A union bound completes the proof.
\end{proof}

\subsection{Proof for binary missing-outcome objectives}
\label{app:proof-ipw-dr-binary}
\begin{proof}
The proof of Theorem \ref{thm:binary-equivalence} uses only algebraic expansion of the squared loss and does not rely on $U_i$ being the true outcome difference.
Therefore the same expansion yields the same optimization identity after replacing $U_i$ by any scalar pseudo-difference $\widetilde{U}_i$, including IPW and DR estimators.
\end{proof}

\subsection{Proof of Corollary \ref{cor:welfare-risk-binary}}
\label{app:proof-welfare-cor-binary}
\begin{proof}
Let $f=2\delta-1$.
Expanding $R_{\zeta}\p{f}$ yields
\[
R_{\zeta}\p{f}
=
\frac{1}{2\zeta}\bbE\sqb{U^2}-\bbE\sqb{Uf\p{X}}+\frac{\zeta}{2}\bbE\sqb{f\p{X}^2}.
\]
Since $Uf=2U\delta-U$, we have
\[
R_{\zeta}\p{f}
=
\frac{1}{2\zeta}\bbE\sqb{U^2}-2\bbE\sqb{U\delta\p{X}}+\bbE\sqb{U}+\frac{\zeta}{2}\bbE\sqb{\p{2\delta\p{X}-1}^2}.
\]
Rearranging yields
\[
\bbE\sqb{U\delta\p{X}}-\frac{\zeta}{4}\bbE\sqb{\p{2\delta\p{X}-1}^2}
=
C
-
\frac{1}{2}R_{\zeta}\p{f},
\]
where $C$ depends on $\bbE\sqb{U}$ and $\bbE\sqb{U^2}$ but not on $\delta$.
Adding $\bbE\sqb{Y\p{0}}$ on both sides gives $W_{\zeta/4}\p{\delta}=C'-R_{\zeta}\p{f}/2$ for a constant $C'$ that does not depend on $\delta$.
Taking differences between $\delta_1$ and $\delta_2$ cancels the constant.
\end{proof}

\subsection{Proof of Corollary \ref{cor:welfare-unregularized}}
\label{app:proof-welfare-cor-unreg}
\begin{proof}
Since $0\le \p{2\delta\p{X}-1}^2\le 1$, we have $0\le \bbE\sqb{\p{2\delta\p{X}-1}^2}\le 1$.
Therefore,
\[
W_{\lambda}\p{\delta}=V\p{\delta}-\lambda\bbE\sqb{\p{2\delta\p{X}-1}^2}
\le
V\p{\delta}
\le
W_{\lambda}\p{\delta}+\lambda.
\]
Since $\delta^\star_{\lambda}$ maximizes $W_{\lambda}$, we have $W_{\lambda}\p{\delta^\star}\le W_{\lambda}\p{\delta^\star_{\lambda}}$. Hence $V\p{\delta^\star}\le W_{\lambda}\p{\delta^\star_{\lambda}}+\lambda$, while $W_{\lambda}\p{\delta}\le V\p{\delta}$. Combining these inequalities yields the result.
\end{proof}

\subsection{Proof of Corollary \ref{cor:welfare-risk-K}}
\label{app:proof-welfare-cor-K}
\begin{proof}
Expanding the full-vector population risk yields
\[
R^{\mathrm{Full}}_{\zeta}\p{\delta}
=
\frac{1}{2\zeta}\sum_{a=1}^K \bbE\sqb{Y\p{a}^2}
-
\sum_{a=1}^K \bbE\sqb{Y\p{a}\delta_a\p{X}}
+
\frac{\zeta}{2}\bbE\sqb{\sum_{a=1}^K \delta_a\p{X}^2}.
\]
The welfare is $V\p{\delta}=\sum_{a=1}^K \bbE\sqb{Y\p{a}\delta_a\p{X}}$.
Thus,
\[
W^{\mathrm{Full}}_{\zeta/2}\p{\delta}
=
V\p{\delta}-\frac{\zeta}{2}\bbE\sqb{\sum_{a=1}^K \delta_a\p{X}^2}
=
C
-
R^{\mathrm{Full}}_{\zeta}\p{\delta},
\]
where $C=\frac{1}{2\zeta}\sum_{a=1}^K \bbE\sqb{Y\p{a}^2}$ is a constant that does not depend on $\delta$.
Taking differences yields the result.
\end{proof}

\subsection{Proof of Corollary \ref{cor:welfare-missing}}
\label{app:proof-welfare-missing}
\begin{proof}
Let $\widetilde{R}^{\mathrm{Full}}_{\zeta}\p{\delta}$ denote the population risk associated with $\widehat{L}^{\mathrm{IPW}}_n$ or $\widehat{L}^{\mathrm{DR}}_n$ after substituting a policy $\delta$ for $\delta_{\theta}$ and taking expectation.
Expanding the square yields
\[
\widetilde{R}^{\mathrm{Full}}_{\zeta}\p{\delta}
=
C
-
\sum_{a=1}^K \bbE\sqb{\widetilde{Y}\p{a}\delta_a\p{X}}
+
\frac{\zeta}{2}\bbE\sqb{\sum_{a=1}^K \delta_a\p{X}^2},
\]
where $C=\frac{1}{2\zeta}\sum_{a=1}^K \bbE\sqb{\widetilde{Y}\p{a}^2}$ does not depend on $\delta$.
Under the conditional mean correctness property in Proposition \ref{prop:dr-target},
\[
\bbE\sqb{\widetilde{Y}\p{a}\delta_a\p{X}}
=
\bbE\sqb{\bbE\sqb{\widetilde{Y}\p{a}\mid X}\delta_a\p{X}}
=
\bbE\sqb{\gamma_a\p{X}\delta_a\p{X}}
=
\bbE\sqb{Y\p{a}\delta_a\p{X}}.
\]
Thus, up to a constant, $\widetilde{R}^{\mathrm{Full}}_{\zeta}\p{\delta}$ differs from $W^{\mathrm{Full}}_{\zeta/2}\p{\delta}$ only through a sign, and taking differences between $\delta_1$ and $\delta_2$ yields the stated identity.
\end{proof}

\section{Details of GBPLNet}
\label{appdx:detail_gbpl}
This appendix summarizes computational details for GBPLNet, the neural network implementation example introduced in Section~\ref{sec:computation}.
We focus on the binary action setting with full feedback and comment on extensions.

\subsection{Binary Action Setting and GBPLNet Score Model}
We work with i.i.d. data $\calD=\cb{Z_i}_{i=1}^n$ where $Z_i=\p{X_i,Y_i\p{1},Y_i\p{0}}$.
Let $U_i=Y_i\p{1}-Y_i\p{0}$.
GBPLNet models a bounded score function by
$f_{w}\p{x}=\tanh\p{g_{w}\p{x}}$,
where $g_{w}\p{x}$ is an unconstrained neural network mapping $\calX$ to $\bbR$ and $w\in\calW$ denotes its weights.
By construction, $f_{w}\p{x}\in\p{-1,1}$ for all $x$.

The corresponding policy is $\delta_{w}\p{x}=\p{f_{w}\p{x}+1}/2$.
When a deterministic policy is required, we use the thresholding rule $a_{w}\p{x}=\mathbbm{1}\sqb{f_{w}\p{x}\ge 0}$.

\subsection{General Bayes Posterior and MAP Objective}
GBPLNet substitutes $f_{w}$ into the binary surrogate loss in \eqref{eq:loss-binary},
\[
\ell\p{w;\p{x,y\p{1},y\p{0}}}
=
\frac{1}{2}\p{\frac{1}{\sqrt{\zeta}}\p{y\p{1}-y\p{0}}-\sqrt{\zeta}f_{w}\p{x}}^2.
\]
The generalized posterior is
\[
\rmd \Pi_{\eta}\p{w\mid \calD}
\propto
\rmd \Pi\p{w}\exp\p{-\eta\sum^n_{i=1}\ell\p{w;Z_i}}.
\]

In many implementations, $\Pi$ is chosen as an isotropic Gaussian prior over weights, $\Pi=\mathcal{N}\p{0,\sigma_w^2 I}$.
Then $-\log \Pi\p{w}=\frac{1}{2\sigma_w^2}\sum_{j=1}^p w_j^2+\mathrm{const}$, where $p$ is the number of parameters.
The MAP estimator satisfies
\[
\widehat{w}_{\mathrm{MAP}}
\in
\arg\min_{w\in\calW}\cb{
\eta\sum^n_{i=1}\ell\p{w;Z_i}
+
\frac{1}{2\sigma_w^2}\sum_{j=1}^p w_j^2
}.
\]
This objective is standard neural network training under the squared surrogate loss with weight decay, up to the scaling by $\eta$.

\subsection{Objective Decomposition}
The binary surrogate loss admits a simple decomposition, which is useful for implementation and for interpreting the effect of $\zeta$.

\begin{proposition}
\label{prop:gbplnet-decomposition}
Let $U_i=Y_i\p{1}-Y_i\p{0}$.
For any $\zeta>0$ and any score function $f$, we have
\[
\frac{1}{2}\p{\frac{1}{\sqrt{\zeta}}U_i-\sqrt{\zeta}f\p{X_i}}^2
=
\frac{1}{2\zeta}U_i^2
-
U_i f\p{X_i}
+
\frac{\zeta}{2}f\p{X_i}^2.
\]
Consequently,
\[
\sum^n_{i=1}\frac{1}{2}\p{\frac{1}{\sqrt{\zeta}}U_i-\sqrt{\zeta}f\p{X_i}}^2
=
\frac{1}{2\zeta}\sum^n_{i=1}U_i^2
-
\sum^n_{i=1}U_i f\p{X_i}
+
\frac{\zeta}{2}\sum^n_{i=1}f\p{X_i}^2.
\]
\end{proposition}

Proposition \ref{prop:gbplnet-decomposition} implies that, for a fixed $\zeta$, minimizing the surrogate empirical loss is equivalent to minimizing
$-\sum^n_{i=1}U_i f\p{X_i}+\frac{\zeta}{2}\sum^n_{i=1}f\p{X_i}^2$, up to an additive constant that does not depend on $f$.
This is consistent with Theorem \ref{thm:binary-equivalence} after the reparameterization $f=2\delta-1$.

When tuning $\zeta$ by a validation criterion, Proposition \ref{prop:gbplnet-decomposition} clarifies that the full squared loss contains the term $\frac{1}{2\zeta}\sum^n_{i=1}U_i^2$ that does not depend on the fitted score.
If the validation criterion is taken as the surrogate loss itself, this term can bias selection toward larger $\zeta$ values.
A practical alternative is to tune $\zeta$ by validation welfare, or by the reduced objective that drops the constant term.

\subsection{Gaussian Approximation}
A Gaussian approximation to the generalized posterior is obtained by a second order expansion around a point estimator, such as the MAP.
Let $\widehat{w}_{\mathrm{MAP}}$ be a MAP estimate and define
\[
\Phi\p{w}
=
\eta\sum^n_{i=1}\ell\p{w;Z_i}
-
\log \Pi\p{w}.
\]
Let $H$ be the Hessian of $\Phi$ at $\widehat{w}_{\mathrm{MAP}}$.
A Gaussian approximation takes the form $w\mid\calD\approx \mathcal{N}\p{\widehat{w}_{\mathrm{MAP}},H^{-1}}$.
In large networks, $H$ is high dimensional, so implementations typically use diagonal, block-diagonal, or low-rank approximations.

\subsection{Stochastic Gradient Langevin Dynamics}
SGLD provides an approximate sampling method for generalized posteriors.
Let $\widehat{\Phi}_t\p{w}$ be a minibatch approximation of $\Phi\p{w}$ based on a batch $\calB_t\subset\cb{1,\dots,n}$, scaled to match the full sample objective.
Given a stepsize sequence $\cb{\epsilon_t}$, SGLD iterates
\[
w_{t+1}
=
w_t
-
\frac{\epsilon_t}{2}\nabla \widehat{\Phi}_t\p{w_t}
+
\sqrt{\epsilon_t}\xi_t,
\quad
\xi_t\sim \mathcal{N}\p{0,I}.
\]
After a burn-in period, draws are collected by thinning.
In practice, gradient clipping is often helpful when the outcomes have heavy tails, because the squared loss can yield large gradients.

\subsection{Posterior Predictive Welfare and Credible Intervals}
Let $\cb{w^{\p{s}}}_{s=1}^S$ be draws from the generalized posterior, obtained for example by SGLD.
For each draw, we can evaluate a policy on a held-out sample $\cb{Z_i}_{i=1}^n$.

If we use the deterministic decision rule $a_{w}\p{x}=\mathbbm{1}\sqb{f_{w}\p{x}\ge 0}$, the welfare draw is
\[
\widehat{V}^{\p{s}}
=
\frac{1}{n}\sum^n_{i=1}\Bigp{
a_{w^{\p{s}}}\p{X_i}Y_i\p{1}
+
\p{1-a_{w^{\p{s}}}\p{X_i}}Y_i\p{0}
}.
\]
If we instead report welfare for the randomized policy $\delta_{w}\p{x}=\p{f_{w}\p{x}+1}/2$, then
\[
\widehat{V}^{\p{s}}
=
\frac{1}{n}\sum^n_{i=1}\Bigp{
\delta_{w^{\p{s}}}\p{X_i}Y_i\p{1}
+
\p{1-\delta_{w^{\p{s}}}\p{X_i}}Y_i\p{0}
}.
\]
A $95\%$ credible interval for welfare is given by the empirical quantiles of $\cb{\widehat{V}^{\p{s}}}_{s=1}^S$, the $0.025$ and $0.975$ quantiles.
These intervals summarize posterior uncertainty under the generalized posterior, and their frequentist coverage depends on calibration of the loss scale.

\subsection{Multi-Action Extension}
For $K$ actions, the full-vector surrogate in \eqref{eq:loss-K-full} requires a policy model $\delta_{\theta}\p{x}\in\Delta_K$.
A neural implementation replaces $\delta_{\theta}$ by a network that outputs logits $h_{w}\p{x}\in\bbR^K$ and maps them to $\Delta_K$ by the softmax function.
This multi-action implementation is symmetric across actions, and it avoids baseline dependence in the surrogate construction.

\subsection{Minimal PyTorch Implementation for GBPLNet}
\label{app:code}
The following code illustrates MAP training for the binary surrogate with a tanh-squashed neural network score.
\begin{lstlisting}[language=Python]
import torch
import torch.nn as nn

class TanhScoreNet(nn.Module):
    def __init__(self, in_dim, hidden=(128, 128)):
        super().__init__()
        layers = []
        d = in_dim
        for h in hidden:
            layers += [nn.Linear(d, h), nn.ReLU()]
            d = h
        layers += [nn.Linear(d, 1)]
        self.net = nn.Sequential(*layers)

    def forward(self, x):
        return torch.tanh(self.net(x)).view(-1)

def surrogate_loss(y1, y0, f, zeta):
    d = y1 - y0
    return 0.5 * (d / (zeta ** 0.5) - (zeta ** 0.5) * f) ** 2

def map_objective(model, xb, y1b, y0b, n_total, eta, prior_sd, zeta):
    f = model(xb)
    loss_batch = surrogate_loss(y1b.view(-1), y0b.view(-1), f, zeta).sum()
    B = xb.size(0)
    data_term = eta * (n_total / B) * loss_batch
    prior_term = 0.0
    for p in model.parameters():
        prior_term = prior_term + (p ** 2).sum()
    prior_term = 0.5 * prior_term / (prior_sd ** 2)
    return data_term + prior_term
\end{lstlisting}

\section{Posterior Visualization Example}
\label{app:posterior-visualization}
This section illustrates how the generalized posterior in GBPLNet (Section~\ref{sec:computation}) induces uncertainty bands for the score function and credible intervals for welfare, under a controlled full feedback binary setting.

\subsection{Data Generating Process and Training Setup}
We generate one-dimensional covariates $X\in\bbR$ with $X\sim \mathrm{Unif}\p{-2.5,2.5}$.
The baseline mean function is $\gamma_0\p{x}=0.2x+0.2\sin\p{1.5x}$ and the conditional outcome gap is $\tau\p{x}=1.2\sin\p{x}$.
We generate potential outcomes as $Y\p{0}=\gamma_0\p{X}+\epsilon_0$ and $Y\p{1}=\gamma_0\p{X}+\tau\p{X}+\epsilon_1$, with independent $\epsilon_0,\epsilon_1\sim\mathcal{N}\p{0,0.6^2}$.
We set $\zeta=1.0$ to avoid trivial saturation of the bounded score in most of the covariate range.
The sample size is $n=1500$, and we split the data into training, validation, and test sets with proportions $0.6$, $0.2$, and $0.2$.

For the score model, we use the GBPLNet parameterization $f_w\p{x}=\tanh\p{g_w\p{x}}$ with a two-hidden-layer MLP $g_w$ of widths $\p{64,64}$ and ReLU activation.
We compute a MAP estimate under the squared-loss surrogate with a Gaussian prior $\Pi=\mathcal{N}\p{0,\sigma_w^2 I}$, using $\sigma_w=1.0$ and $\eta=1.0$.
We train with minibatches of size $128$, learning rate $10^{-3}$, and weight decay $10^{-4}$, and we apply early stopping based on the validation surrogate loss.
In this run, the best validation surrogate loss is $0.3482$.

\subsection{Generalized Posterior Sampling by SGLD}
To visualize the generalized posterior, we approximate draws from $\Pi_\eta\p{\cdot\mid\calD}$ by SGLD starting from the MAP solution.
We use a burn-in period of $1200$ iterations, collect $300$ draws, thin by a factor $8$, and use step size $2\times 10^{-5}$ with batch size $128$.
These settings define the finite-iteration SGLD approximation used in this illustration.

\subsection{Posterior Draws for the Score Function}
Figure~\ref{fig:gbplnet-posterior-fx} plots posterior draws of $f_w\p{x}$ on a grid, together with the posterior mean and a pointwise $95\%$ credible band.
The figure also overlays the population target $\max\p{-1,\min\p{1,\tau\p{x}/\zeta}}$ from Section~\ref{sec:theory}, which equals $\max\p{-1,\min\p{1,1.2\sin\p{x}}}$ for $\zeta=1.0$.
The posterior mean tracks the target in regions where the target lies in $\p{-1,1}$ and saturates near the boundaries when the target exceeds the action range.
The credible band widens in transition regions where the decision boundary $f_w\p{x}=0$ is sensitive to the fit, which is reflected in posterior variability in downstream welfare.

\begin{figure*}[t]
  \centering
  \includegraphics[width=0.90\textwidth]{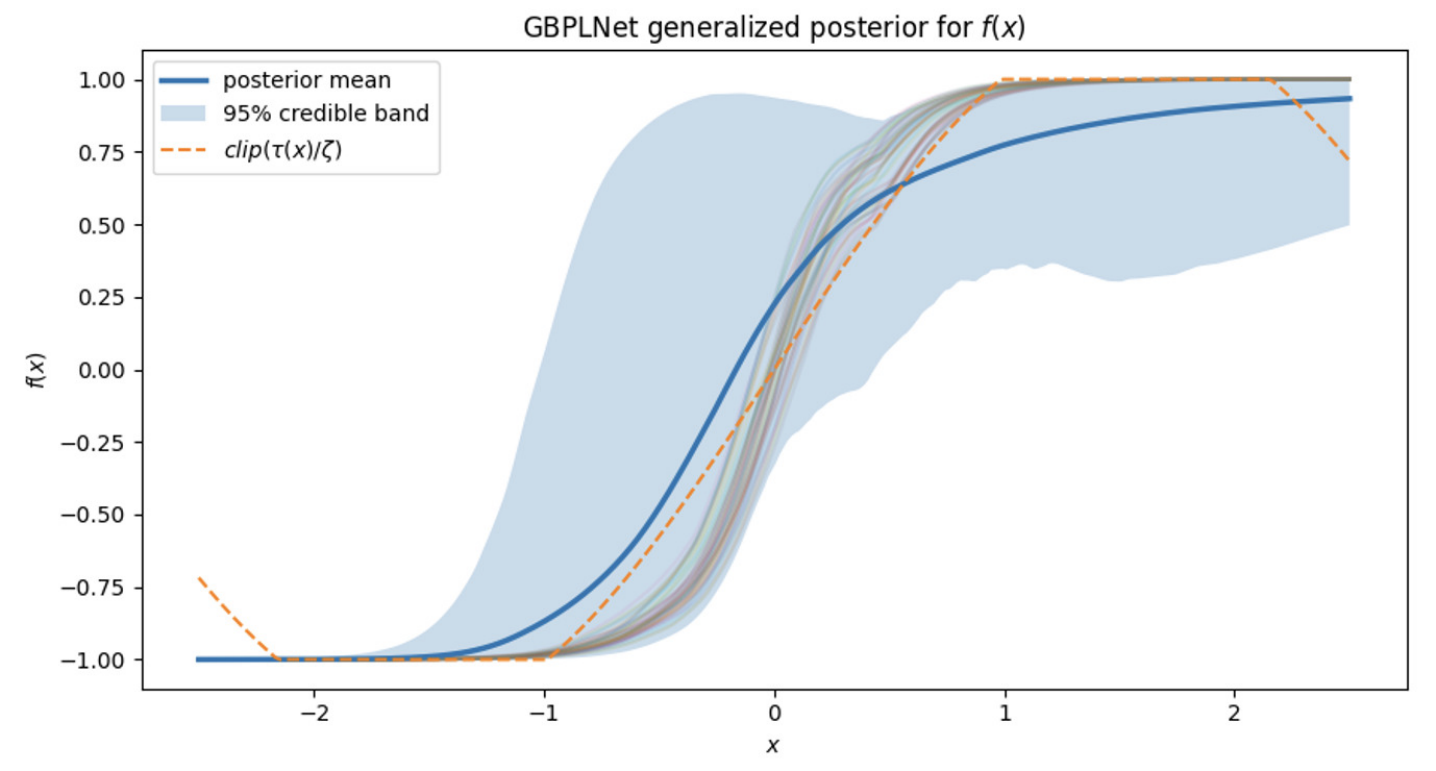}
  \caption{GBPLNet posterior visualization in the one-dimensional binary example.
  The plot shows posterior draws of the score function $f_w\p{x}$, the posterior mean, and a pointwise $95\%$ credible band, along with the population target $\max\p{-1,\min\p{1,\tau\p{x}/\zeta}}$.}
  \label{fig:gbplnet-posterior-fx}
\end{figure*}

\subsection{Posterior Distribution of Welfare and Credible Intervals}
For each posterior draw $w^{\p{s}}$, we compute welfare on the test set using the deterministic policy $a_{w^{\p{s}}}\p{x}=\mathbbm{1}\sqb{f_{w^{\p{s}}}\p{x}\ge 0}$.
The welfare draw is
\[
\widehat{V}^{\p{s}}
=
\frac{1}{n}\sum^n_{i=1}\Bigp{
a_{w^{\p{s}}}\p{X_i}Y_i\p{1}
+
\p{1-a_{w^{\p{s}}}\p{X_i}}Y_i\p{0}
}.
\]
Figure~\ref{fig:gbplnet-posterior-welfare} shows the empirical distribution of $\cb{\widehat{V}^{\p{s}}}_{s=1}^S$.
In this run, the posterior welfare mean is $0.4699$, and the $95\%$ credible interval given by the $0.025$ and $0.975$ quantiles is $\sqb{0.3327,0.5021}$.
These intervals summarize uncertainty under the generalized posterior.
Their frequentist coverage depends on calibration choices for $\eta$ and the prior, and they should be interpreted as loss-based posterior uncertainty summaries rather than as automatically calibrated confidence intervals.

\begin{figure*}[t]
  \centering
  \includegraphics[width=0.9\textwidth]{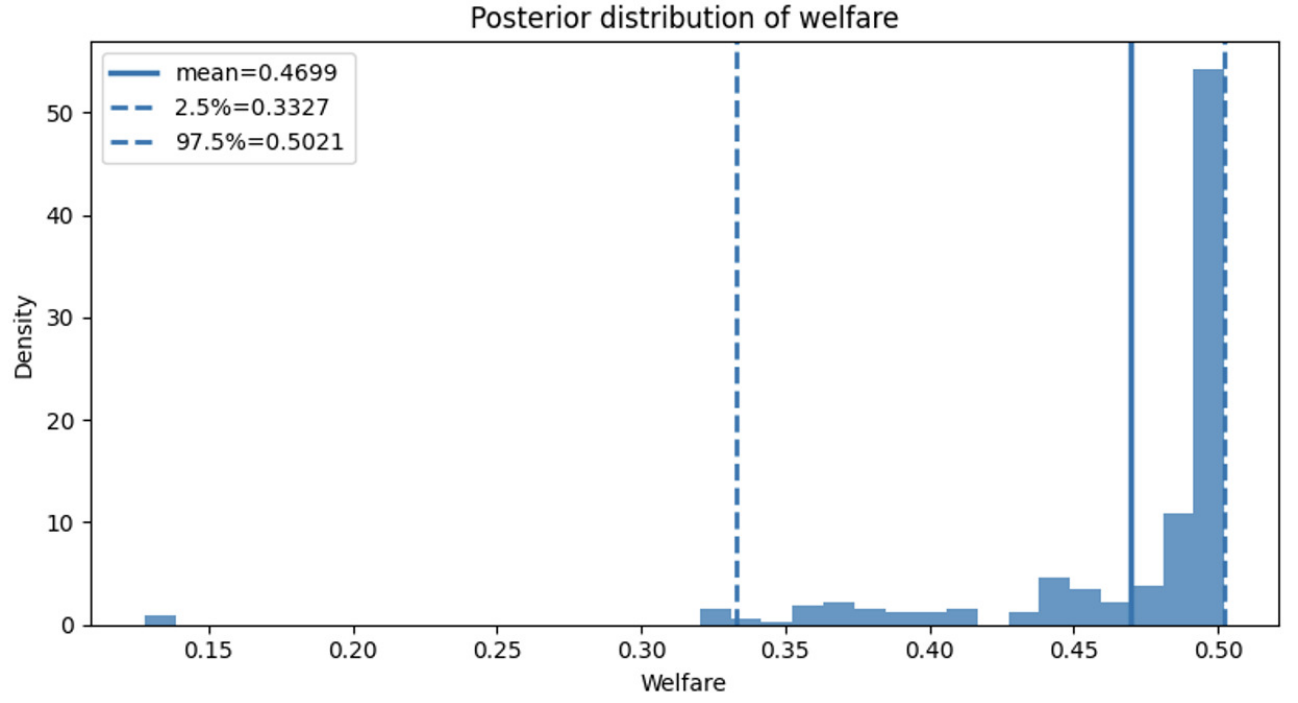}
  \caption{Posterior distribution of test welfare in the one-dimensional binary example.
  The histogram is computed from SGLD draws of $w$ and welfare evaluation under the induced deterministic policy.}
  \label{fig:gbplnet-posterior-welfare}
\end{figure*}

\subsubsection{How to use posterior uncertainty in practice}
Posterior summaries can be used for several purposes.
First, credible intervals for welfare provide an uncertainty report for the induced decision rule under the chosen loss scale.
Moreover, posterior draws can be used to assess decision stability, for example, by inspecting posterior variability of the decision boundary through the distribution of $f_w\p{x_0}$ at selected covariate values.
Figure~\ref{fig:gbplnet-posterior-fx0} illustrates these distributions at two selected points.
Finally, posterior lower quantiles of welfare computed on a validation set can be used to compare policies under the chosen loss scale.

\begin{figure*}[t]
  \centering
  \includegraphics[width=0.90\textwidth]{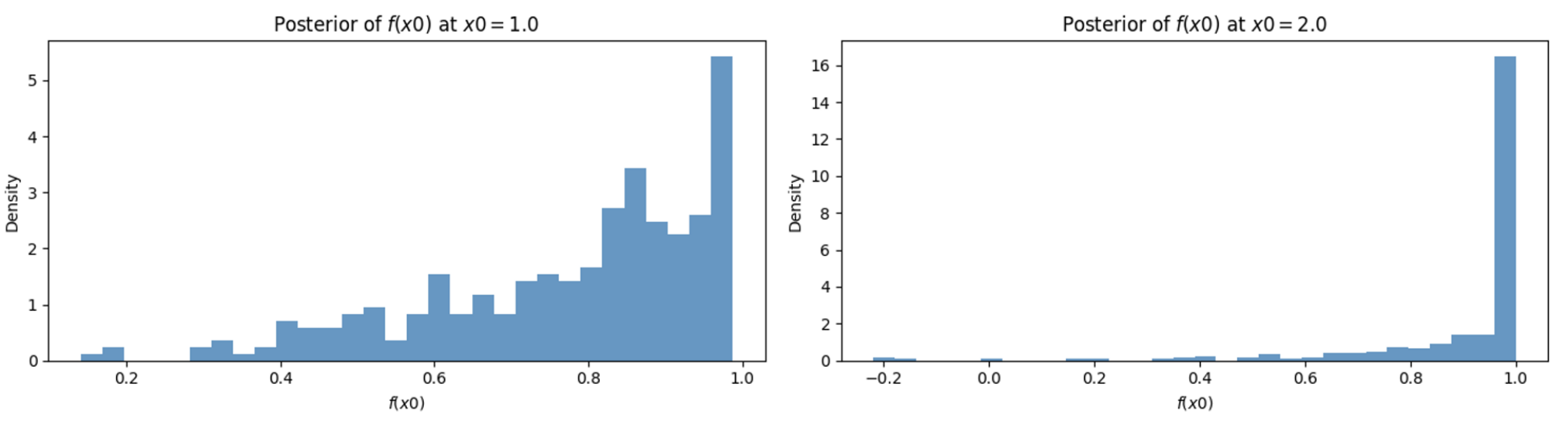}
  \caption{Posterior distributions of $f_w\p{x_0}$ at two fixed covariate values $x_0\in\cb{1,2}$ in the one-dimensional binary example.
  These histograms show posterior variation in the score at the selected covariate values.}
  \label{fig:gbplnet-posterior-fx0}
\end{figure*}

\section{Additional Experiments for Learning Curves, Weak Overlap, and Posterior Uncertainty}
\label{app:additional-posterior-experiments}

This appendix reports additional experiments that complement the main simulation tables. The experiments examine learning curves under full feedback and logged feedback, behavior under weak overlap, and posterior uncertainty summaries. In the two learning-curve subsections below, regret is measured relative to the Bayes-optimal action under the known conditional means and is not directly comparable with realized-outcome regret in Section~\ref{sec:experiments}. All reported values are computed over independent replications. Values in parentheses are Monte Carlo standard errors.

\subsection{Learning Curves under Full Feedback}
\label{app:full-feedback-learning-curves}

We first report full-feedback learning curves for binary DGP2. For each observed sample size, we train the methods on the observed sample and evaluate the learned deterministic rule on an independent evaluation sample using the known conditional means. Figure~\ref{fig:appendix-full-feedback-learning} shows welfare and regret as functions of the observed sample size. Table~\ref{tab:appendix-full-feedback-learning} gives the corresponding numerical summaries.

\begin{figure*}[t]
\centering
\includegraphics[width=0.49\textwidth]{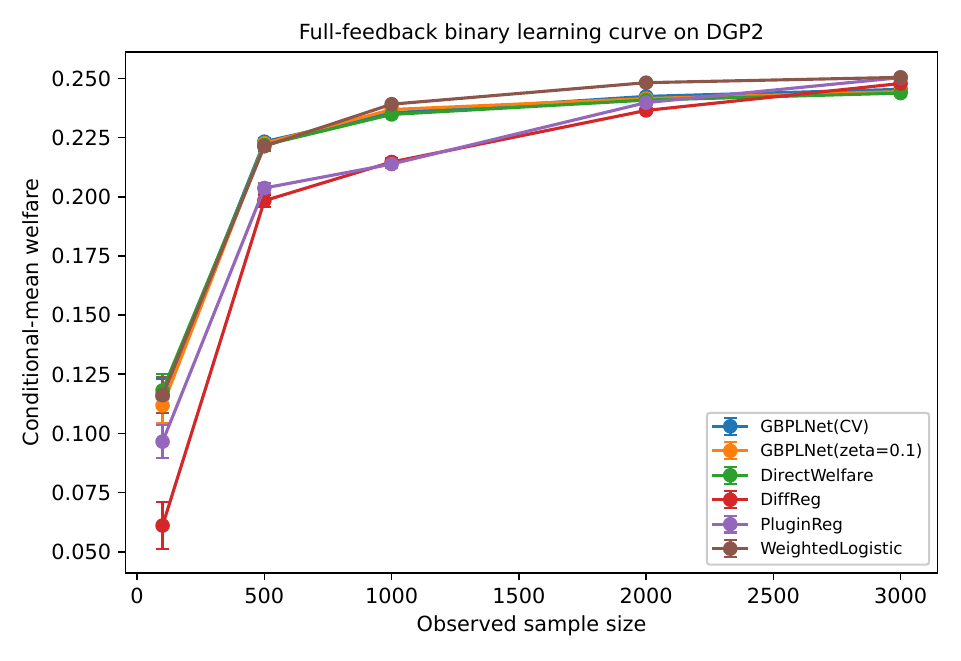}
\includegraphics[width=0.49\textwidth]{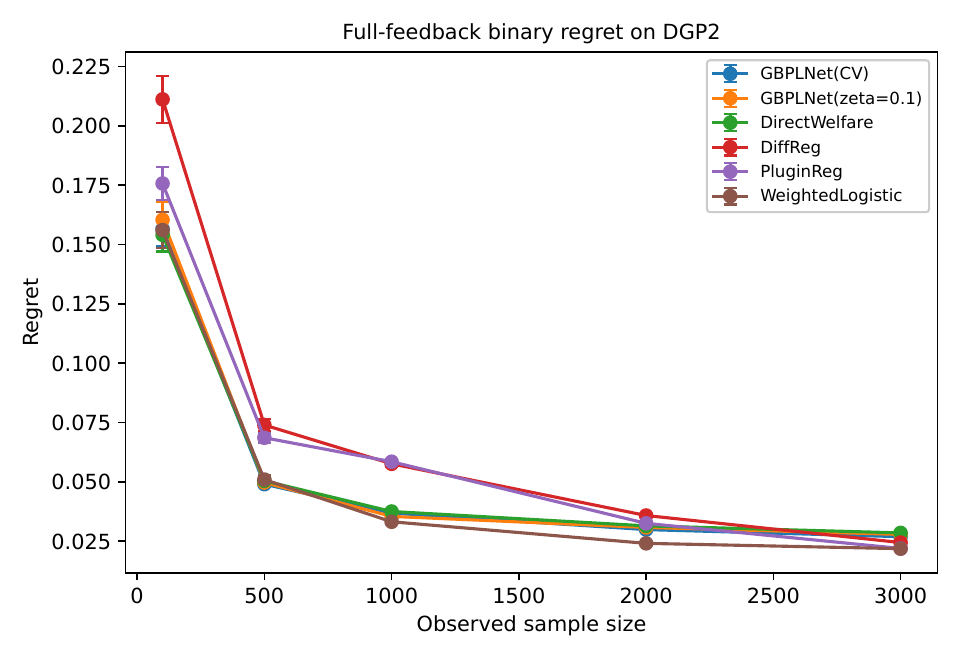}
\caption{Full-feedback learning curves for binary DGP2 over $100$ replications. The left panel reports welfare, and the right panel reports regret. Error bars show Monte Carlo standard errors.}
\label{fig:appendix-full-feedback-learning}
\end{figure*}

\begin{table*}[t]
\centering
\caption{Full-feedback DGP2 learning-curve summary over $100$ replications. Entries are means followed by Monte Carlo standard errors in parentheses.}
\label{tab:appendix-full-feedback-learning}
\begin{tabular}{rlcccc}
\toprule
$n$ & Method & $\zeta$ & Welfare & Regret & Agreement \\
\midrule
\multirow{5}{*}{100} & GBPLNet & CV & 0.1157 (0.0072) & 0.1566 (0.0072) & 0.7730 (0.0065) \\
 & DirectWelfare & -- & 0.1182 (0.0070) & 0.1541 (0.0070) & 0.7752 (0.0063) \\
 & DiffReg & -- & 0.0611 (0.0100) & 0.2112 (0.0100) & 0.7238 (0.0088) \\
 & PluginReg & -- & 0.0965 (0.0070) & 0.1757 (0.0070) & 0.7551 (0.0063) \\
 & WeightedLogistic & -- & 0.1162 (0.0075) & 0.1561 (0.0075) & 0.7729 (0.0069) \\
\addlinespace
\multirow{5}{*}{500} & GBPLNet & CV & 0.2233 (0.0015) & 0.0490 (0.0013) & 0.8860 (0.0018) \\
 & DirectWelfare & -- & 0.2218 (0.0013) & 0.0504 (0.0011) & 0.8834 (0.0016) \\
 & DiffReg & -- & 0.1983 (0.0025) & 0.0739 (0.0025) & 0.8545 (0.0027) \\
 & PluginReg & -- & 0.2037 (0.0021) & 0.0686 (0.0021) & 0.8615 (0.0023) \\
 & WeightedLogistic & -- & 0.2213 (0.0018) & 0.0509 (0.0017) & 0.8826 (0.0024) \\
\addlinespace
\multirow{5}{*}{1000} & GBPLNet & CV & 0.2354 (0.0009) & 0.0368 (0.0007) & 0.9047 (0.0012) \\
 & DirectWelfare & -- & 0.2348 (0.0009) & 0.0375 (0.0007) & 0.9036 (0.0012) \\
 & DiffReg & -- & 0.2147 (0.0017) & 0.0576 (0.0015) & 0.8762 (0.0019) \\
 & PluginReg & -- & 0.2138 (0.0015) & 0.0585 (0.0013) & 0.8756 (0.0015) \\
 & WeightedLogistic & -- & 0.2391 (0.0011) & 0.0332 (0.0009) & 0.9109 (0.0016) \\
\addlinespace
\multirow{5}{*}{2000} & GBPLNet & CV & 0.2424 (0.0008) & 0.0299 (0.0005) & 0.9182 (0.0010) \\
 & DirectWelfare & -- & 0.2408 (0.0008) & 0.0314 (0.0005) & 0.9148 (0.0010) \\
 & DiffReg & -- & 0.2365 (0.0010) & 0.0358 (0.0006) & 0.9053 (0.0009) \\
 & PluginReg & -- & 0.2398 (0.0010) & 0.0325 (0.0006) & 0.9108 (0.0009) \\
 & WeightedLogistic & -- & 0.2482 (0.0008) & 0.0241 (0.0004) & 0.9315 (0.0010) \\
\addlinespace
\multirow{5}{*}{3000} & GBPLNet & CV & 0.2454 (0.0008) & 0.0269 (0.0004) & 0.9248 (0.0009) \\
 & DirectWelfare & -- & 0.2438 (0.0008) & 0.0285 (0.0003) & 0.9207 (0.0007) \\
 & DiffReg & -- & 0.2479 (0.0008) & 0.0244 (0.0003) & 0.9231 (0.0006) \\
 & PluginReg & -- & 0.2503 (0.0007) & 0.0219 (0.0003) & 0.9272 (0.0005) \\
 & WeightedLogistic & -- & 0.2504 (0.0007) & 0.0218 (0.0003) & 0.9378 (0.0008) \\
\bottomrule
\end{tabular}
\par\smallskip
\begin{minipage}{0.96\textwidth}
\footnotesize Agreement is the fraction of evaluation points at which the learned deterministic action equals the Bayes-optimal action defined by the known conditional means. DirectWelfare directly maximizes empirical full-feedback welfare over the policy class used in this experiment.
\end{minipage}
\end{table*}

\subsection{Learning Curves under Logged Feedback}
\label{app:logged-feedback-learning-curves}

We next report logged-feedback learning curves for binary DGP2. Only the outcome under the logged action is observed. GBPLNet-DR uses cross-fitted outcome regressions to form DR pseudo-outcomes, while GBPLNet-IPW uses IPW pseudo-outcomes. Figure~\ref{fig:appendix-logged-feedback-learning} reports welfare and regret. Table~\ref{tab:appendix-logged-feedback-learning} reports the numerical summaries.

\begin{figure*}[t]
\centering
\includegraphics[width=0.49\textwidth]{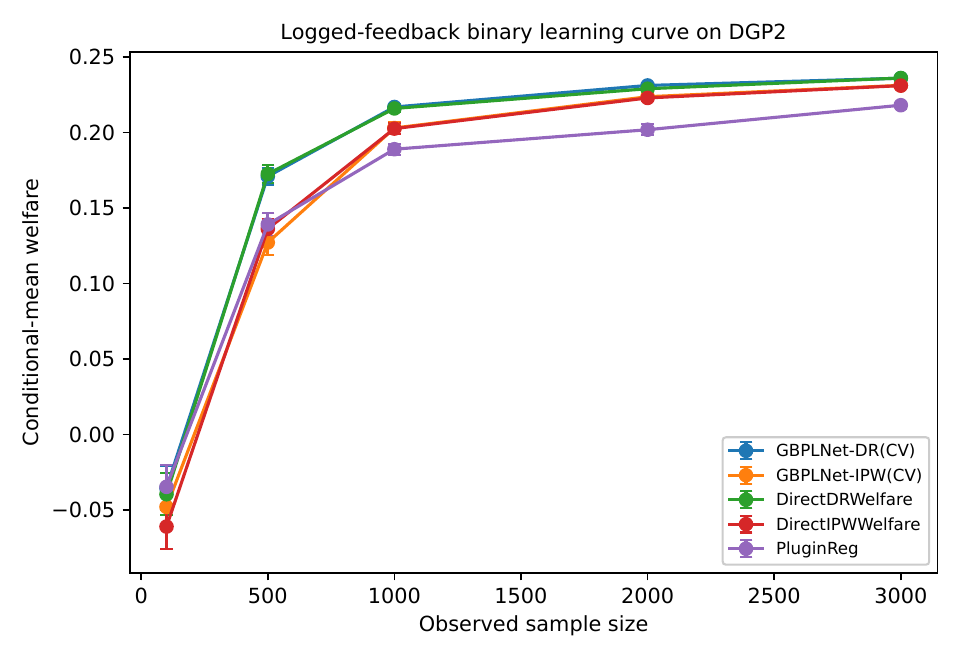}
\includegraphics[width=0.49\textwidth]{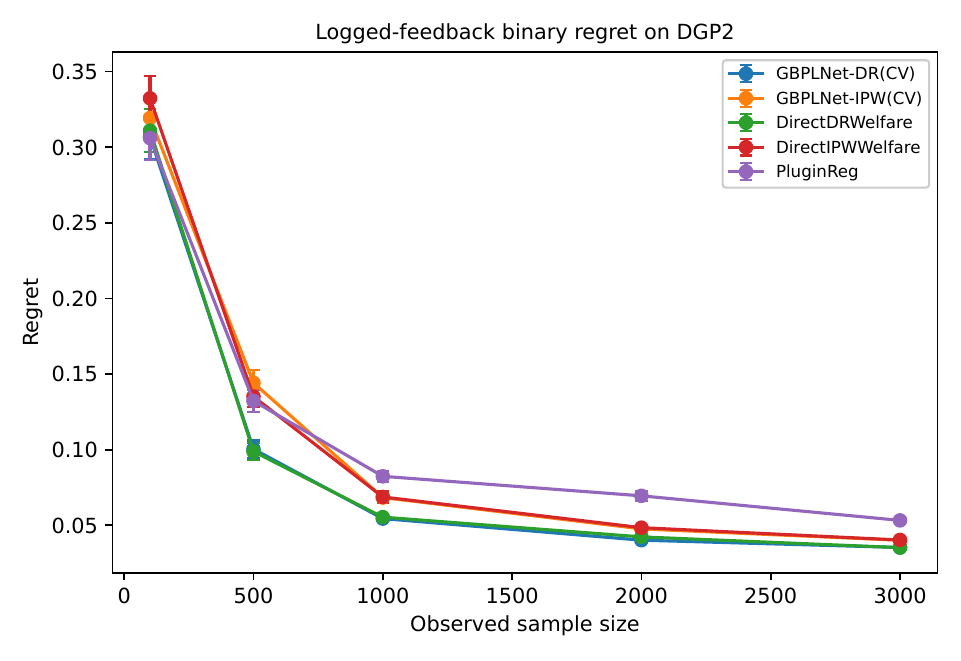}
\caption{Logged-feedback learning curves for binary DGP2 over $50$ replications. The left panel reports welfare, and the right panel reports regret. Error bars show Monte Carlo standard errors.}
\label{fig:appendix-logged-feedback-learning}
\end{figure*}

\begin{table*}[t]
\centering
\caption{Logged-feedback DGP2 learning-curve summary over $50$ replications. Entries are means followed by Monte Carlo standard errors in parentheses.}
\label{tab:appendix-logged-feedback-learning}
\begin{tabular}{rlcccc}
\toprule
$n$ & Method & $\zeta$ & Welfare & Regret & Agreement \\
\midrule
\multirow{5}{*}{100} & GBPLNet-DR & CV & -0.0356 (0.0146) & 0.3068 (0.0147) & 0.6405 (0.0126) \\
 & GBPLNet-IPW & CV & -0.0480 (0.0140) & 0.3192 (0.0140) & 0.6299 (0.0119) \\
 & DirectDRWelfare & -- & -0.0397 (0.0140) & 0.3109 (0.0141) & 0.6367 (0.0120) \\
 & DirectIPWWelfare & -- & -0.0611 (0.0151) & 0.3323 (0.0150) & 0.6185 (0.0129) \\
 & PluginReg & -- & -0.0347 (0.0147) & 0.3060 (0.0146) & 0.6411 (0.0124) \\
\addlinespace
\multirow{5}{*}{500} & GBPLNet-DR & CV & 0.1708 (0.0056) & 0.1004 (0.0057) & 0.8267 (0.0057) \\
 & GBPLNet-IPW & CV & 0.1271 (0.0081) & 0.1442 (0.0083) & 0.7851 (0.0077) \\
 & DirectDRWelfare & -- & 0.1724 (0.0059) & 0.0988 (0.0059) & 0.8284 (0.0059) \\
 & DirectIPWWelfare & -- & 0.1361 (0.0068) & 0.1351 (0.0069) & 0.7931 (0.0066) \\
 & PluginReg & -- & 0.1390 (0.0075) & 0.1322 (0.0075) & 0.7950 (0.0069) \\
\addlinespace
\multirow{5}{*}{1000} & GBPLNet-DR & CV & 0.2168 (0.0022) & 0.0544 (0.0017) & 0.8777 (0.0023) \\
 & GBPLNet-IPW & CV & 0.2029 (0.0039) & 0.0684 (0.0037) & 0.8624 (0.0042) \\
 & DirectDRWelfare & -- & 0.2158 (0.0024) & 0.0554 (0.0022) & 0.8773 (0.0029) \\
 & DirectIPWWelfare & -- & 0.2025 (0.0037) & 0.0687 (0.0038) & 0.8621 (0.0044) \\
 & PluginReg & -- & 0.1888 (0.0036) & 0.0824 (0.0035) & 0.8458 (0.0038) \\
\addlinespace
\multirow{5}{*}{2000} & GBPLNet-DR & CV & 0.2311 (0.0017) & 0.0402 (0.0013) & 0.8993 (0.0021) \\
 & GBPLNet-IPW & CV & 0.2236 (0.0027) & 0.0476 (0.0023) & 0.8890 (0.0032) \\
 & DirectDRWelfare & -- & 0.2289 (0.0015) & 0.0423 (0.0010) & 0.8950 (0.0017) \\
 & DirectIPWWelfare & -- & 0.2228 (0.0022) & 0.0484 (0.0017) & 0.8872 (0.0025) \\
 & PluginReg & -- & 0.2017 (0.0035) & 0.0695 (0.0033) & 0.8628 (0.0037) \\
\addlinespace
\multirow{5}{*}{3000} & GBPLNet-DR & CV & 0.2359 (0.0017) & 0.0354 (0.0012) & 0.9079 (0.0021) \\
 & GBPLNet-IPW & CV & 0.2311 (0.0016) & 0.0401 (0.0014) & 0.8999 (0.0021) \\
 & DirectDRWelfare & -- & 0.2360 (0.0015) & 0.0353 (0.0009) & 0.9073 (0.0016) \\
 & DirectIPWWelfare & -- & 0.2309 (0.0018) & 0.0403 (0.0015) & 0.8996 (0.0023) \\
 & PluginReg & -- & 0.2180 (0.0021) & 0.0533 (0.0017) & 0.8824 (0.0020) \\
\bottomrule
\end{tabular}
\par\smallskip
\begin{minipage}{0.96\textwidth}
\footnotesize DirectDRWelfare and DirectIPWWelfare directly maximize the DR and IPW pseudo-welfare criteria, respectively, over the policy class used in this experiment. Agreement is defined as in Appendix~\ref{app:full-feedback-learning-curves}.
\end{minipage}
\end{table*}

\subsection{Weak-Overlap Diagnostics}
\label{app:weak-overlap-diagnostics}

We examine the effect of propensity clipping under weak overlap. For each replication, we generate logged actions from a logistic policy with a stronger index and evaluate clipping levels $\epsilon_{\mathrm{clip}}\in\cb{0.01,0.05,0.10}$. Figure~\ref{fig:appendix-weak-overlap} shows welfare and the upper tail of the absolute pseudo-difference. Table~\ref{tab:appendix-weak-overlap} reports welfare, regret, the clipped fraction, the effective sample size (ESS), and pseudo-outcome tail summaries. Increasing the clipping level raises the effective sample size and reduces the tail of the pseudo-difference. Moreover, for GBPLNet-DR (CV), welfare increases slightly as clipping becomes more aggressive in this design.

\begin{figure*}[t]
\centering
\includegraphics[width=0.49\textwidth]{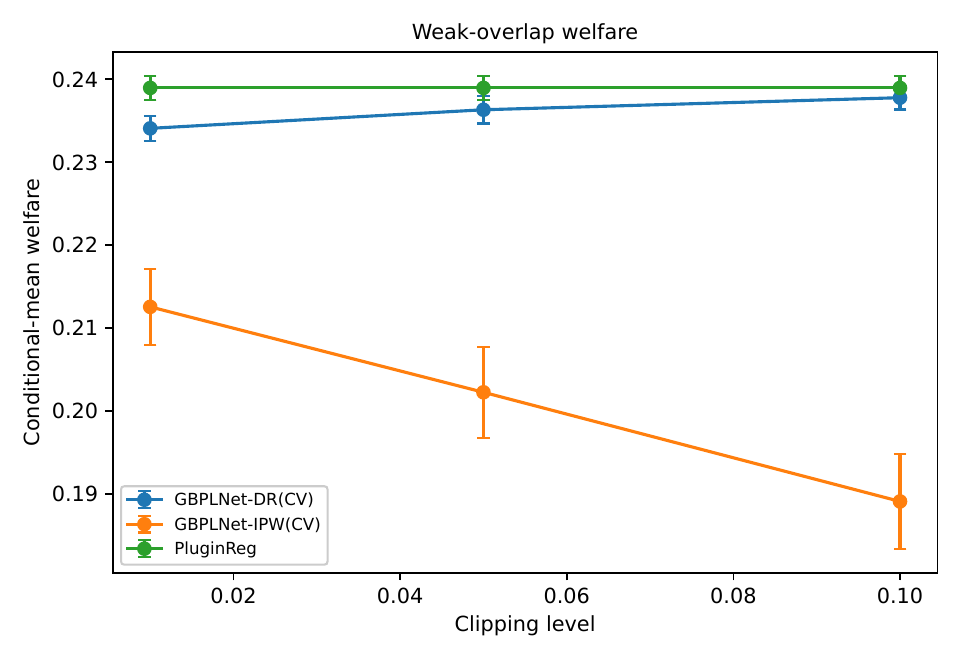}
\includegraphics[width=0.49\textwidth]{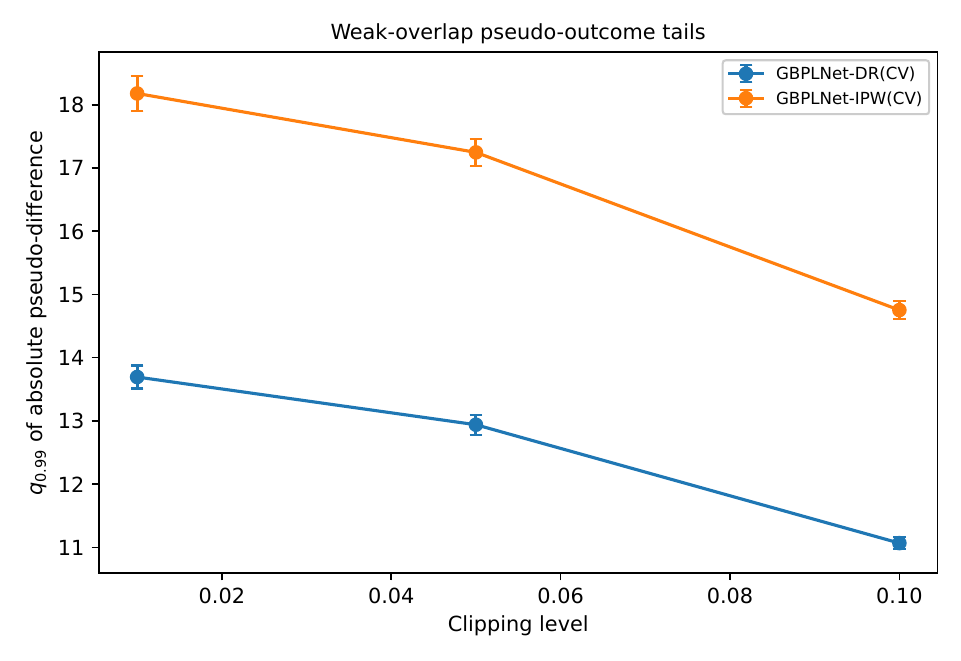}
\caption{Weak-overlap diagnostics over $50$ replications. The left panel reports welfare by clipping level. The right panel reports the $0.99$ quantile of the absolute pseudo-difference. Error bars show Monte Carlo standard errors.}
\label{fig:appendix-weak-overlap}
\end{figure*}

\begin{table*}[t]
\centering
\caption{Weak-overlap diagnostics over $50$ replications. Welfare and regret are means followed by Monte Carlo standard errors in parentheses. The clipped fraction, effective sample size (ESS), $q_{0.99}$, and maximum are computed from the training portion of the logged sample.}
\label{tab:appendix-weak-overlap}
\setlength{\tabcolsep}{2.5pt}
\begin{tabular}{rllcccccc}
\toprule
Clip & Method & $\zeta$ & Welfare & Regret & Clipped fraction & ESS & $q_{0.99}\p{\abs{\widetilde U}}$ & $\max_i\abs{\widetilde U_i}$ \\
\midrule
\multirow{3}{*}{0.01} & GBPLNet-DR & CV & 0.2341 (0.0015) & 0.0371 (0.0010) & \multirow{3}{*}{0.0684} & \multirow{3}{*}{1090.1} & 13.6938 & 179.0290 \\
 & GBPLNet-IPW & CV & 0.2125 (0.0046) & 0.0587 (0.0042) &  &  & 18.1773 & 196.5811 \\
 & PluginReg & -- & 0.2390 (0.0015) & 0.0323 (0.0008) &  &  & -- & -- \\
\addlinespace
\multirow{3}{*}{0.05} & GBPLNet-DR & CV & 0.2363 (0.0017) & 0.0349 (0.0011) & \multirow{3}{*}{0.2433} & \multirow{3}{*}{2071.9} & 12.9393 & 57.6668 \\
 & GBPLNet-IPW & CV & 0.2022 (0.0055) & 0.0690 (0.0052) &  &  & 17.2457 & 66.3725 \\
 & PluginReg & -- & 0.2390 (0.0015) & 0.0323 (0.0008) &  &  & -- & -- \\
\addlinespace
\multirow{3}{*}{0.10} & GBPLNet-DR & CV & 0.2378 (0.0014) & 0.0334 (0.0008) & \multirow{3}{*}{0.3845} & \multirow{3}{*}{2760.7} & 11.0675 & 32.1321 \\
 & GBPLNet-IPW & CV & 0.1891 (0.0058) & 0.0821 (0.0054) &  &  & 14.7522 & 39.0025 \\
 & PluginReg & -- & 0.2390 (0.0015) & 0.0323 (0.0008) &  &  & -- & -- \\
\bottomrule
\end{tabular}
\par\smallskip
\begin{minipage}{0.96\textwidth}
\footnotesize The ESS is $\p{\sum_i \omega_i}^2/\sum_i \omega_i^2$, where $\omega_i=1/\widehat e^{\mathrm{clip}}_{A_i}\p{X_i}$.
\end{minipage}
\end{table*}

\subsection{Repeated-Simulation Evaluation of Welfare Uncertainty}
\label{app:posterior-uncertainty-evaluation}

This subsection evaluates posterior welfare intervals in the one-dimensional binary design described in Appendix~\ref{app:posterior-visualization}. For each replication, we fit GBPLNet, fix the hidden-layer parameters at their MAP values, and construct a generalized Gauss--Newton Gaussian approximation for the final layer. We draw $400$ approximate posterior samples and compute the welfare of each sampled policy on an independent evaluation sample. The resulting intervals are summaries from the loss-based posterior. Their frequentist coverage depends on the loss scale $\eta$, so we evaluate coverage by repeated simulation.

Table~\ref{tab:appendix-uq-calibration} reports calibration results for $\eta\in\cb{0.25,0.5,1,2,4,8}$ over $100$ replications. The value $\eta=4$ gives empirical coverage closest to $0.90$ for the nominal $90\%$ interval. Figure~\ref{fig:appendix-uq-calibration} shows the calibration curve. Table~\ref{tab:appendix-uq-evaluation} then reports an independent evaluation over $200$ new replications for $\eta=1$ and $\eta=4$. Calibration reduces interval width and brings the $90\%$ and $95\%$ intervals close to their nominal coverage levels in this design. Figure~\ref{fig:appendix-uq-evaluation} displays the same independent-evaluation coverages.

\begin{table}[t]
\centering
\caption{Welfare-interval calibration over $100$ replications. Each entry is empirical coverage followed by mean interval width in parentheses.}
\label{tab:appendix-uq-calibration}
\begin{tabular}{rcccc}
\toprule
$\eta$ & $50\%$ & $80\%$ & $90\%$ & $95\%$ \\
\midrule
0.25 & 0.93 (0.0479) & 1.00 (0.1114) & 1.00 (0.1670) & 1.00 (0.2219) \\
0.5 & 0.78 (0.0376) & 0.99 (0.0822) & 0.99 (0.1195) & 1.00 (0.1566) \\
1 & 0.69 (0.0274) & 0.95 (0.0604) & 0.99 (0.0869) & 0.99 (0.1152) \\
2 & 0.63 (0.0195) & 0.88 (0.0426) & 0.97 (0.0615) & 0.99 (0.0830) \\
4 & 0.54 (0.0135) & 0.80 (0.0289) & 0.90 (0.0413) & 0.93 (0.0545) \\
8 & 0.47 (0.0093) & 0.72 (0.0192) & 0.80 (0.0265) & 0.84 (0.0342) \\
\bottomrule
\end{tabular}
\end{table}

\begin{figure}[t]
\centering
\includegraphics[width=0.70\columnwidth]{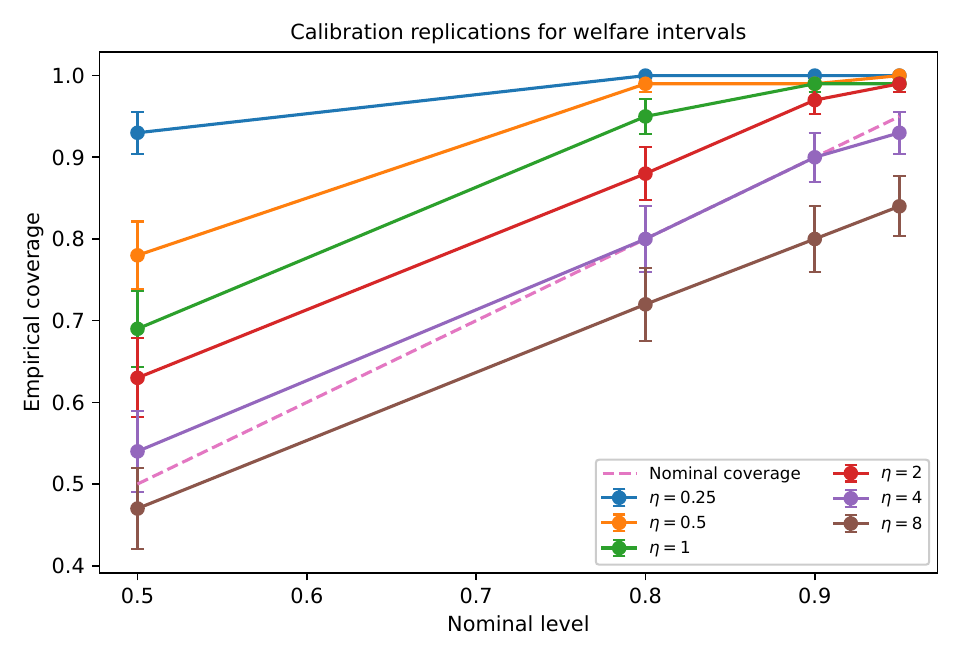}
\caption{Empirical welfare-interval coverage on the $100$ calibration replications.}
\label{fig:appendix-uq-calibration}
\end{figure}

\begin{table}[t]
\centering
\caption{Independent welfare-interval evaluation over $200$ replications. Each entry is empirical coverage followed by mean interval width in parentheses.}
\label{tab:appendix-uq-evaluation}
\begin{tabular}{rcccc}
\toprule
$\eta$ & $50\%$ & $80\%$ & $90\%$ & $95\%$ \\
\midrule
1 & 0.750 (0.0283) & 0.945 (0.0624) & 0.980 (0.0913) & 0.990 (0.1231) \\
4 & 0.635 (0.0139) & 0.855 (0.0300) & 0.910 (0.0433) & 0.950 (0.0575) \\
\bottomrule
\end{tabular}
\end{table}

\begin{figure}[t]
\centering
\includegraphics[width=0.70\columnwidth]{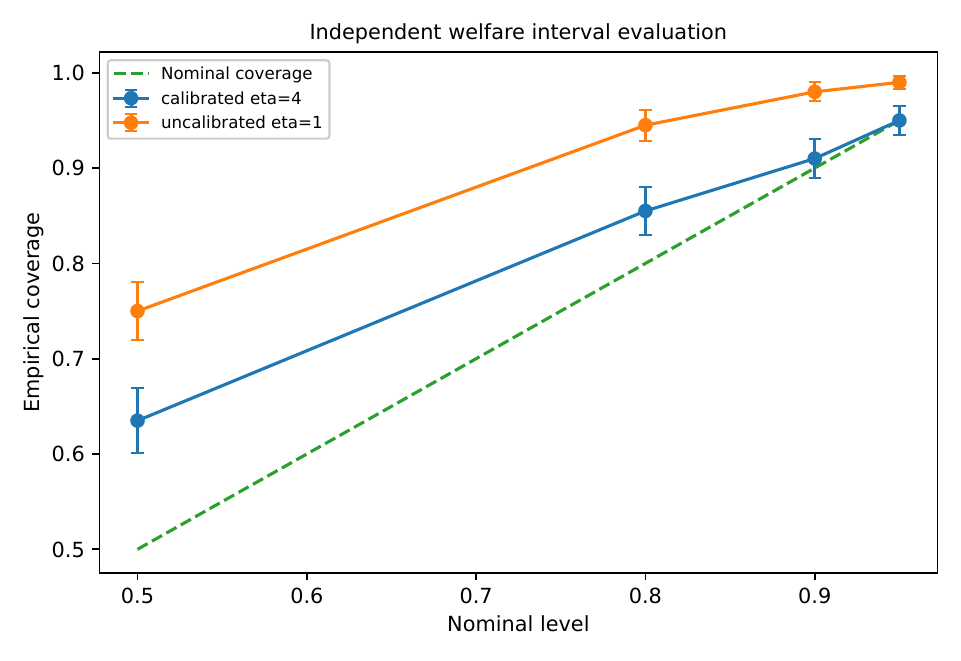}
\caption{Independent welfare-interval evaluation over $200$ replications. The dashed line is nominal coverage.}
\label{fig:appendix-uq-evaluation}
\end{figure}

\subsection{Posterior Action Disagreement}
\label{app:posterior-action-disagreement}

We next examine whether posterior uncertainty is large in regions where the learned action is unstable. For each evaluation point, define
$p_{\calD}\p{x}
=
\mathbb{P}_{w\sim\Pi_{\eta}\p{\cdot\mid\calD}}
\p{f_w\p{x}\geq0}$
and
$u_{\calD}\p{x}
=
2\min\cb{
p_{\calD}\p{x},
1-p_{\calD}\p{x}
}$.
We approximate $p_{\calD}\p{x}$ using the same last-layer generalized Gauss--Newton draws as above. The score $u_{\calD}\p{x}$ is close to zero when posterior draws mostly agree on the action and close to one when posterior draws are split.

Let
$\widehat d_{\calD}\p{x}
=
\mathbbm{1}\sqb{f_{\widehat w_{\mathrm{MAP}}}\p{x}\geq0}$
and
$d^\star\p{x}
=
\mathbbm{1}\sqb{\tau\p{x}\geq0}$.
We define the conditional margin by
$M\p{x}=\abs{\tau\p{x}}$
and the regret contribution by
$r_{\calD}\p{x}
=
M\p{x}\mathbbm{1}\sqb{\widehat d_{\calD}\p{x}\neq d^\star\p{x}}$.
We divide evaluation points into five groups according to
$u_{\calD}\p{x}\in[0,0.2)$,
$u_{\calD}\p{x}\in[0.2,0.4)$,
$u_{\calD}\p{x}\in[0.4,0.6)$,
$u_{\calD}\p{x}\in[0.6,0.8)$,
and
$u_{\calD}\p{x}\in\sqb{0.8,1}$.
Table~\ref{tab:appendix-action-disagreement} reports the results. Higher posterior disagreement is associated with lower agreement with the Bayes-optimal action and a larger regret contribution. Figure~\ref{fig:appendix-action-disagreement} shows the same pattern graphically.

\begin{table*}[t]
\centering
\caption{Posterior action disagreement over $200$ replications. Points is the pooled number of evaluation points in each group. The remaining entries are averages of replication-level group means followed by Monte Carlo standard errors in parentheses.}
\label{tab:appendix-action-disagreement}
\begin{tabular}{rrrrrr}
\toprule
Group & Points & Disagreement & Agreement & Margin & Regret \\
\midrule
1 & 391154 & 0.0023 (0.0001) & 0.9955 (0.0005) & 0.8822 (0.0002) & 0.0002 (0.0000) \\
2 & 3026 & 0.2923 (0.0005) & 0.7176 (0.0173) & 0.0863 (0.0032) & 0.0165 (0.0017) \\
3 & 2210 & 0.4958 (0.0006) & 0.6568 (0.0153) & 0.0802 (0.0037) & 0.0219 (0.0018) \\
4 & 1859 & 0.6972 (0.0007) & 0.5799 (0.0122) & 0.0762 (0.0039) & 0.0278 (0.0019) \\
5 & 1751 & 0.8992 (0.0006) & 0.5225 (0.0103) & 0.0748 (0.0040) & 0.0348 (0.0023) \\
\bottomrule
\end{tabular}
\end{table*}

\begin{figure*}[t]
\centering
\includegraphics[width=0.49\textwidth]{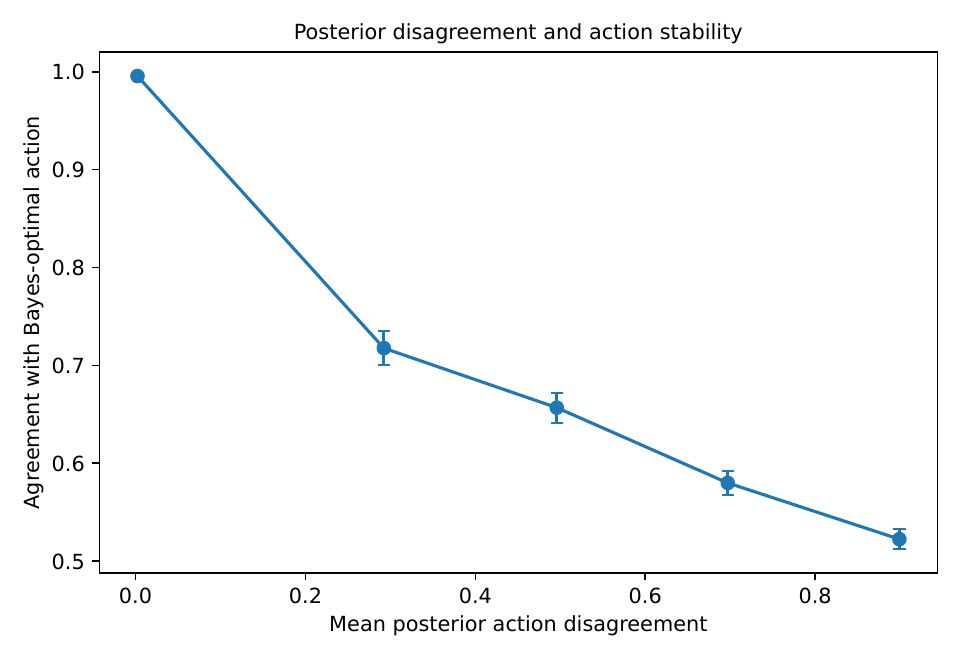}
\includegraphics[width=0.49\textwidth]{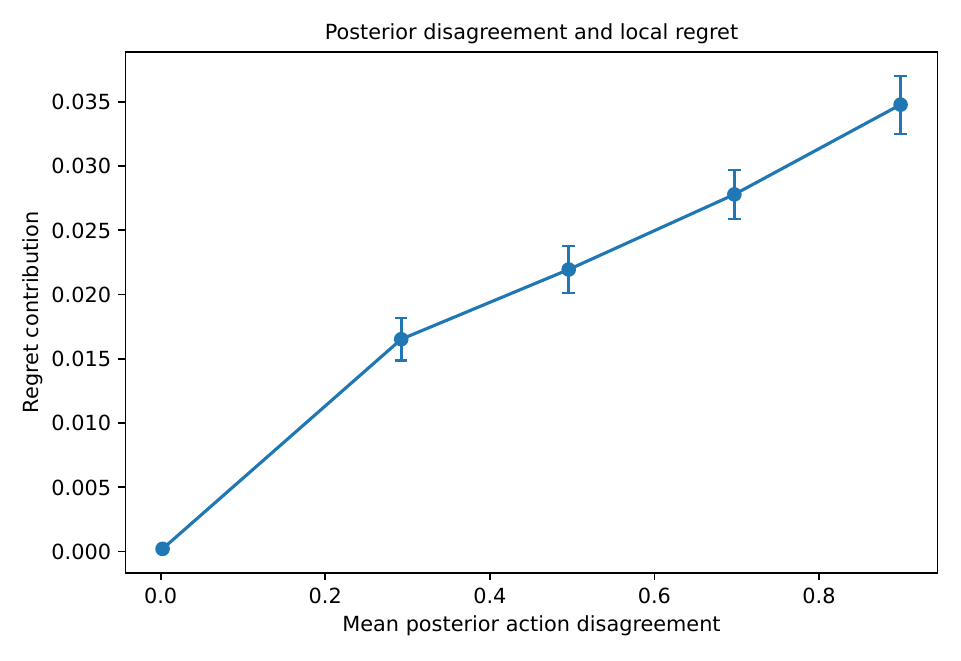}
\caption{Posterior action disagreement over $200$ replications. The left panel reports agreement with the Bayes-optimal action. The right panel reports the regret contribution. Error bars show Monte Carlo standard errors.}
\label{fig:appendix-action-disagreement}
\end{figure*}

\subsection{Posterior-Threshold Rules}
\label{app:posterior-threshold-rules}

The MAP rule is a point decision and does not directly use the width of the welfare interval. The posterior can also define rules based on posterior action probabilities. For $\alpha\in\cb{0.05,0.10,0.20,0.30}$, we consider the threshold rule that selects action $1$ when $p_{\calD}\p{x}>1-\alpha$ and selects action $0$ otherwise. The uncertain-region rate is the fraction of evaluation points satisfying $\alpha\leq p_{\calD}\p{x}\leq1-\alpha$, and the action-change rate is the fraction for which the threshold rule differs from the MAP rule.

Table~\ref{tab:appendix-posterior-threshold-rules} reports the results. In this design, the threshold rule changes a small fraction of actions and has a small welfare loss relative to the MAP rule. Figure~\ref{fig:appendix-posterior-threshold-rules} plots the welfare difference against the uncertain-region rate.

\begin{table}[t]
\centering
\caption{Posterior-threshold rules over $200$ replications. Welfare and Difference are means followed by Monte Carlo standard errors in parentheses. The fifth percentile is computed across replications.}
\label{tab:appendix-posterior-threshold-rules}
\begin{tabular}{rcccccc}
\toprule
$\alpha$ & Welfare & 5th pct. & MAP & Difference & Uncertain & Change \\
\midrule
0.05 & 0.4309 (0.0001) & 0.4278 & 0.4315 & -0.0006 (0.0001) & 0.0285 & 0.0144 \\
0.10 & 0.4311 (0.0001) & 0.4287 & 0.4315 & -0.0004 (0.0001) & 0.0222 & 0.0112 \\
0.20 & 0.4313 (0.0001) & 0.4293 & 0.4315 & -0.0002 (0.0001) & 0.0146 & 0.0074 \\
0.30 & 0.4314 (0.0001) & 0.4293 & 0.4315 & -0.0001 (0.0000) & 0.0090 & 0.0046 \\
\bottomrule
\end{tabular}
\end{table}

\begin{figure}[t]
\centering
\includegraphics[width=0.70\columnwidth]{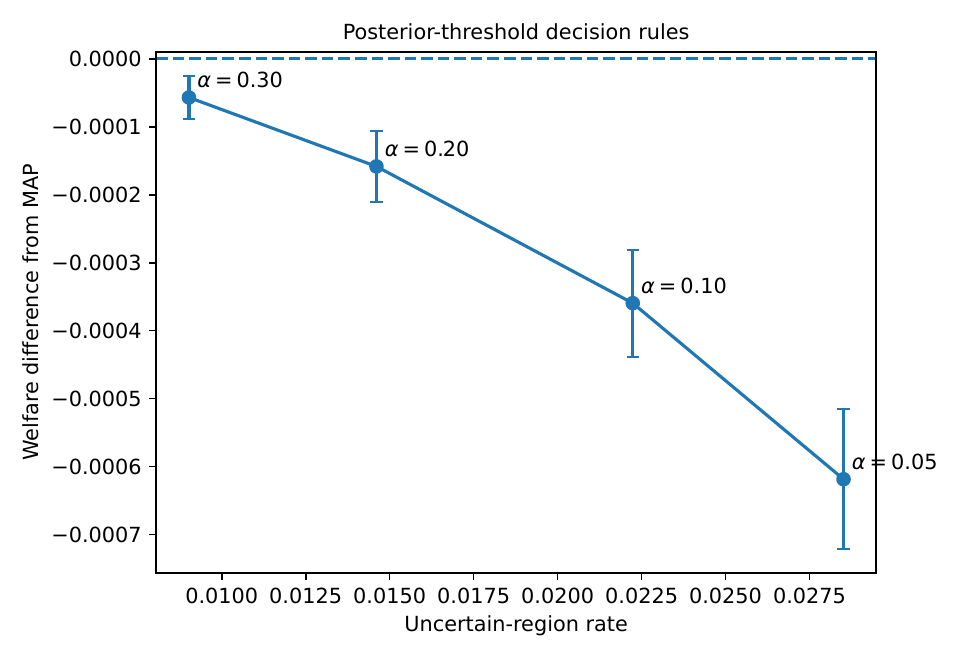}
\caption{Posterior-threshold rules over $200$ replications. The vertical axis reports the welfare difference from the MAP rule.}
\label{fig:appendix-posterior-threshold-rules}
\end{figure}

\clearpage
\section{Experimental Details}
\label{app:exp-details}

\subsection{Common Protocol}
This subsection summarizes the experimental protocol used in Section~\ref{sec:experiments}.
In all synthetic experiments, we observe the full feedback vector for each unit, and welfare is computed on a held-out test set using the realized potential outcomes.
We repeat each experiment $100$ times with independent random seeds.
For each trial, we split the data into training, validation, and test sets with proportions $0.6$, $0.2$, and $0.2$.
We report the mean and variance of the test welfare across trials, together with boxplots.
For GBPL and GBPLNet, we compare four values of $\zeta$, namely $\zeta\in\cb{1,0.1,0.01}$ and a validation-selected value among a fixed grid.

\subsection{Synthetic Experiments with Binary Actions}
In the binary action experiments, we generate i.i.d. covariates $X\in\bbR^{d}$ with $d=10$ from a standard normal distribution.
For each draw, we generate two potential outcomes $Y\p{0}$ and $Y\p{1}$ with additive Gaussian noise.
We consider three data generating processes.

In DGP1, the baseline outcome mean is
$\gamma_{0}\p{x}=x_{1}+0.5x_{2}^{2}-0.25x_{3}x_{4}$ and the conditional treatment effect is
$\tau\p{x}=2\tanh\p{x^{\top}w/\sqrt{d}}$ for a fixed unit-norm vector $w$.
We set $Y\p{0}=\gamma_{0}\p{X}+\varepsilon_{0}$ and $Y\p{1}=\gamma_{0}\p{X}+\tau\p{X}+\varepsilon_{1}$ with independent $\varepsilon_{0},\varepsilon_{1}\sim\mathcal{N}\p{0,1}$.

In DGP2, the baseline outcome mean is
$\gamma_{0}\p{x}=0.5\sin\p{x_{1}}+0.3x_{2}-0.2x_{3}^{2}$ and the treatment effect is
$\tau\p{x}=1.5\sin\p{x_{1}+x_{2}}$, and we generate outcomes as above with independent standard normal noise.

In DGP3, the baseline outcome mean is
$\gamma_{0}\p{x}=0.2x_{1}-0.1x_{2}+0.1\tanh\p{x_{3}}$ and the treatment effect is
$\tau\p{x}=2.5\p{\mathbbm{1}\sqb{x_{1}>0}-0.5+0.2x_{2}}$, and we generate outcomes as above with independent standard normal noise.

For evaluation, a fitted score $f(x)\in(-1,1)$ induces a deterministic policy that selects action $1$ if $f(x)\ge 0$ and selects action $0$ otherwise.
The oracle policy selects the action with the larger realized potential outcome.

\subsection{Synthetic Experiments with Multiple Actions}
In the multi-action experiments, we set $K=5$ and generate i.i.d. covariates $X\in\bbR^{d}$ with $d=10$ from a standard normal distribution.
For each unit, we generate the full feedback vector $\p{Y(1),\dots,Y(K)}$ with additive Gaussian noise.
We again consider three data generating processes.

In DGP1, the baseline component is
$b\p{x}=x_{1}+0.5x_{2}^{2}-0.25x_{3}x_{4}$, and for each action $a\in\cb{1,\dots,K}$ we generate a random unit-norm vector $w_{a}$ and define 
$\gamma_{a}\p{x}=b\p{x}+1.5\sin\p{x^{\top}w_{a}/\sqrt{d}+0.3a}$.
We set $Y\p{a}=\gamma_{a}\p{X}+\varepsilon_{a}$ with independent $\varepsilon_{a}\sim\mathcal{N}\p{0,1}$.

In DGP2, the baseline component is
$b\p{x}=0.5\sin\p{x_{1}}+0.3x_{2}-0.2x_{3}^{2}$, and
$\gamma_{a}\p{x}=b\p{x}+2\tanh\p{x^{\top}w_{a}/\sqrt{d}-0.2a}$.

In DGP3, the baseline component is
$b\p{x}=0.2x_{1}-0.1x_{2}+0.1\tanh\p{x_{3}}$, and
$\gamma_{a}\p{x}=b\p{x}+1.0\mathbbm{1}\sqb{x^{\top}w_{a}/\sqrt{d}>0}+0.1a$.

A fitted model outputs a probability vector $\delta\p{x}\in\Delta_{K}$, and we evaluate the corresponding stochastic policy whose welfare is given by $\bbE\sqb{\sum_{a\in\cb{1,\dots,K}}\delta_{a}\p{X}Y\p{a}}$, which is estimated by the sample average on the test set. This differs from the deterministic argmax rule in Section~\ref{sec:computation}; the reported welfare therefore includes the effect of policy randomization.

\subsection{UCI Experiments}
The UCI/OpenML experiments use real covariates and a real-valued response $y$ obtained from OpenML \citep{Vanschoren2013openml} and the UCI repository \citep{Dua2019uci}.
Because these benchmark datasets provide only a single response, we construct synthetic full feedback vectors by defining action-dependent shifts as smooth functions of $x$ and adding them to a standardized version of $y$. 
Concretely, we construct $Y(a)$ by adding an action-specific term of the form $\tanh\p{x^{\top}w_{a}/\sqrt{d}+c_{a}}$ to the standardized response, and we then evaluate welfare using the resulting full feedback vector.
This construction preserves the observed covariates and the common response component. Because the same standardized response enters every action, action differences are determined by the added functions of $x$; the benchmark evaluates policy approximation under real covariates rather than action-specific outcome noise.

\section{Additional Experimental Results with Full-Feedback Settings}
\label{appdx:additionalsim}

\subsection{Experiments with Multiple Actions}
We consider $K=5$ actions and the baseline-free symmetric full-vector surrogate in Section~\ref{sec:K}.
Each trial uses $n=6000$ observations and $d=10$ features.
We compare DirectWelfare, a direct optimization baseline that maximizes empirical welfare in the policy class, and GBPLNet, the proposed General Bayes method based on the full-vector squared-loss surrogate.
For GBPLNet, we again compare $\zeta\in\cb{1,0.1,0.01}$ and a validation choice from $\cb{1,0.1,0.01,0.001}$.
Table~\ref{tab:synthetic-K} reports welfare and regret, and Figure~\ref{fig:synthetic-K-boxplots} shows welfare means with Monte Carlo standard errors.

Table~\ref{tab:synthetic-K} indicates that GBPLNet is competitive with DirectWelfare across the three scenarios.
Moreover, the results show that the choice of $\zeta$ can matter for welfare.
In particular, in DGP3, $\zeta=1.0$ yields noticeably lower welfare, while selection by validation welfare avoids this deterioration.
Across the $100$ DGP3 trials, the validation rule selects $\zeta=0.001$, $0.01$, and $0.1$ in $5$, $45$, and $50$ trials, respectively, and never selects $\zeta=1.0$. Because these experiments evaluate the simplex-valued stochastic policy, the reduction at $\zeta=1.0$ includes stronger shrinkage toward uniform randomization. It does not imply a change in the population action ranking of the deterministic argmax rule.

\begin{table*}[t]
\centering
\caption{Synthetic $K=5$ results, $100$ trials. Columns report welfare mean, welfare variance, welfare standard error, regret mean, and regret standard error across trials.}
\label{tab:synthetic-K}
\begin{tabular}{lllrrrrr}
\toprule
dataset & method & $\zeta$ & welfare\_mean & welfare\_var & welfare\_se & regret\_mean & regret\_se \\
\midrule
\multirow{5}{*}{DGP1} & DirectWelfare & -- & 1.3366 & 0.0117 & 0.0108 & 0.4972 & 0.0066 \\
\cmidrule(lr){2-8}
 & \multirow{4}{*}{GBPLNet} & 0.01 & 1.3404 & 0.0115 & 0.0107 & 0.4934 & 0.0063 \\
 &  & 0.1 & 1.3432 & 0.0105 & 0.0102 & 0.4906 & 0.0056 \\
 &  & 1.0 & 1.3417 & 0.0095 & 0.0097 & 0.4921 & 0.0050 \\
 &  & CV & 1.3588 & 0.0098 & 0.0099 & 0.4751 & 0.0051 \\
\midrule
\multirow{5}{*}{DGP2} & DirectWelfare & -- & 1.4239 & 0.0096 & 0.0098 & 0.4850 & 0.0052 \\
\cmidrule(lr){2-8}
 & \multirow{4}{*}{GBPLNet} & 0.01 & 1.4245 & 0.0090 & 0.0095 & 0.4844 & 0.0049 \\
 &  & 0.1 & 1.4226 & 0.0101 & 0.0100 & 0.4864 & 0.0057 \\
 &  & 1.0 & 1.4154 & 0.0103 & 0.0102 & 0.4936 & 0.0048 \\
 &  & CV & 1.4379 & 0.0095 & 0.0098 & 0.4710 & 0.0046 \\
\midrule
\multirow{5}{*}{DGP3} & DirectWelfare & -- & 0.8385 & 0.0031 & 0.0056 & 0.5136 & 0.0043 \\
\cmidrule(lr){2-8}
 & \multirow{4}{*}{GBPLNet} & 0.01 & 0.8389 & 0.0031 & 0.0055 & 0.5131 & 0.0043 \\
 &  & 0.1 & 0.8403 & 0.0029 & 0.0054 & 0.5117 & 0.0041 \\
 &  & 1.0 & 0.7466 & 0.0023 & 0.0048 & 0.6054 & 0.0037 \\
 &  & CV & 0.8413 & 0.0028 & 0.0053 & 0.5107 & 0.0041 \\
\bottomrule
\end{tabular}
\end{table*}

\begin{figure*}[t]
\centering
\includegraphics[width=0.95\linewidth]{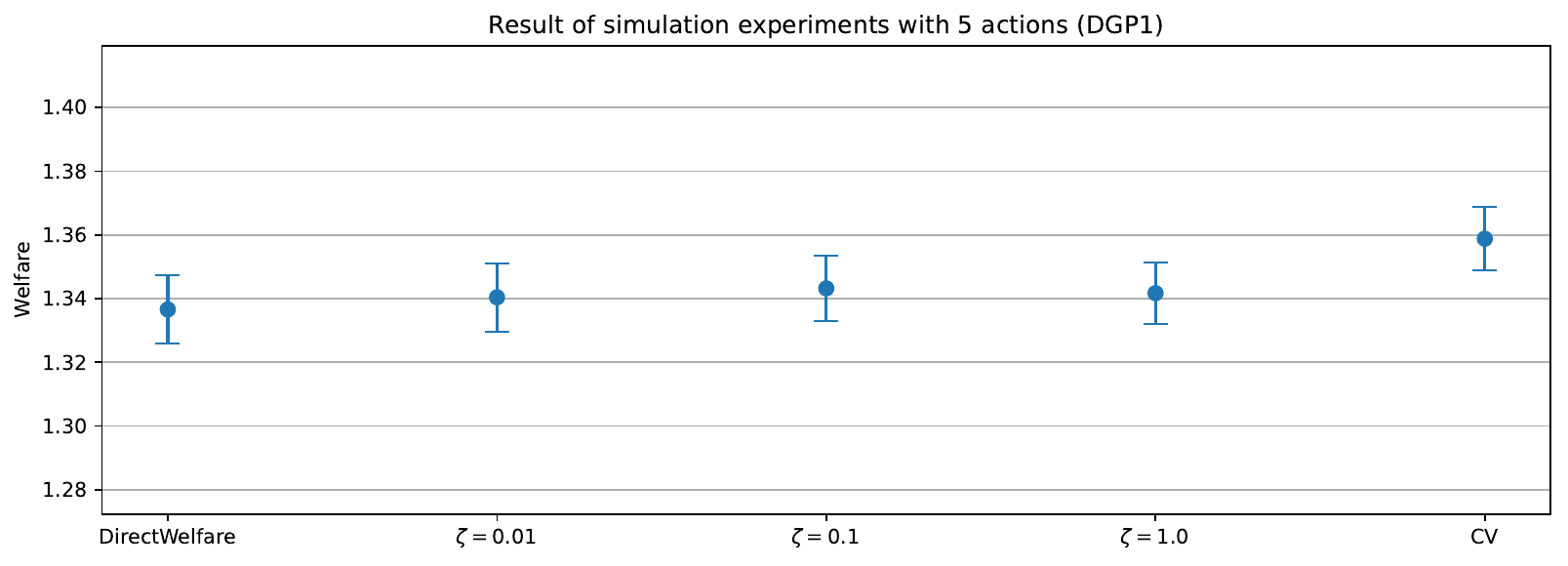}
\includegraphics[width=0.95\linewidth]{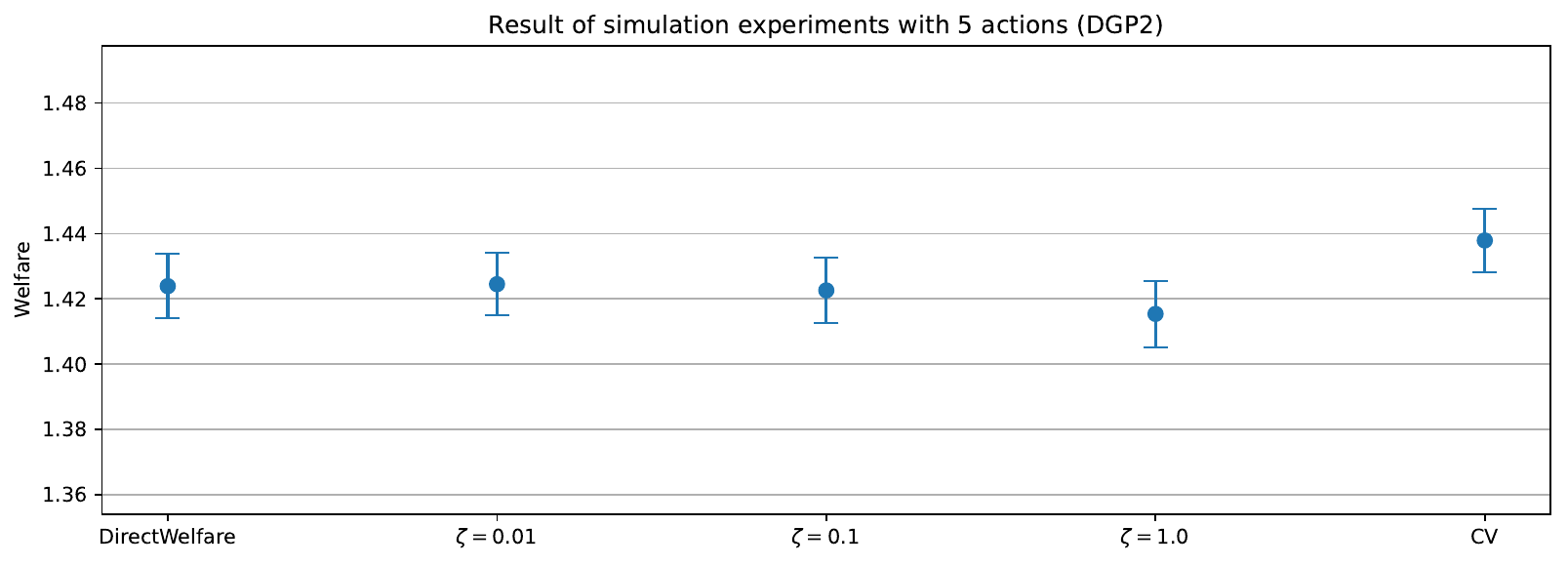}
\includegraphics[width=0.95\linewidth]{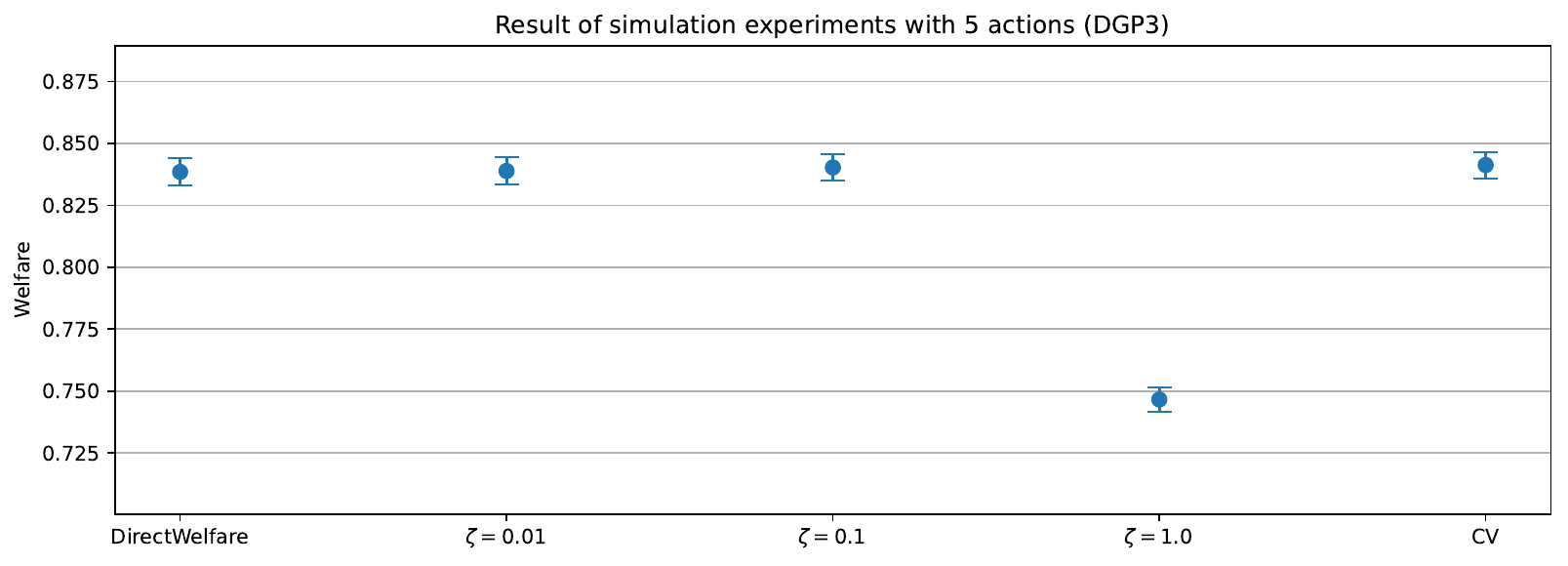}
\caption{Synthetic $K=5$ welfare means with Monte Carlo standard errors across $100$ trials for DGP1, DGP2, and DGP3.}
\label{fig:synthetic-K-boxplots}
\end{figure*}

\subsection{UCI and OpenML Experiments}
We also evaluate the binary action setting on two regression datasets, yacht and energy\_efficiency, accessed via OpenML \citep{Vanschoren2013openml}, and originally hosted in the UCI Machine Learning Repository \citep{Dua2019uci}. The yacht and energy\_efficiency datasets contain $308$ and $768$ observations, respectively.
Since these datasets provide a single response variable, we construct full potential outcomes by adding an action-dependent effect to the original response, so that each observation contains $\p{Y\p{0},Y\p{1}}$.
This construction yields a controlled full feedback benchmark that preserves the feature distribution of each dataset.

We run $100$ trials using random splits, and we compare the same methods as in the synthetic binary experiments.
Table~\ref{tab:uci} reports welfare and regret, and Figure~\ref{fig:uci-boxplots} shows welfare boxplots.
In these two datasets, all methods achieve similar welfare. The small regret values reflect the controlled construction in which action differences are deterministic functions of the covariates. DiffReg is slightly better on average, and WeightedLogistic is slightly worse.

\begin{table*}[t]
\centering
\caption{UCI and OpenML results, binary actions, $100$ trials. Columns report welfare mean, welfare variance, welfare standard error, regret mean, and regret standard error across trials.}
\label{tab:uci}
\setlength{\tabcolsep}{4pt}
\begin{tabular}{lllrrrrr}
\toprule
dataset & method & $\zeta$ & welfare\_mean & welfare\_var & welfare\_se & regret\_mean & regret\_se \\
\midrule
\multirow{7}{*}{energy\_efficiency} & DiffReg & -- & 0.3569 & 0.0295 & 0.0172 & 0.0001 & 0.0000 \\
\cmidrule(lr){2-8}
 & \multirow{4}{*}{GBPLNet} & 0.01 & 0.3565 & 0.0295 & 0.0172 & 0.0004 & 0.0001 \\
 &  & 0.1 & 0.3568 & 0.0295 & 0.0172 & 0.0002 & 0.0001 \\
 &  & 1.0 & 0.3568 & 0.0295 & 0.0172 & 0.0002 & 0.0000 \\
 &  & CV & 0.3568 & 0.0295 & 0.0172 & 0.0002 & 0.0000 \\
\cmidrule(lr){2-8}
 & PluginReg & -- & 0.3567 & 0.0296 & 0.0172 & 0.0002 & 0.0000 \\
\cmidrule(lr){2-8}
 & WeightedLogistic & -- & 0.3553 & 0.0296 & 0.0172 & 0.0016 & 0.0002 \\
\midrule
\multirow{7}{*}{yacht} & DiffReg & -- & 0.3465 & 0.0234 & 0.0153 & 0.0001 & 0.0000 \\
\cmidrule(lr){2-8}
 & \multirow{4}{*}{GBPLNet} & 0.01 & 0.3458 & 0.0234 & 0.0153 & 0.0008 & 0.0002 \\
 &  & 0.1 & 0.3463 & 0.0234 & 0.0153 & 0.0004 & 0.0001 \\
 &  & 1.0 & 0.3464 & 0.0234 & 0.0153 & 0.0002 & 0.0001 \\
 &  & CV & 0.3464 & 0.0234 & 0.0153 & 0.0003 & 0.0001 \\
\cmidrule(lr){2-8}
 & PluginReg & -- & 0.3463 & 0.0234 & 0.0153 & 0.0003 & 0.0000 \\
\cmidrule(lr){2-8}
 & WeightedLogistic & -- & 0.3449 & 0.0233 & 0.0153 & 0.0017 & 0.0002 \\
\bottomrule
\end{tabular}
\end{table*}

\begin{figure*}[t]
\centering
\includegraphics[width=0.95\linewidth]{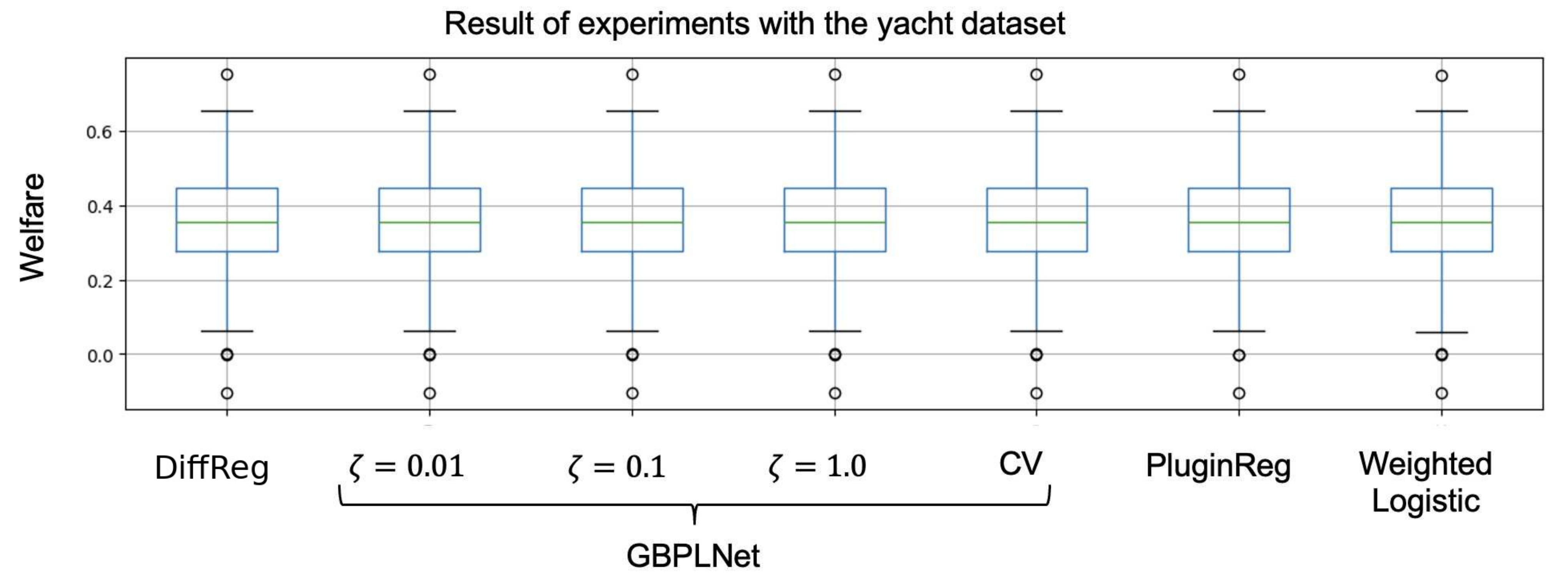}
\includegraphics[width=0.95\linewidth]{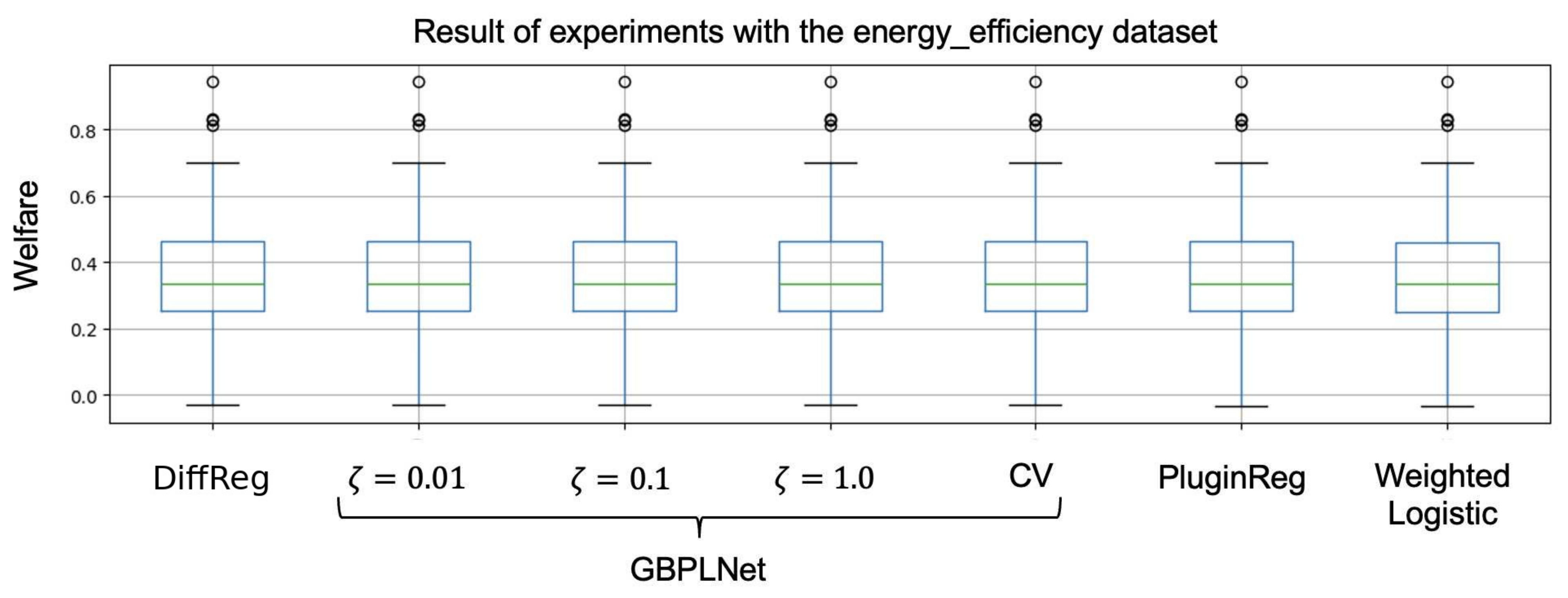}
\caption{UCI and OpenML welfare boxplots across $100$ trials for yacht and energy\_efficiency.}
\label{fig:uci-boxplots}
\end{figure*}

\section{Additional Experimental Results with Counterfactual Settings}
\label{app:cf-exp}

This appendix reports additional empirical results for the missing outcome setting in Section~\ref{sec:missing}.
In these experiments, the learner observes only the logged outcome $Y_i=Y_i\p{A_i}$ and the action $A_i$, and the full outcome vector is used only for evaluation on a held-out test set.

\subsection{Common Protocol}
We follow the same training protocol as in Appendix~\ref{app:exp-details}.
Each trial splits the data into training, validation, and test sets with proportions $0.6$, $0.2$, and $0.2$.
For each candidate model, we train on the training set and select the epoch by early stopping on the validation loss.
For the surrogate tuning parameter, we compare $\zeta\in\cb{1,0.1,0.01}$ and a validation-selected choice from the grid $\cb{1,0.1,0.01,0.001}$. In the tables, the label (CV) indicates the validation-selected $\zeta$. Because the full outcome vector is not observed in this setting, we select $\zeta$ by the IPW or DR pseudo-welfare on the validation set, matching the pseudo-outcome construction used for training.

\subsection{Logged Data Construction and Pseudo-Outcomes}
The synthetic experiments generate potential outcomes as in Appendix~\ref{app:exp-details} and then convert them into logged data by drawing actions from a stochastic logging policy.
For binary actions, we generate a propensity score $e\p{x}=\bbP\p{A=1\mid X=x}$ from a logistic model with a random linear index, and we clip $e\p{x}$ to satisfy overlap.
For $K$ actions, we generate propensities $\cb{e_a\p{x}}_{a=1}^K$ from a softmax model with random linear logits, and we sample $A$ from the resulting multinomial distribution.

Given logged data $\cb{\p{X_i,A_i,Y_i}}_{i=1}^n$, we compute pseudo-outcomes by IPW estimator in \eqref{eq:ipw-pseudo} or by the DR construction in \eqref{eq:dr-pseudo}.
In the binary action case, we form a pseudo-outcome difference by $\widetilde{U}_i=\widetilde{Y}_i\p{1}-\widetilde{Y}_i\p{0}$, and we train a bounded score model by minimizing the squared-loss surrogate in \eqref{eq:surrogate-empirical} with $Y_i\p{1}-Y_i\p{0}$ replaced by $\widetilde{U}_i$.
Furthermore, in the multi-action case, we use the full-vector empirical losses $\widehat{L}^{\mathrm{IPW}}_n\p{\theta}$ and its DR analog.

For DR pseudo-outcomes, we estimate the nuisance outcome regression $\widehat\gamma_a\p{x}$ by a neural network trained with a masked squared loss on observed outcomes.
Moreover, for propensity scores, we use the true propensity in the synthetic experiments to isolate the effect of the surrogate and the pseudo-outcome construction.
In the UCI and OpenML experiments, the propensity is estimated by multinomial or binary logistic regression on the training set.

\subsection{Synthetic Experiments with Binary Actions}
We use the three synthetic DGPs described in Appendix~\ref{app:exp-details} and generate logged data with $n=6000$ and $d=10$.
We repeat the experiment for $50$ trials.
Table~\ref{tab:cf-synthetic-binary} reports welfare and regret.

The comparison methods are as follows.
GBPLNet-DR and GBPLNet-IPW apply the proposed squared-loss surrogate with DR and IPW pseudo-outcome differences.
PluginReg fits a nuisance outcome regression model and chooses the action with the larger predicted outcome.
Since the outcome difference is not observed, DiffReg is not applicable in this setting.

Across the three DGPs, DR-based GBPLNet is typically more stable than IPW-based GBPLNet.
The sensitivity to $\zeta$ remains scenario dependent, and the validation-selected choice performs competitively in all three DGPs.

\begin{table*}[t]
\centering
\caption{Counterfactual synthetic binary results, $50$ trials. Columns report welfare mean, welfare variance, welfare standard error, regret mean, and regret standard error across trials. GBPLNet-DR and GBPLNet-IPW use DR and IPW pseudo-outcome differences, respectively. For PluginReg, welfare standard errors are computed as the square root of the reported welfare variance divided by the number of trials; a dash indicates that a regret standard error is not reported.}
\label{tab:cf-synthetic-binary}
\begin{tabular}{lllrrrrr}
\toprule
dgp & method & $\zeta$ & welfare\_mean & welfare\_var & welfare\_se & regret\_mean & regret\_se \\
\midrule
\multirow{9}{*}{DGP1} & \multirow{4}{*}{GBPLNet-DR} & 0.01 & 1.0184 & 0.0086 & 0.0131 & 0.2336 & 0.0055 \\
 &  & 0.1 & 1.0161 & 0.0082 & 0.0128 & 0.2359 & 0.0051 \\
 &  & 1.0 & 1.0212 & 0.0088 & 0.0133 & 0.2309 & 0.0054 \\
 &  & CV & 1.0211 & 0.0081 & 0.0127 & 0.2309 & 0.0050 \\
\cmidrule(lr){2-8}
 & \multirow{4}{*}{GBPLNet-IPW} & 0.01 & 1.0041 & 0.0085 & 0.0130 & 0.2479 & 0.0054 \\
 &  & 0.1 & 1.0002 & 0.0082 & 0.0128 & 0.2518 & 0.0054 \\
 &  & 1.0 & 1.0068 & 0.0082 & 0.0128 & 0.2452 & 0.0052 \\
 &  & CV & 1.0075 & 0.0081 & 0.0127 & 0.2445 & 0.0052 \\
\cmidrule(lr){2-8}
 & PluginReg & -- & 1.0198 & 0.0077 & 0.0124 & 0.2322 & -- \\
\midrule
\multirow{9}{*}{DGP2} & \multirow{4}{*}{GBPLNet-DR} & 0.01 & 0.7096 & 0.0024 & 0.0069 & 0.3296 & 0.0052 \\
 &  & 0.1 & 0.7063 & 0.0025 & 0.0071 & 0.3329 & 0.0053 \\
 &  & 1.0 & 0.6970 & 0.0025 & 0.0071 & 0.3422 & 0.0055 \\
 &  & CV & 0.6984 & 0.0026 & 0.0072 & 0.3408 & 0.0051 \\
\cmidrule(lr){2-8}
 & \multirow{4}{*}{GBPLNet-IPW} & 0.01 & 0.6967 & 0.0022 & 0.0066 & 0.3425 & 0.0055 \\
 &  & 0.1 & 0.6896 & 0.0026 & 0.0072 & 0.3496 & 0.0051 \\
 &  & 1.0 & 0.6778 & 0.0027 & 0.0073 & 0.3613 & 0.0058 \\
 &  & CV & 0.6780 & 0.0023 & 0.0068 & 0.3612 & 0.0052 \\
\cmidrule(lr){2-8}
 & PluginReg & -- & 0.6720 & 0.0031 & 0.0079 & 0.3672 & -- \\
\midrule
\multirow{9}{*}{DGP3} & \multirow{4}{*}{GBPLNet-DR} & 0.01 & 0.4744 & 0.0052 & 0.0102 & 0.3057 & 0.0043 \\
 &  & 0.1 & 0.4757 & 0.0055 & 0.0105 & 0.3044 & 0.0045 \\
 &  & 1.0 & 0.4744 & 0.0053 & 0.0103 & 0.3056 & 0.0042 \\
 &  & CV & 0.4764 & 0.0053 & 0.0103 & 0.3036 & 0.0041 \\
\cmidrule(lr){2-8}
 & \multirow{4}{*}{GBPLNet-IPW} & 0.01 & 0.4729 & 0.0058 & 0.0107 & 0.3072 & 0.0045 \\
 &  & 0.1 & 0.4695 & 0.0057 & 0.0107 & 0.3106 & 0.0046 \\
 &  & 1.0 & 0.4736 & 0.0058 & 0.0108 & 0.3065 & 0.0045 \\
 &  & CV & 0.4744 & 0.0053 & 0.0103 & 0.3056 & 0.0040 \\
\cmidrule(lr){2-8}
 & PluginReg & -- & 0.4749 & 0.0049 & 0.0099 & 0.3051 & -- \\
\bottomrule
\end{tabular}
\end{table*}

\subsection{Synthetic Experiments with Multiple Actions}
We next consider $K=5$ actions and use the three multi-action DGPs described in Appendix~\ref{app:exp-details}.
We generate logged data with $n=8000$ and $d=10$, and we repeat the experiment for $50$ trials.
Table~\ref{tab:cf-synthetic-K} reports welfare and regret.

In this setting, we use the baseline-free full-vector surrogate in Section~\ref{sec:K} and the missing-outcome empirical losses in Section~\ref{sec:missing}.
We implement the policy class by a softmax network that outputs $\pi_{\theta}\p{x}\in\Delta_K$, and we refer to the resulting method as GBPLNet.
The comparison method PluginRegK fits a nuisance outcome regression model and chooses the action with the largest predicted outcome. Table~\ref{tab:cf-synthetic-K} shows that PluginRegK has higher mean welfare in DGP1 and DGP2, whereas the best GBPLNet specification has higher mean welfare in DGP3. GBPLNet is evaluated as a stochastic simplex-valued policy, while PluginRegK uses a deterministic argmax rule, so the comparison also includes the welfare effect of policy randomization.

\begin{table*}[t]
\centering
\caption{Counterfactual synthetic $K=5$ results, $50$ trials. Columns report welfare mean, welfare variance, welfare standard error, regret mean, and regret standard error across trials. GBPLNet-DR and GBPLNet-IPW use DR and IPW pseudo-outcome vectors, respectively. For PluginRegK, welfare standard errors are computed as the square root of the reported welfare variance divided by the number of trials; a dash indicates that a regret standard error is not reported.}
\label{tab:cf-synthetic-K}
\begin{tabular}{lllrrrrr}
\toprule
DGP & method & $\zeta$ & welfare\_mean & welfare\_var & welfare\_se & regret\_mean & regret\_se \\
\midrule
\multirow{9}{*}{DGP1} & \multirow{4}{*}{GBPLNet-DR} & 0.01 & 1.3312 & 0.0112 & 0.0150 & 0.5122 & 0.0081 \\
 &  & 0.1 & 1.3467 & 0.0127 & 0.0159 & 0.4967 & 0.0084 \\
 &  & 1.0 & 1.3313 & 0.0124 & 0.0157 & 0.5122 & 0.0079 \\
 &  & CV & 1.3325 & 0.0125 & 0.0158 & 0.5109 & 0.0082 \\
\cmidrule(lr){2-8}
 & \multirow{4}{*}{GBPLNet-IPW} & 0.01 & 1.3081 & 0.0130 & 0.0161 & 0.5353 & 0.0096 \\
 &  & 0.1 & 1.3139 & 0.0125 & 0.0158 & 0.5296 & 0.0084 \\
 &  & 1.0 & 1.2927 & 0.0137 & 0.0165 & 0.5507 & 0.0094 \\
 &  & CV & 1.2959 & 0.0138 & 0.0166 & 0.5475 & 0.0092 \\
\cmidrule(lr){2-8}
 & PluginRegK & -- & 1.4011 & 0.0108 & 0.0147 & 0.4424 & -- \\
\midrule
\multirow{9}{*}{DGP2} & \multirow{4}{*}{GBPLNet-DR} & 0.01 & 1.4595 & 0.0098 & 0.0140 & 0.4790 & 0.0074 \\
 &  & 0.1 & 1.4601 & 0.0092 & 0.0135 & 0.4784 & 0.0072 \\
 &  & 1.0 & 1.4347 & 0.0102 & 0.0143 & 0.5037 & 0.0062 \\
 &  & CV & 1.4360 & 0.0107 & 0.0146 & 0.5025 & 0.0064 \\
\cmidrule(lr){2-8}
 & \multirow{4}{*}{GBPLNet-IPW} & 0.01 & 1.4240 & 0.0108 & 0.0147 & 0.5145 & 0.0085 \\
 &  & 0.1 & 1.4187 & 0.0111 & 0.0149 & 0.5198 & 0.0088 \\
 &  & 1.0 & 1.3951 & 0.0117 & 0.0153 & 0.5434 & 0.0083 \\
 &  & CV & 1.3976 & 0.0109 & 0.0148 & 0.5409 & 0.0081 \\
\cmidrule(lr){2-8}
 & PluginRegK & -- & 1.5138 & 0.0093 & 0.0136 & 0.4247 & -- \\
\midrule
\multirow{9}{*}{DGP3} & \multirow{4}{*}{GBPLNet-DR} & 0.01 & 0.8323 & 0.0028 & 0.0075 & 0.5194 & 0.0062 \\
 &  & 0.1 & 0.8302 & 0.0029 & 0.0076 & 0.5216 & 0.0064 \\
 &  & 1.0 & 0.7224 & 0.0028 & 0.0074 & 0.6293 & 0.0061 \\
 &  & CV & 0.7186 & 0.0027 & 0.0074 & 0.6331 & 0.0061 \\
\cmidrule(lr){2-8}
 & \multirow{4}{*}{GBPLNet-IPW} & 0.01 & 0.8287 & 0.0033 & 0.0081 & 0.5230 & 0.0068 \\
 &  & 0.1 & 0.8227 & 0.0036 & 0.0085 & 0.5291 & 0.0073 \\
 &  & 1.0 & 0.7124 & 0.0030 & 0.0077 & 0.6394 & 0.0064 \\
 &  & CV & 0.7134 & 0.0031 & 0.0078 & 0.6383 & 0.0064 \\
\cmidrule(lr){2-8}
 & PluginRegK & -- & 0.8177 & 0.0030 & 0.0077 & 0.5341 & -- \\
\bottomrule
\end{tabular}
\end{table*}

\subsection{UCI and OpenML Experiments}
We also report counterfactual experiments on the regression datasets yacht and energy\_efficiency, accessed via OpenML.
Since these datasets provide a single response variable, we follow the same semi-synthetic construction as in Appendix~\ref{appdx:additionalsim} and generate potential outcomes by adding an action-dependent effect to the original response.
We then generate logged data by sampling actions from a stochastic logging policy and observing only $Y\p{A}$.
We repeat the experiment for $30$ trials with random splits.
Table~\ref{tab:cf-uci-binary} reports welfare and regret.

\begin{table*}[t]
\centering
\caption{Counterfactual UCI and OpenML results, binary actions, $30$ trials. Columns report welfare mean, welfare variance, welfare standard error, regret mean, and regret standard error across trials. For PluginReg, welfare standard errors are computed as the square root of the reported welfare variance divided by the number of trials; a dash indicates that a regret standard error is not reported.}
\label{tab:cf-uci-binary}
\setlength{\tabcolsep}{4pt}
\begin{tabular}{lllrrrrr}
\toprule
dataset & method & $\zeta$ & welfare\_mean & welfare\_var & welfare\_se & regret\_mean & regret\_se \\
\midrule
\multirow{9}{*}{energy\_efficiency} & \multirow{4}{*}{GBPLNet-DR} & 0.01 & 0.3687 & 0.0422 & 0.0375 & 0.0020 & 0.0004 \\
 &  & 0.1 & 0.3688 & 0.0422 & 0.0375 & 0.0019 & 0.0004 \\
 &  & 1.0 & 0.3684 & 0.0424 & 0.0376 & 0.0023 & 0.0004 \\
 &  & CV & 0.3680 & 0.0423 & 0.0376 & 0.0026 & 0.0005 \\
\cmidrule(lr){2-8}
 & \multirow{4}{*}{GBPLNet-IPW} & 0.01 & 0.3399 & 0.0462 & 0.0393 & 0.0307 & 0.0066 \\
 &  & 0.1 & 0.3419 & 0.0479 & 0.0400 & 0.0288 & 0.0062 \\
 &  & 1.0 & 0.3394 & 0.0471 & 0.0396 & 0.0313 & 0.0050 \\
 &  & CV & 0.3387 & 0.0494 & 0.0406 & 0.0319 & 0.0065 \\
\cmidrule(lr){2-8}
 & PluginReg & -- & 0.3674 & 0.0418 & 0.0373 & 0.0033 & -- \\
\midrule
\multirow{9}{*}{yacht} & \multirow{4}{*}{GBPLNet-DR} & 0.01 & 0.3690 & 0.0245 & 0.0286 & 0.0032 & 0.0009 \\
 &  & 0.1 & 0.3684 & 0.0244 & 0.0285 & 0.0038 & 0.0009 \\
 &  & 1.0 & 0.3683 & 0.0243 & 0.0285 & 0.0040 & 0.0007 \\
 &  & CV & 0.3690 & 0.0244 & 0.0285 & 0.0033 & 0.0006 \\
\cmidrule(lr){2-8}
 & \multirow{4}{*}{GBPLNet-IPW} & 0.01 & 0.3189 & 0.0257 & 0.0293 & 0.0534 & 0.0080 \\
 &  & 0.1 & 0.3095 & 0.0270 & 0.0300 & 0.0628 & 0.0083 \\
 &  & 1.0 & 0.3051 & 0.0254 & 0.0291 & 0.0671 & 0.0093 \\
 &  & CV & 0.3082 & 0.0261 & 0.0295 & 0.0641 & 0.0096 \\
\cmidrule(lr){2-8}
 & PluginReg & -- & 0.3659 & 0.0230 & 0.0277 & 0.0064 & -- \\
\bottomrule
\end{tabular}
\end{table*}

\begin{remark}
The counterfactual experiments illustrate how the framework in Section~\ref{sec:missing} can be implemented with standard nuisance estimation and validation.
In applications with logged outcomes, one can proceed as follows.
First, choose a policy class, either a bounded score model $f_{\theta}\p{x}$ for binary actions or a simplex-valued model $\pi_{\theta}\p{x}$ for multiple actions.
Second, estimate propensities $\widehat e_a\p{x}$ and, for DR, outcome regression functions $\widehat\gamma_a\p{x}$ using flexible supervised learners.
Third, construct IPW or DR pseudo-outcomes and minimize the corresponding empirical loss.
Finally, select $\zeta$ by validation, and report welfare using an off-policy evaluation estimator that matches the pseudo-outcome construction.
In settings where posterior uncertainty is desired, the same empirical losses can be used in generalized posterior sampling by the approximation methods described in Section~\ref{sec:computation}.
\end{remark}

\end{document}